%% file: TMLR.tex
\documentclass[10pt]{article}
\usepackage[preprint]{tmlr}
\usepackage{graphicx}
\graphicspath{{../}}

\usepackage[utf8]{inputenc} 
\usepackage[T1]{fontenc}    
\usepackage{hyperref}       
\hypersetup{hidelinks}
\usepackage{url}            
\usepackage{booktabs}       
\usepackage{amsfonts}       
\usepackage{nicefrac}       
\usepackage{microtype}      
\usepackage{xcolor}         
\usepackage{grffile}        

\usepackage{caption,subcaption}
\usepackage{bbm} 
\usepackage{framed}
\usepackage{amsmath}
\usepackage{amssymb}
\usepackage{mathtools}
\usepackage{amsthm}
\usepackage[capitalize,noabbrev]{cleveref}
\input{math_commands.tex}

\theoremstyle{plain}
\newtheorem{theorem}{Theorem}[section]

\newtheorem{lemma}[theorem]{Lemma}

\theoremstyle{definition}
\newtheorem{definition}[theorem]{Definition}

\theoremstyle{remark}
\newtheorem{remark}[theorem]{Remark}

\newcommand{\norm}[1]{\left\lVert#1\right\rVert}

\newcommand{\one}{\mathbbm{1}}

\renewcommand{\Pr}[1]{\mathbf{Pr}\left[ #1 \right]}

\title{A Spectral Theory of Normalized Corrected GNN Propagation}

\author{\name Qihan Chen  \email hunk.qi.chen@gmail.com\\
      \addr Center for Discrete Mathematics, Fuzhou University
      \AND
      \name Wei Li \thanks{Wei Li is the corresponding author (liwei@fafu.edu.cn).}\email liwei@fafu.edu.cn\\
      \addr School of Computer and Information Science, Fujian Agriculture and Forestry University
      \AND
      \name Meng Qin \email mengqin\_az@foxmail.com \\
       \addr Department of Strategic \& Advanced Interdisciplinary Research,\\
       Pengcheng Laboratory (PCL), Shenzhen, Guangdong, China
      \AND
      \name Jianfeng Hou \email jfhou@fzu.edu.cn\\
      \addr Center for Discrete Mathematics, Fuzhou University}

\def\month{06}
\def\year{2026}
\def\openreview{\url{https://openreview.net/forum?id=XXXX}}

\begin{document}

\maketitle

\begin{abstract}
{
We develop a spectral theory for \emph{normalized corrected GNN propagation}.
The object of study is the symmetric normalized adjacency with its
degree-stationary component removed, matching the normalization used by standard
GCN-style models while isolating the stationary direction most directly tied to
oversmoothing.  The central theoretical question is whether this corrected
normalized operator preserves class-discriminative signal after many propagation
layers.  Our main result is a high-probability exact-recovery theorem for the
binary Contextual Stochastic Block Model after \(k=O(\log n)\) propagation steps
in the dense polylogarithmic regime \(p\ge C\log^B n/n\), for any fixed
\(B>4\), under explicit graph-signal and feature-SNR conditions.  We also establish a multi-class partial recovery theorem showing
contraction toward class centers for most nodes.  Synthetic and real
node-classification experiments are included as empirical checks of the
theory's predicted dependence on depth, graph signal, and feature noise.
}
\end{abstract}

\section{Introduction}
{

Deep GNNs repeatedly apply graph propagation operators to node features
\citep{scarselli:gnn,kipf:gcn,HYL17,wu2019simplifying,gasteiger_predict_2019}.
For GCN-style architectures, the basic linear backbone is the symmetric
normalized adjacency \(D^{-1/2}AD^{-1/2}\)
\citep{kipf:gcn,chen2020simple,wu2019simplifying}.  This normalization is
essential in practice, but it also creates a precise spectral mechanism for
oversmoothing: repeated propagation amplifies the degree-stationary component
while attenuating directions that separate classes
\citep{li2018deeper,oono2020graph,keriven2022not,rusch2023survey,jin2025oversmoothing}.
Thus a central theoretical question is not only whether oversmoothing occurs,
but whether the normalized propagation operator can be corrected in a way that
provably preserves class information at logarithmic depth.

We study this question through the operator
\[
\widehat A
=
D^{-1/2}AD^{-1/2}
-
\frac{1}{\one^\top D\one}D^{1/2}\one\one^\top D^{1/2},
\]
which subtracts the degree-stationary rank-one component before propagation.
This operator-level viewpoint follows the broader spectral and graph-signal
processing perspective on graph filters
\citep{shuman2013emerging,balcilar2021analyzing,chien2021adaptive,goksu2025krawtchouknet}.
This is the normalized counterpart of the corrected convolution analyzed by
\citep{Wang2024Analysis} for unnormalized adjacency matrices.  The normalized setting has three characteristic features: the stationary vector is
degree-dependent and random, each entry contains nonlinear degree factors, and
powers of the residual matrix are tracked entrywise in addition to their
spectral-norm behavior.

This paper develops a spectral theory of this normalized corrected GNN
propagation operator under the Contextual Stochastic Block Model (CSBM)
\citep{DSM18,Lu:2020:contextual,pmlr-v139-baranwal21a,baranwal2023effects,dalle2024optimal}.
The main theorem gives exact recovery by a linear classifier after
\(k=O(\log n)\) propagation steps when \(p\ge C\log^B n/n\) for any fixed
\(B>4\), provided the graph signal and feature signal-to-noise ratio exceed
explicit thresholds.  The theorem characterizes when the stationary correction is
sufficient to preserve separability through logarithmic depth.  The proof
reduces the normalized operator to a corrected rank-one signal plus a residual
and then develops an entrywise Krylov bound for multi-layer residual powers.
The key technical device is a degree-truncated atom expansion combined with
decorated-walk counting, which separates true graph fluctuations from artifacts
introduced by random normalization.

The experiments conducted in this paper examine
the predicted dependencies on depth, graph signal, and feature noise. They
are consistent with reduced depth-dependent degradation on synthetic CSBM
instances and standard node-classification benchmarks. The main contributions are as follows. 
\begin{itemize}
    \item \textbf{A normalized corrected propagation model.}  We formulate the
    corrected GCN backbone by subtracting the degree-stationary component of
    \(D^{-1/2}AD^{-1/2}\), giving an explicit operator-level model of spectral
    correction in normalized GNNs.
    \item \textbf{Binary exact-recovery theorem.}  We prove binary CSBM exact
    recovery after \(k=O(\log n)\) normalized corrected propagation steps in the
    dense polylogarithmic regime \(p\ge C\log^B n/n\), with explicit
    graph-signal and feature-SNR conditions.
    \item \textbf{Entrywise residual-power theory.}  We prove worst-node
    multi-layer bounds for powers of the normalized residual despite random
    degree normalization, using atom expansions and decorated-walk counting.
    \item \textbf{Multi-class recovery theorem and empirical checks.}  We prove
    partial recovery for fixed \(L\)-class CSBMs and include empirical checks of
    the predicted depth dependence on synthetic and real node-classification
    datasets.
      \item \textbf{Empirical validation.}  Experiments on synthetic CSBM data and
    standard node-classification benchmarks support the theory and show that
    normalized spectral correction remains effective in deep GNN regimes. 
\end{itemize}
}

\section{Related Work}
{
\textbf{Message passing and normalized propagation.}
GNNs combine graph structure and node features through repeated local
aggregation \citep{scarselli:gnn,kipf:gcn,HYL17}.  Several influential models can
be viewed, at least in their linear propagation backbone, as applying normalized
graph filters before or between learned feature transformations
\citep{kipf:gcn,wu2019simplifying,gasteiger_predict_2019,chen2020simple}.  This
paper focuses on that propagation backbone rather than on architectural
components such as nonlinear activations, attention, or optimization.  The
specific object of study is the symmetric normalized operator after subtracting
its degree-stationary component.

\textbf{Oversmoothing and its spectral mechanisms.}
Oversmoothing was identified as a depth-dependent loss of class information in
message passing \citep{li2018deeper}; subsequent theory shows that repeated
propagation can drive node representations toward low-dimensional or stationary
subspaces
\citep{oono2020graph,cai2021graph,Hou2020Measuring,keriven2022not,adam2024almost,rusch2023survey}.
Many methods attempt to mitigate this behavior through residual connections,
normalization, edge dropping, diffusion design, or signed/high-frequency
components
\citep{xu2018representation,Zhao2020PairNorm,chen2020simple,rong2019dropedge,chamberlain2021grand,scholkemper2025residual,wang2026a,wang2024mamba}.
Our contribution is complementary: we do not introduce a new trained
architecture, but analyze an explicit correction of the propagation operator and
prove when it preserves separability.

\textbf{Graph spectral filters and correction.}
Graph signal processing interprets message passing as filtering over the graph
spectrum \citep{shuman2013emerging}.  Early spectral GNNs and localized
polynomial filters established this operator viewpoint
\citep{bruna2014spectral,defferrard2016convolutional,balcilar2021analyzing}.
Learnable or adaptive filters, including generalized PageRank and polynomial
constructions, have been proposed for heterophily or oversmoothing
\citep{chien2021adaptive,goksu2025krawtchouknet}.  These works motivate the
operator-level viewpoint, but most are algorithmic or empirical.  Here the
filter is fixed by a correction principle: remove the rank-one stationary
component and analyze the resulting normalized residual through high-probability
spectral and entrywise bounds.

\textbf{CSBM theory and the closest prior work.}
The CSBM provides a controlled model in which graph structure and noisy node
features jointly determine classification performance
\citep{DSM18,Lu:2020:contextual,pmlr-v97-mehta19a,pmlr-v139-baranwal21a,maskey2022generalization,baranwal2023effects,baranwal2023optimality,maskey2024generalization,duranthon2025statistical,dalle2024optimal,wang2025graphattention}.
The closest result to ours is the corrected convolution theory of
\citep{Wang2024Analysis}, which proves recovery guarantees for an
\emph{unnormalized} corrected adjacency operator.  Our setting differs in the
normalization used by standard GCNs: \(D^{-1/2}AD^{-1/2}\) has a random
degree-weighted stationary vector, and its entries contain nonlinear degree
denominators.  These features require the atom expansion and decorated-walk
counting developed below; they are not consequences of the unnormalized
analysis.
}

\section{Preliminaries}
{
We use the following notation throughout this paper.  For \(n\in\mathbb N\), let
\([n]=\{1,\ldots,n\}\).  The all-ones vector is denoted by \(\one\), and
\(e_i\) denotes the \(i\)-th standard basis vector.  For a vector \(x\), let
\(\|x\|\) and \(\|x\|_\infty\) denote its Euclidean and infinity norms.  For a
matrix \(M\), let
\[
\|M\|=\sup_{\|x\|=1}\|Mx\|,
\qquad
\|M\|_F=\left(\sum_{i,j}M_{ij}^2\right)^{1/2}
\]
denote the operator and Frobenius norms. We also use the induced \(\ell_\infty\to\ell_\infty\) matrix norm
\[
\|M\|_{\infty\to\infty}
:=\sup_{\|x\|_\infty=1}\|Mx\|_\infty
=\max_i\sum_j |M_{ij}|,
\]
that is, the maximum absolute row-sum norm.

Let \(G=(V,E)\) be an undirected unweighted graph with \(V=[n]\), adjacency
matrix \(A\), degrees \(d_i\), degree matrix \(D=\operatorname{diag}(d_1,\ldots,d_n)\),
and average degree
\[
d:=\frac{2|E|}{n}.
\]
We use the standard conventions \((D^{-1/2})_{ii}=0\) when \(d_i=0\), and
\(d=(\one^\top D\one)/n\), so \(d^{-1}=n/(\one^\top D\one)\) when
\(\one^\top D\one>0\), and \(d^{-1}=0\) when \(d=0\).
For GCN-style propagation, the symmetric normalized adjacency is
\(D^{-1/2}AD^{-1/2}\).  The dominant stationary component of this operator is
not the uniform vector but the degree-weighted vector \(D^{1/2}\one\).  We
therefore study the corrected normalized propagation matrix
\begin{equation}\label{eq:graph_convolution}
\widehat A
=
D^{-1/2}AD^{-1/2}
-
\frac{1}{\one^\top D\one}
D^{1/2}\one\one^\top D^{1/2}.
\end{equation}
For comparison with the unnormalized theory of \citep{Wang2024Analysis}, we also
write
\[
\widetilde A
=
\frac1d A-\frac1n\one\one^\top .
\]
Both operators remove the leading stationary direction before repeated
propagation.  The normalized operator \(\widehat A\) is the one used in our main
results and experiments.
}

\section{Main Results}

{
We now state the main spectral consequences of normalized corrected
propagation.  The CSBM \citep{duranthon2025statistical} adds Gaussian node features
to a planted block model, making it possible to track three quantities in the
same theorem: graph signal, feature noise, and propagation depth.  Our results
show that, after removing the degree-stationary component, the normalized
operator behaves like a low-rank class-signal operator plus a residual whose
multi-layer action remains controlled.
}

\begin{definition}
\label{def:csbm}
The CSBM, denoted $\mathcal{CSBM}(n, m, p, q, \mu, \nu, \sigma)$, is defined as follows. Let $V = [n]$ be the set of vertices (with $n$ even), partitioned into two disjoint communities $S$ and $T$ of equal size $|S| = |T| = n/2$. 
    Let $y \in \{-1, 1\}^n$ be the community indicator  vector by setting $y_i = 1$ if $i \in S$ 
    and $y_i = -1$ if $i \in T$.
    The model generates a random graph $G=(V, E)$ and a feature matrix $X \in \mathbb{R}^{n \times m}$ 
    according to the following processes:
    \begin{enumerate}
        \item \textbf{Graph Generation:} For every pair of distinct vertices $i, j \in V$, 
        the edge $(i,j)$ is included in $E$ independently with probability:
        \[
            \mathbb{P}((i,j) \in E) = \begin{cases} p & \text{if } y_i = y_j \quad (\text{intra-class}), 
                \\ q & \text{if } y_i \neq y_j \quad (\text{inter-class}). \end{cases}
        \]
        \item \textbf{Feature Generation:} For each vertex $i \in V$, the feature vector 
        $x_i \in \mathbb{R}^m$ (the $i$-th row of $X$) is drawn independently from a Gaussian mixture:
        \[
            x_i \sim \begin{cases} \mathcal{N}(\mu, \sigma^2 I_m) & \text{if } i \in S, 
                \\ \mathcal{N}(\nu, \sigma^2 I_m) & \text{if } i \in T. \end{cases}
        \]
    \end{enumerate}
\end{definition}

For $p>q>0$, set $\gamma(p,q):=\frac{p-q}{p+q}$. In \citep{Wang2024Analysis}, proved that 

\begin{theorem}[Exact Recovery for \(\widetilde A\)~\citep{Wang2024Analysis}]
\label{theorem:exactly-recovery-k-unnormalized}
Suppose we are given a 2-block $m$-dimensional CSBM with parameters $n,p>q, \mu, \nu, \sigma$ satisfying $\gamma(p,q) \geq \Omega\left(k\sqrt{\frac{\log{n}}{np}}\right)$ and $p \geq \frac{\log^3{n}}{n}$. Then after $k = O(\log{n})$ rounds of graph convolution with $\Tilde{A}$, our data is linearly separable with probability $1-n^{-\Omega(1)}$ if
\begin{align*}
    \frac{\norm{\mu - \nu}}{\sigma}\geq \Omega\left(\max\left(\sqrt{\frac{\log{n}}{n}}, \; \Big(\frac{C}{\gamma\sqrt{np}}\Big)^k\sqrt{\log{n}} \right)\right)
\end{align*}
where $C$ is an absolute constant.
\end{theorem}

The first main theoretical result is the normalized analogue.  It has the same
structure, but the random degree normalization requires a dense
polylogarithmic density assumption and a slightly stronger residual-control
term.

\begin{theorem}[Exact Recovery for \(\widehat A\)]
\label{theorem:exactly-recovery-k}
Fix any constant \(B>4\). There exist constants \(C_d,C_\gamma,C_0,C_*,C_{\rm snr},c_0>0\), 
depending only on \(B\), such that the following holds. Consider a balanced two-block 
\(\mathcal{CSBM}(n,m,p,q,\mu,\nu,\sigma)\) with \(p>q>0\), centered feature means \(\mu+\nu=0\). Let $\theta:=C\sqrt{\frac{\log n}{np}}$ and $\rho:=\frac{\theta}{\gamma}$, 
where $C>0$ is a sufficiently large absolute constant. Assume $p\ge C_d\frac{\log^B n}{n}$ 
and suppose $1\le k\le C_0\log n$ and $\gamma\ge C_\gamma k\sqrt{\frac{\log n}{np}}$ 
(where \(C_0\) is chosen no larger than the constant \(c\) in Lemma~\ref{lem:krylov-infty}), 
which equivalently implies $k\rho\le c_0$. Then, after $k$ rounds of normalized corrected 
graph convolution, the rows of $X^{(k)}=\widehat A^kX$ are linearly separable with probability 
at least $1-n^{-\Omega(1)}$, provided
\begin{align*}
    \frac{\|\mu-\nu\|}{\sigma}\ge C_{\rm snr}\max\left\{\sqrt{\frac{\log n}{n}}, \; 
    \left(\frac{C_*}{\gamma}\sqrt{\frac{\log n}{np}}\right)^k\sqrt{\log n}\right\}.
\end{align*}
\end{theorem}

\begin{remark}
Theorem~\ref{theorem:exactly-recovery-k} gives a normalized exact-recovery
guarantee under the dense polylogarithmic condition \(np\ge C\log^B n\) with
fixed \(B>4\).  This is a slightly stronger density assumption than the sharp
\(np\ge C\log^3 n\) target, but it makes Taylor remainders in the normalized
degree expansion controllable by deterministic row-sum bounds.
\end{remark}

\paragraph{Proof roadmap and key technical estimate.}
We include the proof roadmap here, and the main technical contribution of
the theorem lies in the normalized residual-power analysis. First, the feature model is
reduced to the one-dimensional centered form \(Xw=s+g\), where
\(g\sim N(0,\sigma'^2I_n)\), by Lemma~\ref{lemma:csbm-reduction}.  Second,
degree concentration gives the signal-plus-residual decomposition
\[
\widehat A=\eta ss^\top+R',
\qquad
\eta=\frac{(p-q)n}{2d},
\qquad
\|R'\|\lesssim \sqrt{\frac{\log n}{np}} .
\]
However, this spectral-norm control is not enough for exact
recovery. We need worst-node control after \(k=O(\log n)\) layers.  The key
technical estimate is the following short form of Lemma~\ref{lem:krylov-infty}:
with high probability, uniformly over all vertices \(u\) and all
\(1\le a\le k\),
\[
\bigl|e_u^\top (R')^a s\bigr|
\le
\frac{1}{\sqrt n}
\left(K\sqrt{\frac{\log n}{np}}\right)^a .
\tag{Key}
\]
This bound says that residual powers remain delocalized in the planted signal
direction; the extra \(n^{-1/2}\) factor is what allows a sign-based classifier
to control every node rather than only an average error.

The proof of \((\mathrm{Key})\) is the normalized part of the argument.  Unlike
the unnormalized case, each entry of \(R'\) contains random nonlinear degree
denominators and a random stationary correction.  We handle these terms by
expanding normalized entries into positive-level atoms
(Lemma~\ref{lem:entrywise-atom-expansion}) and then bounding the surviving
high-moment terms by decorated-walk counting
(Lemma~\ref{lem:decorated-pattern-counting}).  This yields the
decoupled decorated-walk expansion and moment bound in
Appendix~\ref{appendix:additional_preliminaries}.  Finally,
Lemma~\ref{lem:error-evolution-normalized} combines \((\mathrm{Key})\), the
spectral norm bound, and Gaussian concentration to show that
\(\eta^{-k}\widehat A^k(s+g)\) stays within \(o(n^{-1/2})\) of \(s\) in
\(\ell_\infty\), which gives linear separability.

The second theoretical result extends the analysis to \(L\) classes, with
equal-sized graph blocks and Gaussian class-conditional features.  The result is
weaker than exact recovery, but it retains the same spectral conclusion: after
normalization correction, most propagated representations contract toward their
class centers when the graph signal dominates perturbation and feature noise.

\begin{definition}[Multi-Class CSBM]
\label{def:multi-class-csbm}
    A multi-class Contextual Stochastic Block Model (multi-class CSBM) with $L$ equal-sized classes, denoted as $\mathcal{CSBM}(n, m, p, q, L, \mu, \sigma)$, is generated as follows:
    \begin{enumerate}
        \item \textbf{Node Labels:} The $n$ nodes are partitioned into $L$ disjoint classes $\mathcal{C}_1, \dots, \mathcal{C}_L$, each of size $n/L$.
        \item \textbf{Graph Topology:} An undirected graph $G=(V, E)$ is generated where edges between nodes $i$ and $j$ are drawn independently. The probability of an edge is $p$ if $i$ and $j$ belong to the same class, and $q$ if they belong to different classes.
        \item \textbf{Node Features:} For a node $i \in \mathcal{C}_l$, its $m$-dimensional feature vector is generated as $x_i \sim \mathcal{N}(\mu_l, \sigma^2 I_m)$, where $\mu_l \in \mathbb{R}^m$ is the specific center for class $l$. The parameter $\mu = \{\mu_1, \dots, \mu_L\}$ denotes the set of all class centers. We assume the centers are zero-mean, i.e., $\sum_{l=1}^L \mu_l = 0$.
    \end{enumerate}
\end{definition}

The expected average degree is
\[
\bar{d} := \frac{pn}{L} + \frac{(L-1)qn}{L}.
\]
Let \(\lambda := (p-q)n/(\bar d L)\) be the graph signal strength and
\[
\delta := C\bar d^{-1}\left(\sqrt{\frac{np(1-p)}{L}}+\sqrt{nq(1-q)}\right)
\]
be the graph-noise scale.  Let \(U=\mathbb E[X]\), assume \(U^\top\one=0\), and
set \(\Delta:=\min_{\ell\ne h}\|\mu_\ell-\mu_h\|\).  We first recall the
corresponding unnormalized result of \citep{Wang2024Analysis}; the normalized
analogue follows as the next theorem.

\begin{theorem}[Multi-class Partial Recovery for \(\widetilde A\)~\citep{Wang2024Analysis}]
\label{theorem:multi-class-partial-unnormalized}
    
Given the CSBM with parameters, $p,q,L,n,m,\sigma$, suppose $\min(p,q) \geq \Omega(\frac{\log^{2}n}{n})$ and $|\lambda| > 4k\delta$. Let $X^{(k)} = \frac{1}{\lambda^k}\Tilde{A}^{k}X$ be the feature matrix after $k$ rounds of convolutions with scalaing factor $1/\lambda^k$. Let $x^{(k)}_i$ be the $i^{th}$ row of the matrix $X^{(k)}$. Then with probability $1-n^{-\Omega(1)}$, at least $n - n_e$ nodes, $i$, satisfy $\norm{x^{(k)}_i - \mu_i}< \Delta/2$ where 
$$n_e =  O\Big((k \delta/|\lambda|)^2\frac{\norm{U}_F^2}{\Delta^2} + (L + n(\delta/|\lambda|)^{2k})\frac{\sigma^2m\log{n}}{\Delta^2}\Big).$$

In particular, the quadratic classifer $x\mapsto \text{softmax}(\norm{x-c_l}^2)_{l=1}^L$ will correctly classify at least $n-n_e$ points, and when $n_e = o(n)$, then we can correctly classify $1-o(1)$ fraction of points.
\end{theorem}

The following theorem gives the corresponding partial-recovery guarantee for
the normalized corrected operator \(\widehat A\).

\begin{theorem}[Multi-class Partial Recovery for \(\widehat A\)]
\label{theorem:multi-class-partial}
    Given the multi-class CSBM with parameters $p,q,L,n,m,\sigma$, assume that \(L\) is fixed, \(L\mid n\), and all constants below may depend on \(L\). Suppose the graph is sufficiently dense with $\bar d \ge C\log n$, the feature dimension satisfies \(m\le n^{C_m}\) for a fixed constant \(C_m>0\), and the signal dominates the noise such that $|\lambda| \ge C_0 k(\delta+\epsilon)$ for a sufficiently large absolute constant \(C_0\), where \(\epsilon := C_\epsilon\sqrt{\log n/\bar d}\). 
    Let \(X^{(k)}=\lambda^{-k}\widehat A^kX\) be the scaled feature matrix after \(k\) rounds of propagation. 
    Let \(x^{(k)}_i\) be the \(i\)-th row of \(X^{(k)}\), and let \(\ell(i)\) denote the class of node \(i\). Then with high probability \(1-n^{-\Omega(1)}\), at least \(n-n_e\) nodes satisfy \(\|x^{(k)}_i-\mu_{\ell(i)}\|<\Delta/2\), where \(n_e\) is bounded by
    \begin{gather*}
        n_e =  O\left(\underbrace{\Big[(\frac{\epsilon}{|\lambda|})^{2k}+4k^2(\frac{\delta+\epsilon}{|\lambda|})^2\Big]\frac{\norm{U}_F^2}{\Delta^2}}_{\text{Structural \& Normalization Error}}
        + \underbrace{\Big[(1+\frac{\delta+\epsilon}{|\lambda|})^{2k}(L-1) + n(\frac{\delta+\epsilon}{|\lambda|})^{2k}\Big]\frac{\sigma^2m\log{n}}{\Delta^2}}_{\text{Initial Variance Error}}\right).
    \end{gather*}
In particular, the nearest-center classifier
\(\arg\min_{\ell\in[L]}\|x-\mu_\ell\|\) correctly classifies all nodes outside
this exceptional set.
\end{theorem}

{
Theorem~\ref{theorem:multi-class-partial} gives an explicit partial-recovery
bound for normalized corrected propagation.  When graph perturbation,
normalization error, and feature noise are small relative to \(\Delta\), the
bound on \(n_e^{\rm norm}\) is sublinear.  The two error terms separate
structural error from propagated feature variance:
}
\begin{itemize}
    \item \textbf{Structural error.}  The first term accumulates graph
    perturbation and degree-normalization error relative to class separation.
    \item \textbf{Feature-noise error.}  The second term tracks propagated
    Gaussian noise and its dependence on \(m\), \(\sigma\), and \(\Delta\).
\end{itemize}
Compared with Theorem~\ref{theorem:multi-class-partial-unnormalized}, the
normalized result replaces the graph perturbation scale \(\delta\) by
\(\delta+\epsilon\) and adds the purely normalization-driven term
\((\epsilon/|\lambda|)^{2k}\|U\|_F^2/\Delta^2\).

The proof is deferred to
Appendices~\ref{appendix:additional_preliminaries} and~\ref{proof:theorem:multi-class-partial}.

\section{Experiments}
\label{sec:experiments}
{
The experiments are secondary to the theoretical results and are designed as
empirical checks of the spectral predictions rather than as a full architecture
benchmark.  On synthetic CSBM graphs, we vary feature SNR, graph signal
strength, and propagation depth to test the theorem's predicted dependencies.
On standard node-classification benchmarks, we test whether the same normalized
correction is associated with smaller depth-dependent degradation outside the
idealized model.  We compare against standard normalized propagation and three
recent oversmoothing baselines:
\begin{itemize}
    \item \textbf{DropEdge} \citep{rong2019dropedge}, which randomly removes edges during training.
    \item \textbf{GraphMamba} \citep{wang2024mamba}, which uses state-space modules for long-range propagation.
    \item \textbf{RevGNN} \citep{yang2024mitigating}, which uses reverse diffusion to recover high-frequency information.
\end{itemize}
}

\subsection{Synthetic Data}
{
We generate two-block CSBM graphs with \(n=2000\) and feature dimension
\(m=20\).  After \(k\) propagation steps, a linear classifier is trained on the
resulting embeddings; all curves report average accuracy over 50 trials.
}

We test two regimes:
\begin{enumerate}
    \item \textbf{Feature SNR.}  We fix
    \(\gamma=(p-q)/(p+q)=2/3\) and vary \(\|\mu-\nu\|/\sigma\).  Consistent
    with the theorem, corrected propagation remains more stable at large \(k\)
    when feature SNR is sufficient (Figure~\ref{fig:synthetic-sigma}).
     
    \item \textbf{Graph signal.}  We fix \(\|\mu-\nu\|/\sigma=1\) and vary
    \(\gamma\).  Corrected propagation retains higher accuracy as graph signal
    decreases, and the empirical transition follows the theoretical threshold
    trend (Figure~\ref{fig:synthetic-gamma}).
\end{enumerate}

\begin{figure}[!ht]
    \centering
    \begin{subfigure}[b]{0.32\textwidth}
        \includegraphics[width=\linewidth]{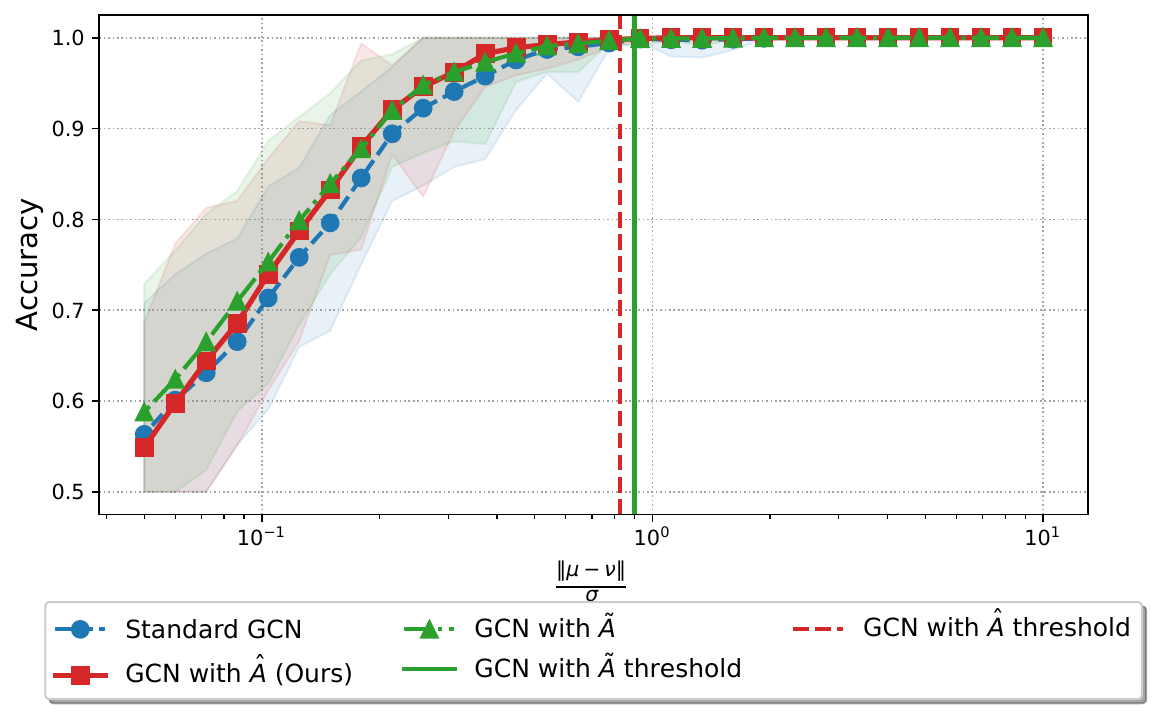}
        \caption{\(k=1\).}
    \end{subfigure}
    \begin{subfigure}[b]{0.32\textwidth}
        \includegraphics[width=\linewidth]{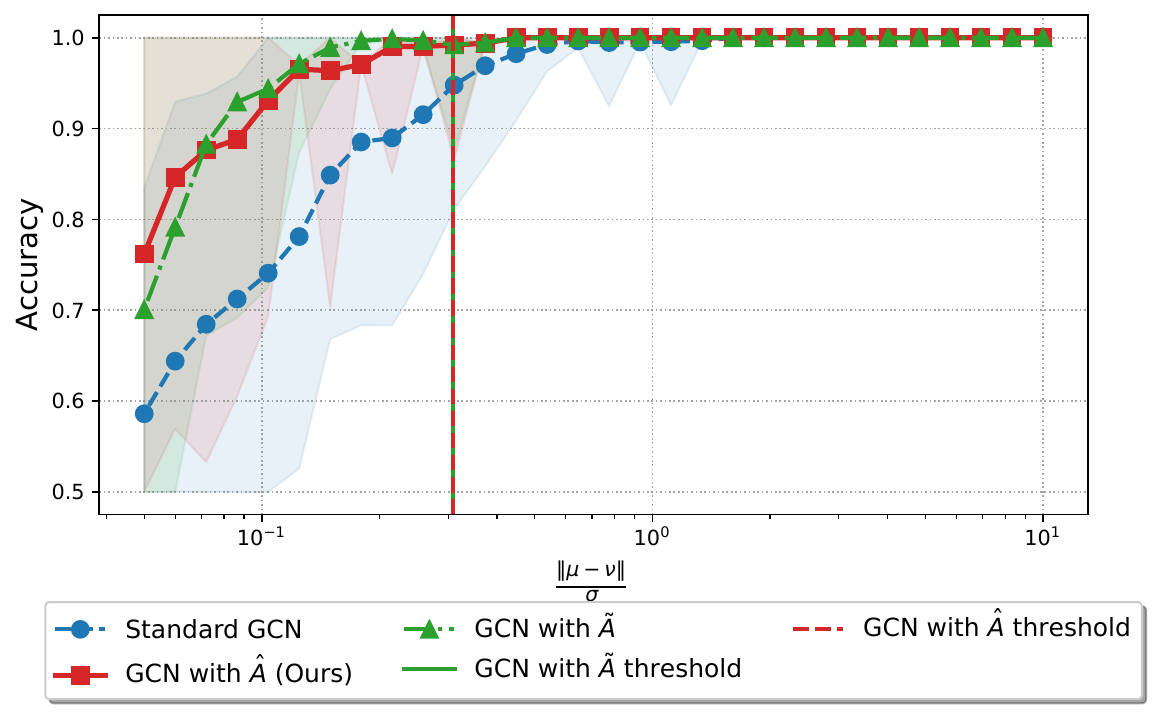}
        \caption{\(k=2\).}
    \end{subfigure}
    \begin{subfigure}[b]{0.32\textwidth}
        \includegraphics[width=\linewidth]{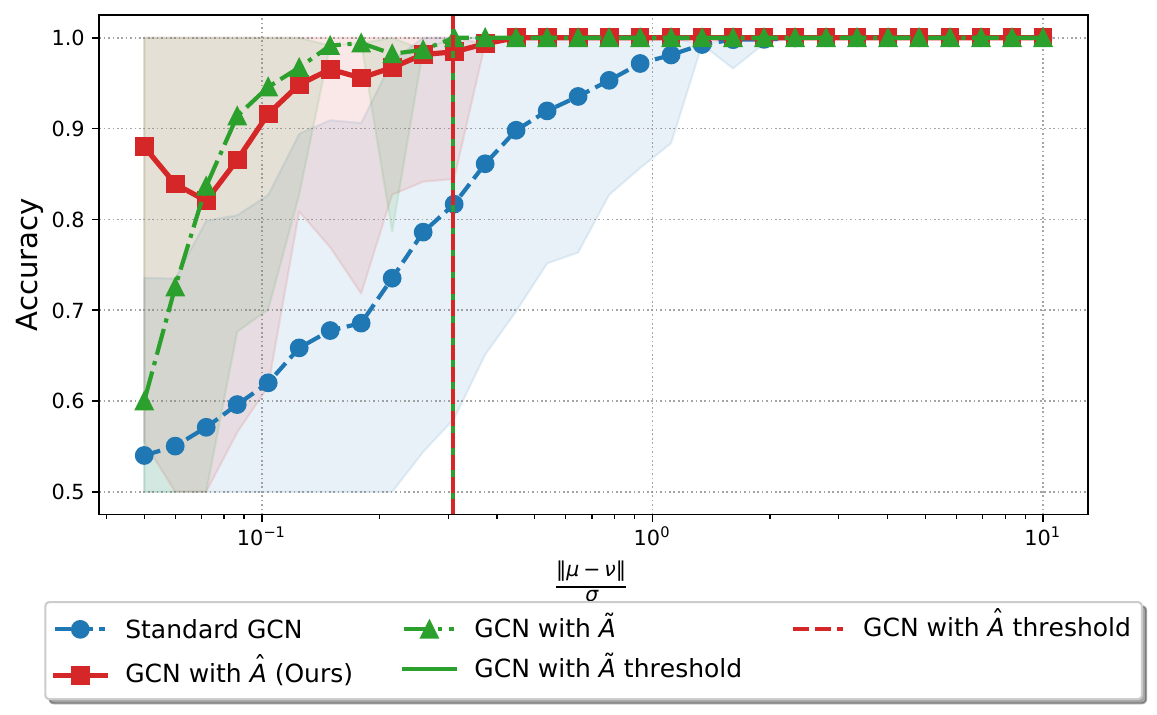}
        \caption{\(k=4\).}
    \end{subfigure}
    \begin{subfigure}[b]{0.32\textwidth}
        \includegraphics[width=\linewidth]{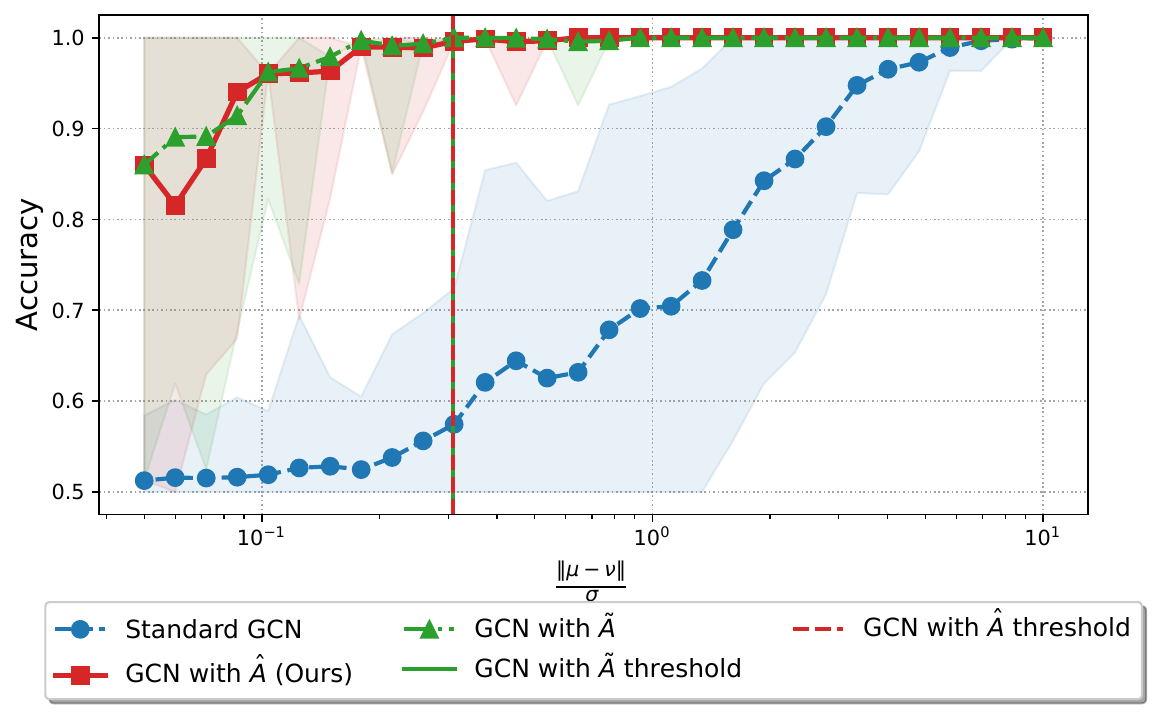}
        \caption{\(k=8\).}
    \end{subfigure}
    \begin{subfigure}[b]{0.32\textwidth}
        \includegraphics[width=\linewidth]{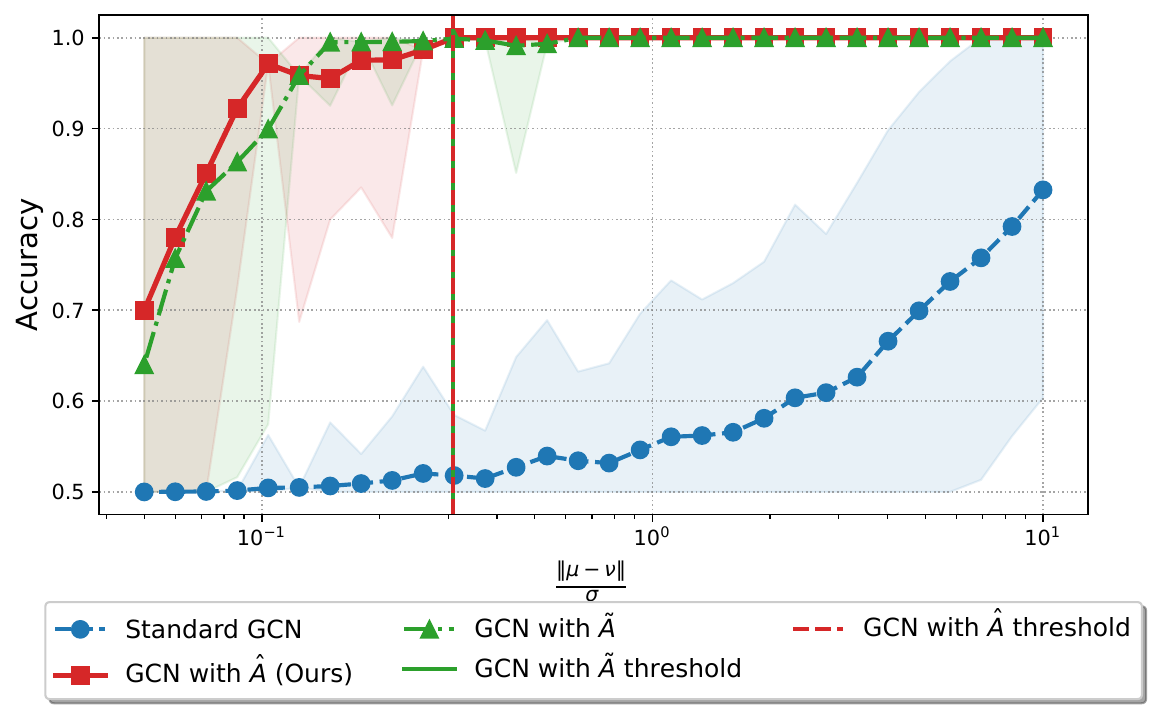}
        \caption{\(k=12\).}
    \end{subfigure}
    \begin{subfigure}[b]{0.32\textwidth}
        \includegraphics[width=\linewidth]{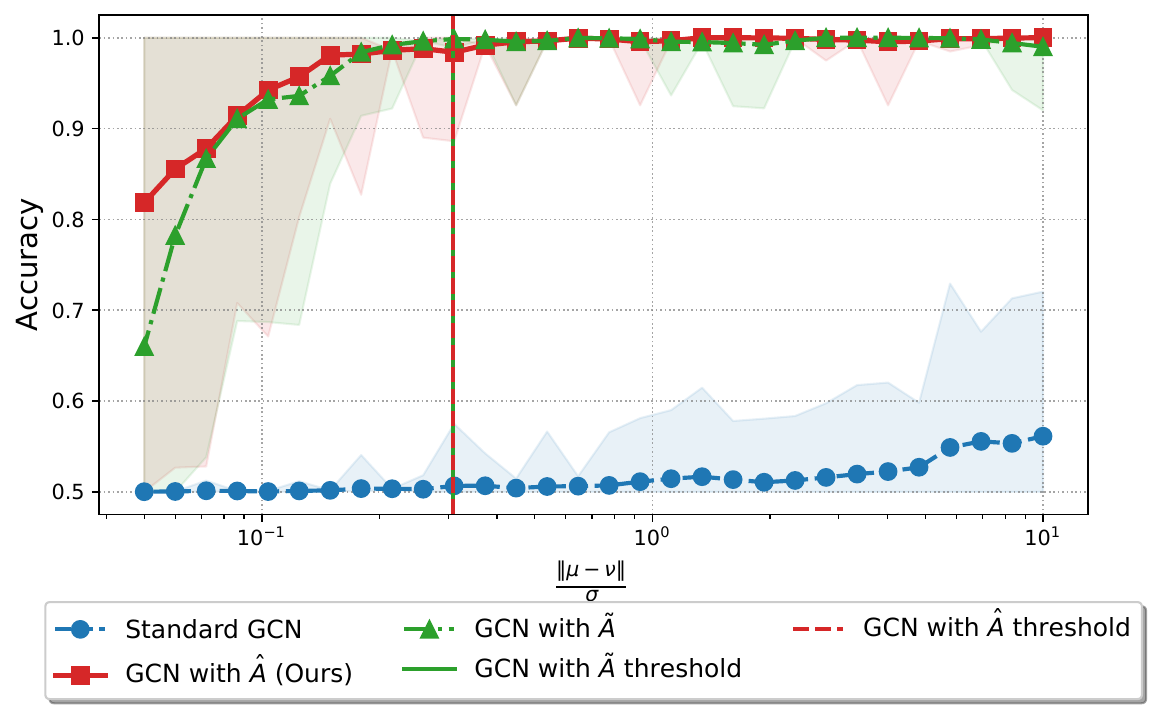}
        \caption{\(k=16\).}
    \end{subfigure}
    \caption{Synthetic CSBM accuracy versus feature SNR for increasing propagation depth.}
    \label{fig:synthetic-sigma}
\end{figure}

\begin{figure}[!ht]
    \centering
    \begin{subfigure}[b]{0.32\textwidth}
        \includegraphics[width=\linewidth]{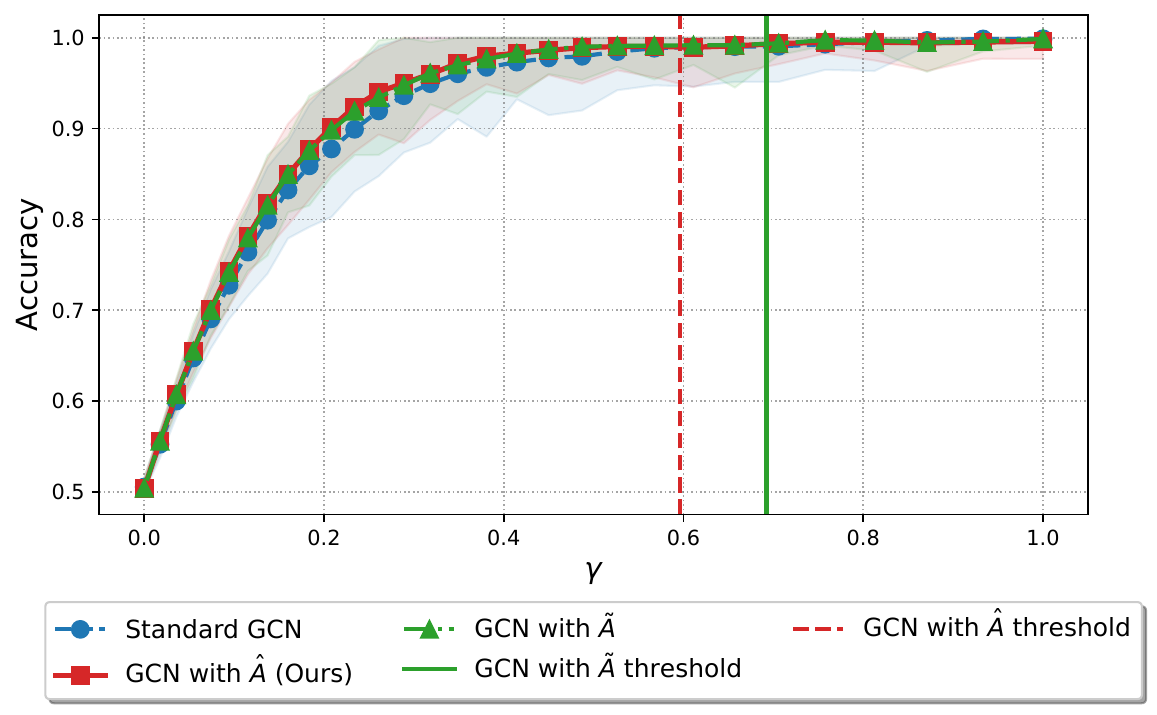}
        \caption{\(k=1\).}
    \end{subfigure}
    \begin{subfigure}[b]{0.32\textwidth}
        \includegraphics[width=\linewidth]{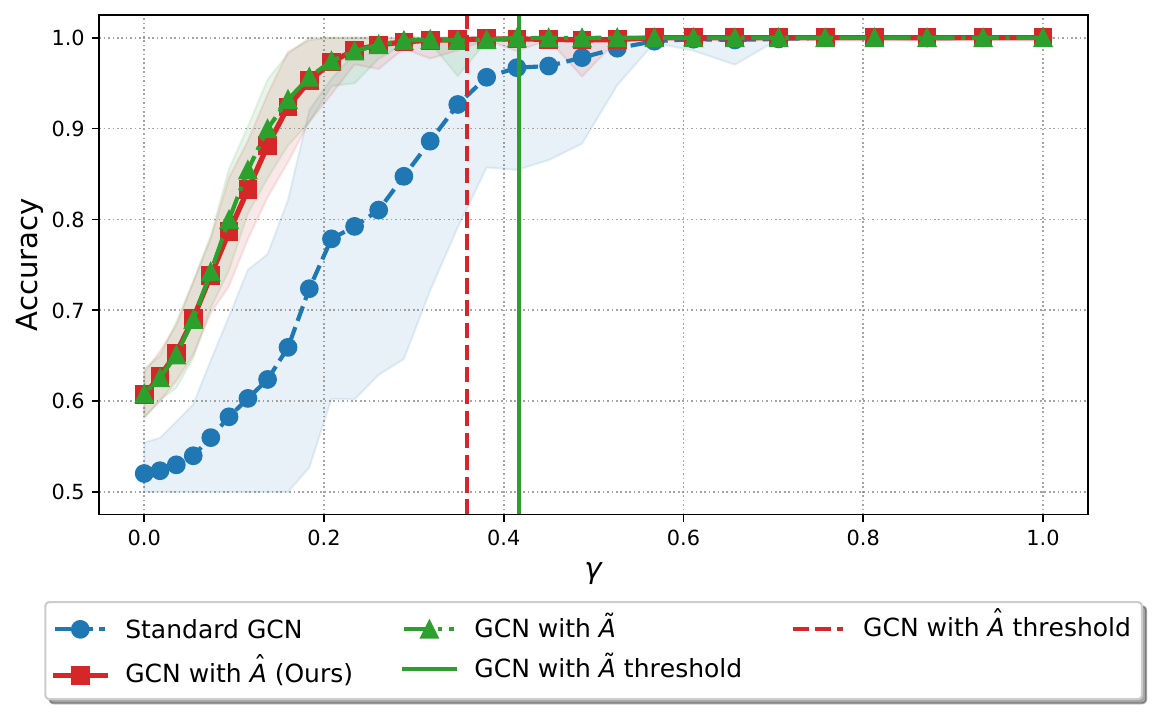}
        \caption{\(k=2\).}
    \end{subfigure}
    \begin{subfigure}[b]{0.32\textwidth}
        \includegraphics[width=\linewidth]{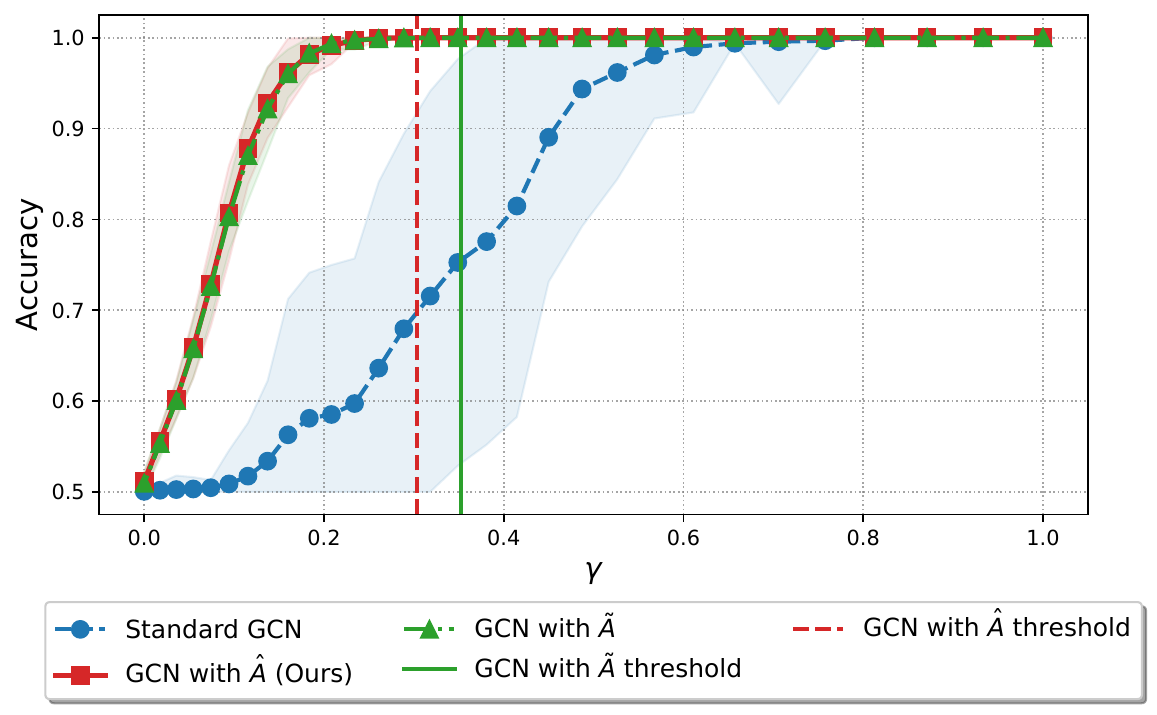}
        \caption{\(k=3\).}
    \end{subfigure}
    \begin{subfigure}[b]{0.32\textwidth}
        \includegraphics[width=\linewidth]{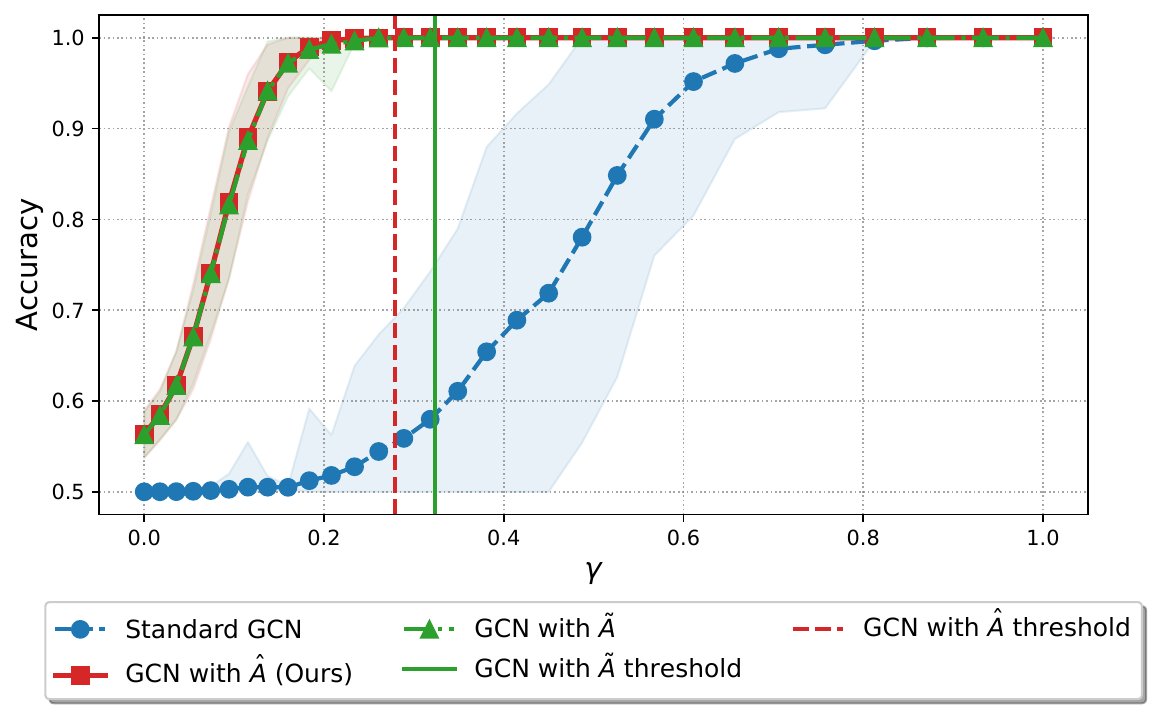}
        \caption{\(k=4\).}
    \end{subfigure}
    \begin{subfigure}[b]{0.32\textwidth}
        \includegraphics[width=\linewidth]{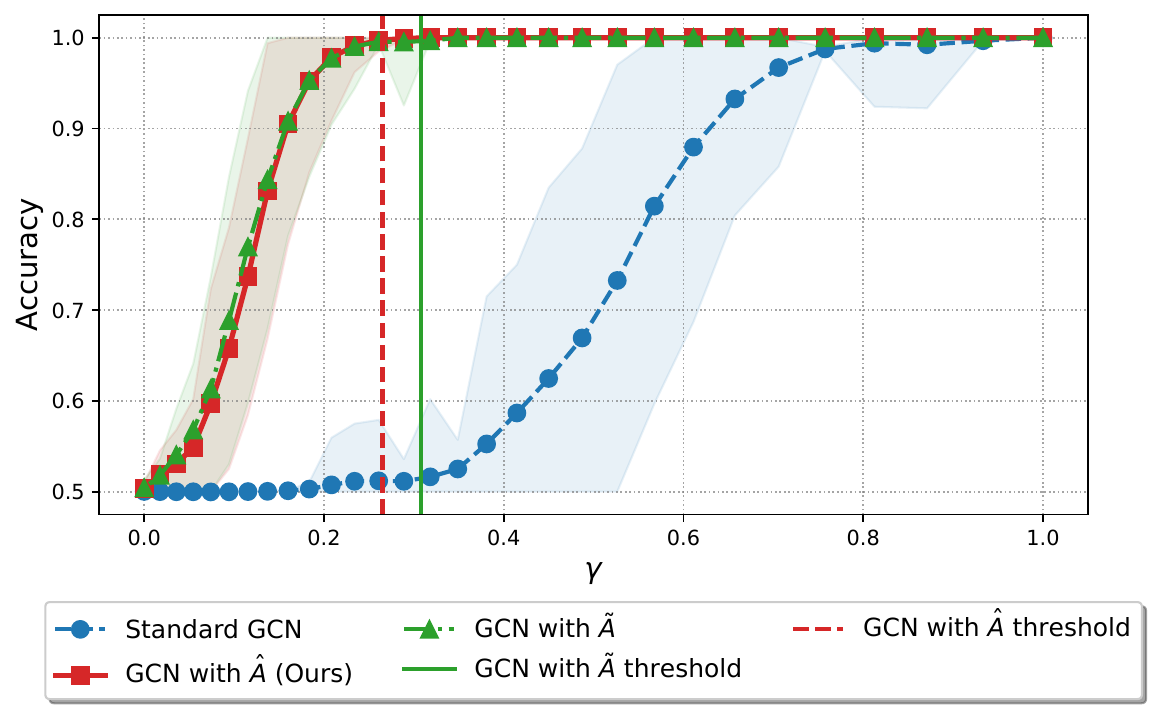}
        \caption{\(k=5\).}
    \end{subfigure}
    \begin{subfigure}[b]{0.32\textwidth}
        \includegraphics[width=\linewidth]{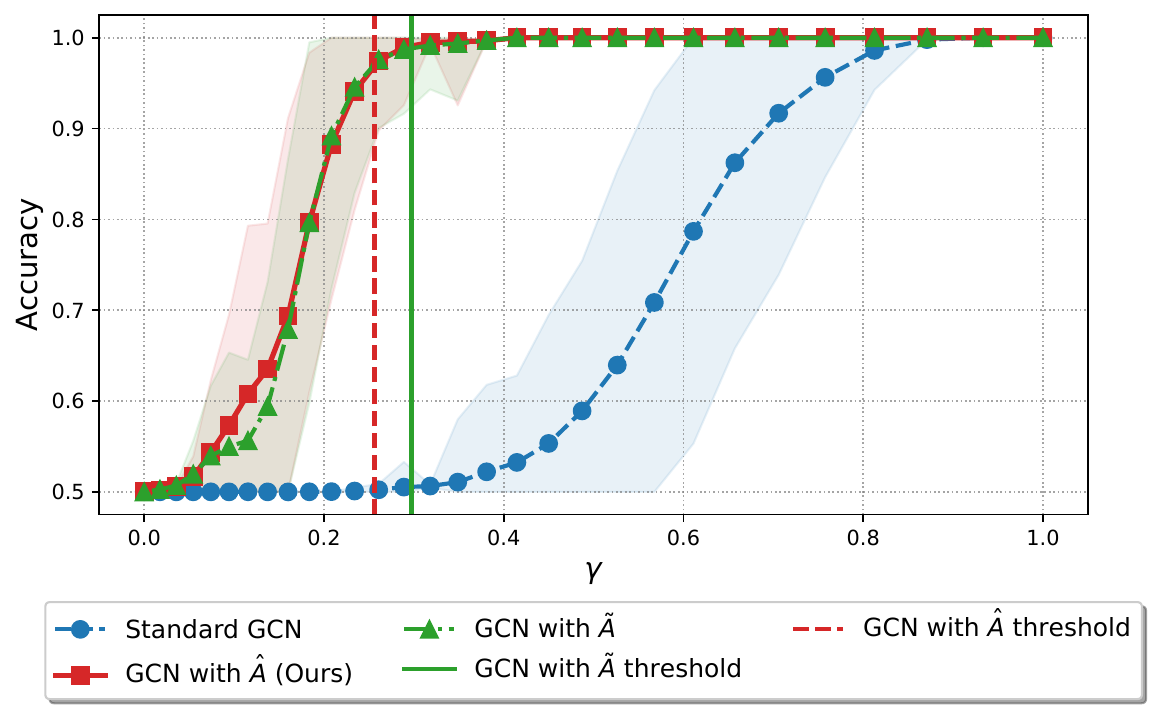}
        \caption{\(k=6\).}
    \end{subfigure}
    \caption{Synthetic CSBM accuracy versus graph signal strength.}
    \label{fig:synthetic-gamma}
\end{figure}

{
\textbf{Baseline comparison.}  Against recent oversmoothing baselines, corrected
propagation shows comparable accuracy trends while preserving the same
stability pattern across both feature-SNR and graph-signal sweeps
(Figures~\ref{fig:sota-synthetic-sigma} and~\ref{fig:sota-synthetic-gamma}).
}

\begin{figure}[!ht]
    \centering
    \begin{subfigure}[b]{0.32\textwidth}
        \includegraphics[width=\linewidth]{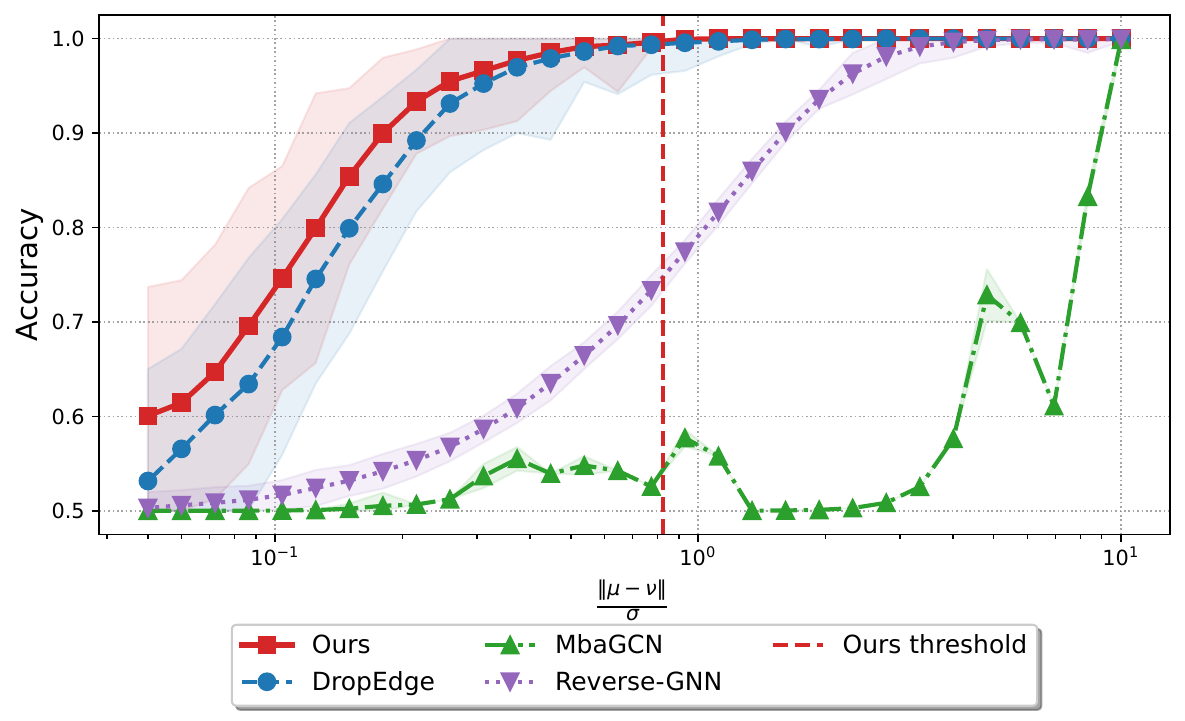}
        \caption{\(k=1\).}
    \end{subfigure}
    \begin{subfigure}[b]{0.32\textwidth}
        \includegraphics[width=\linewidth]{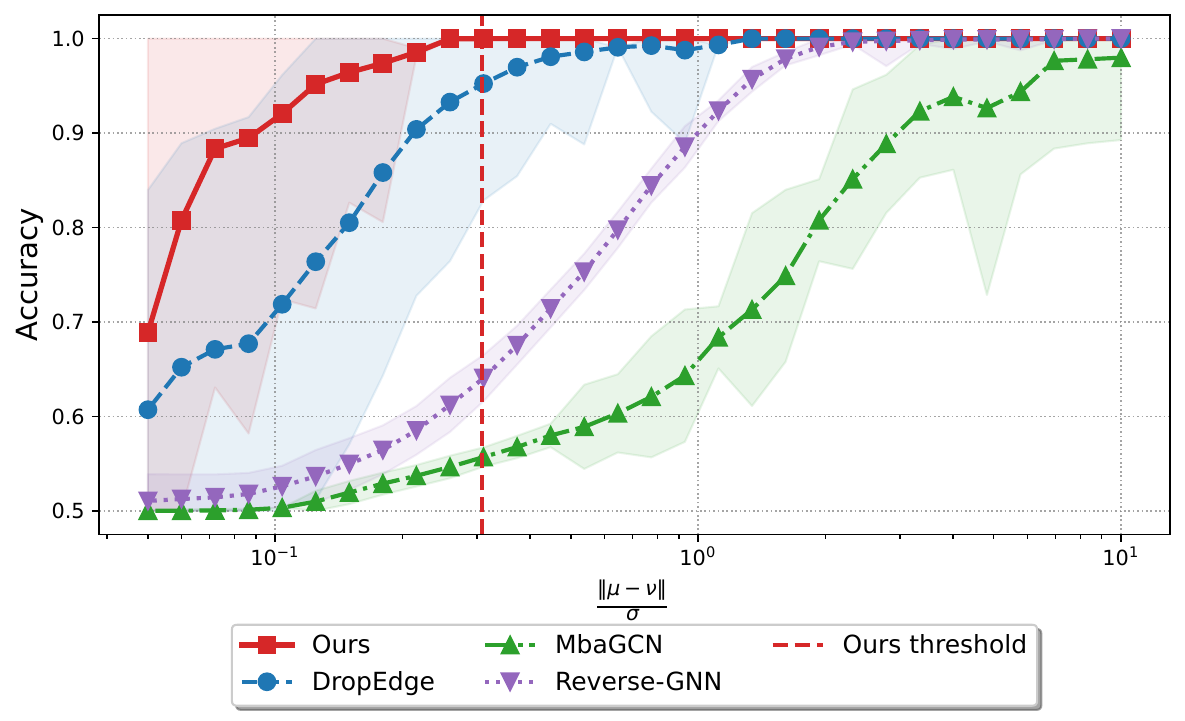}
        \caption{\(k=2\).}
    \end{subfigure}
    \begin{subfigure}[b]{0.32\textwidth}
        \includegraphics[width=\linewidth]{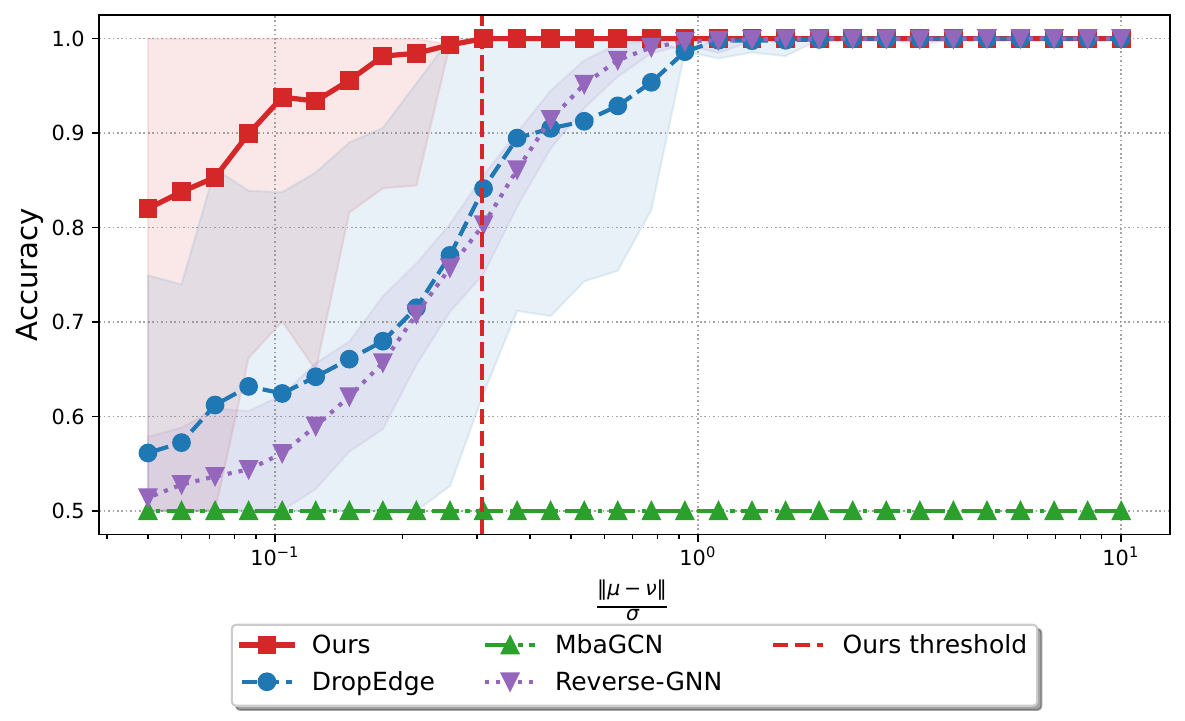}
        \caption{\(k=4\).}
    \end{subfigure}
    \begin{subfigure}[b]{0.32\textwidth}
        \includegraphics[width=\linewidth]{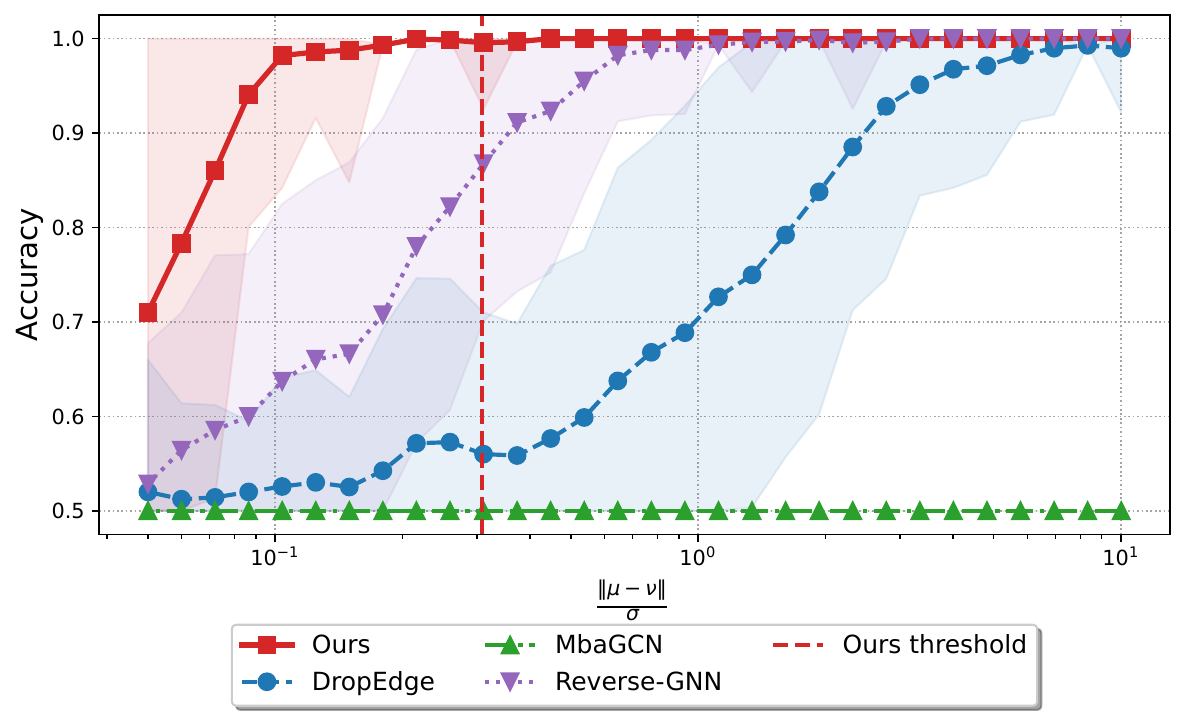}
        \caption{\(k=8\).}
    \end{subfigure}
    \begin{subfigure}[b]{0.32\textwidth}
        \includegraphics[width=\linewidth]{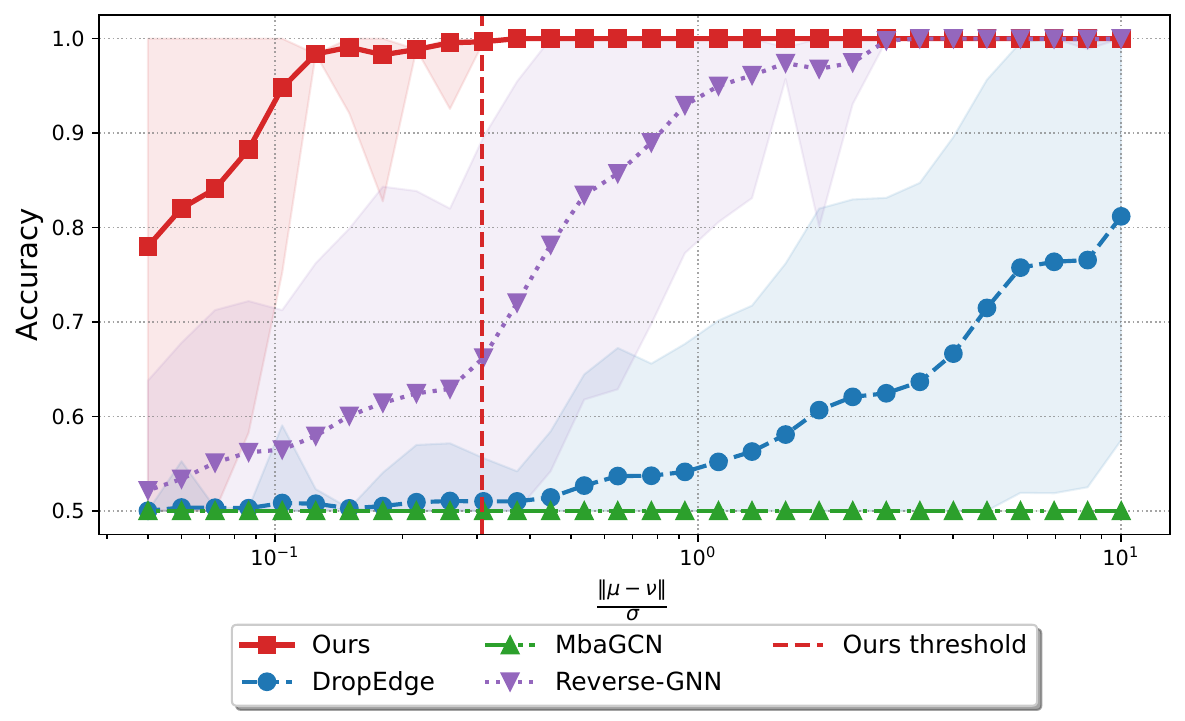}
        \caption{\(k=12\).}
    \end{subfigure}
    \begin{subfigure}[b]{0.32\textwidth}
        \includegraphics[width=\linewidth]{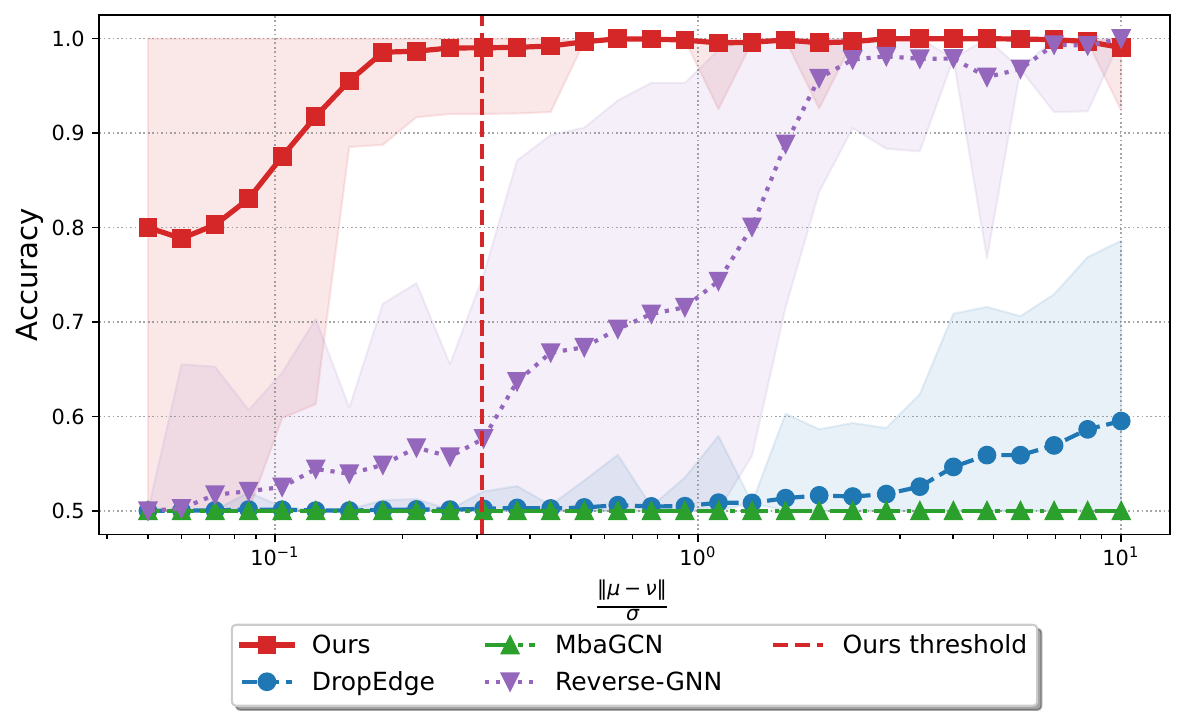}
        \caption{\(k=16\).}
    \end{subfigure}
    \caption{Synthetic CSBM comparison with recent oversmoothing baselines under varying feature SNR.}
    \label{fig:sota-synthetic-sigma}
\end{figure}

\begin{figure}[!ht]
    \centering
    \begin{subfigure}[b]{0.32\textwidth}
        \includegraphics[width=\linewidth]{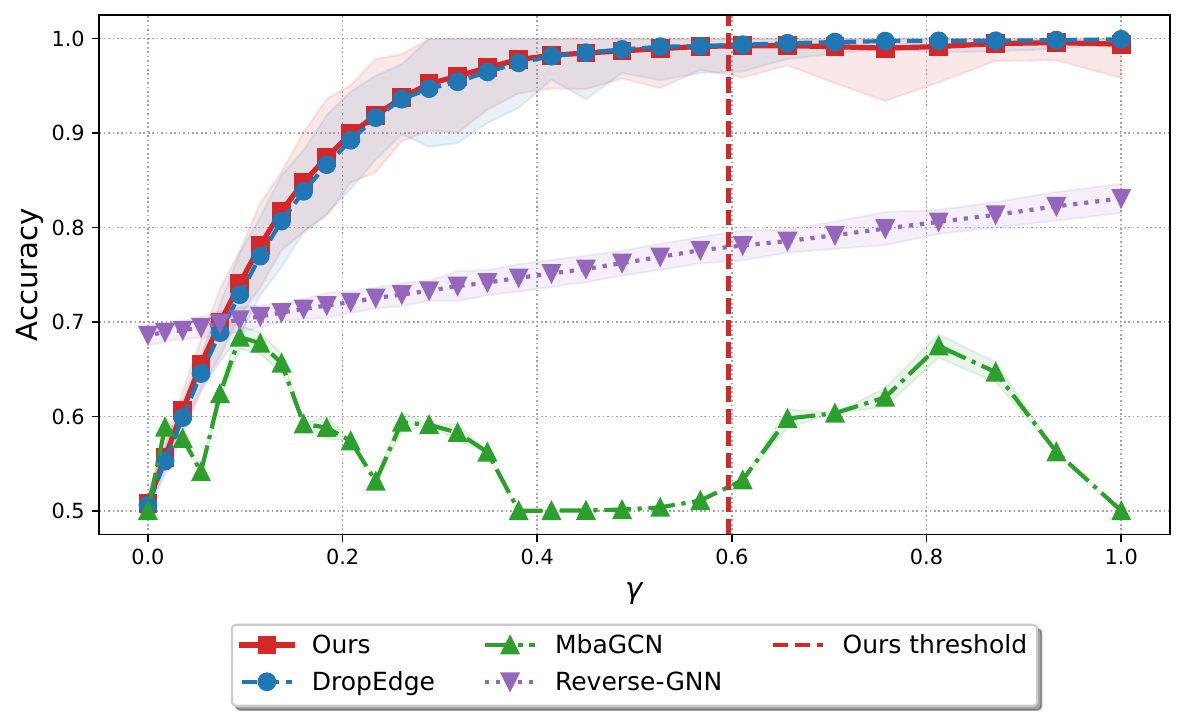}
        \caption{\(k=1\).}
    \end{subfigure}
    \begin{subfigure}[b]{0.32\textwidth}
        \includegraphics[width=\linewidth]{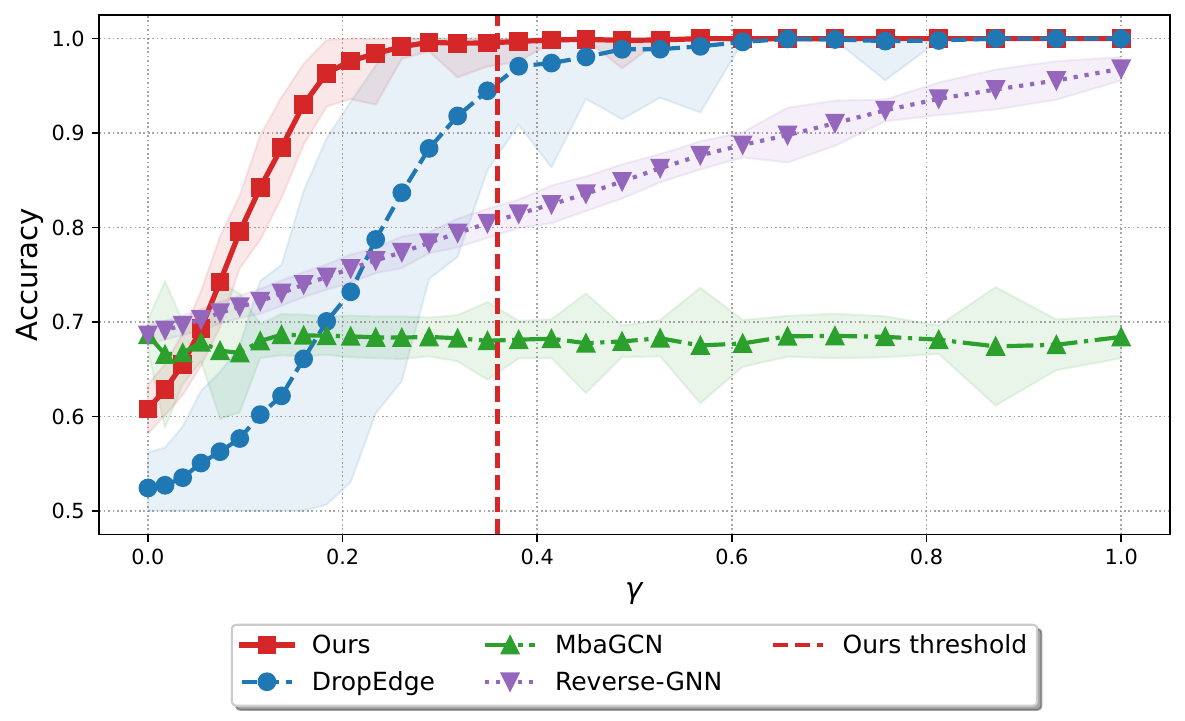}
        \caption{\(k=2\).}
    \end{subfigure}
    \begin{subfigure}[b]{0.32\textwidth}
        \includegraphics[width=\linewidth]{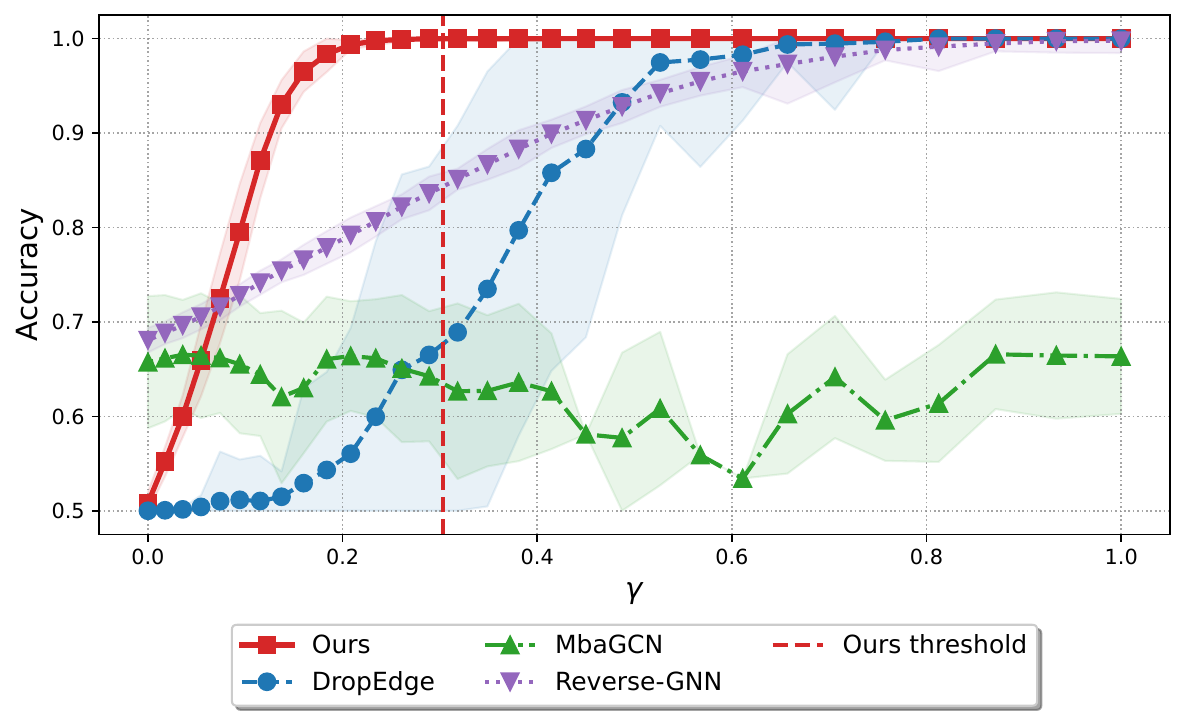}
        \caption{\(k=3\).}
    \end{subfigure}
    \begin{subfigure}[b]{0.32\textwidth}
        \includegraphics[width=\linewidth]{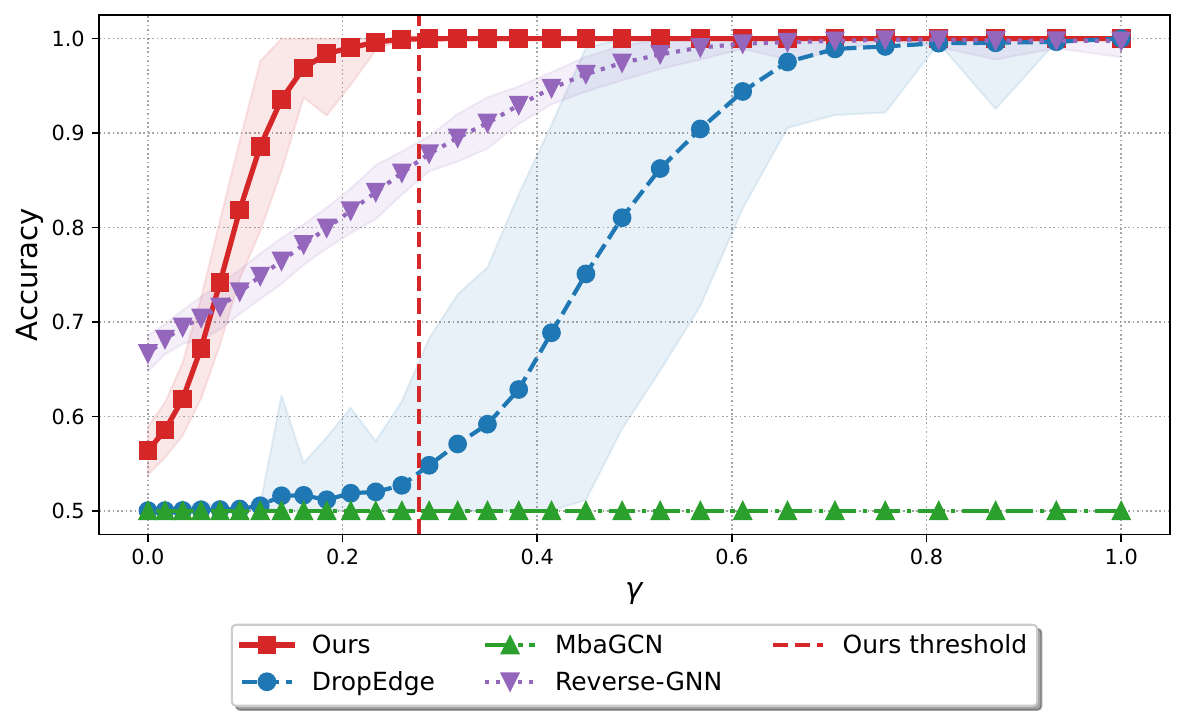}
        \caption{\(k=4\).}
    \end{subfigure}
    \begin{subfigure}[b]{0.32\textwidth}
        \includegraphics[width=\linewidth]{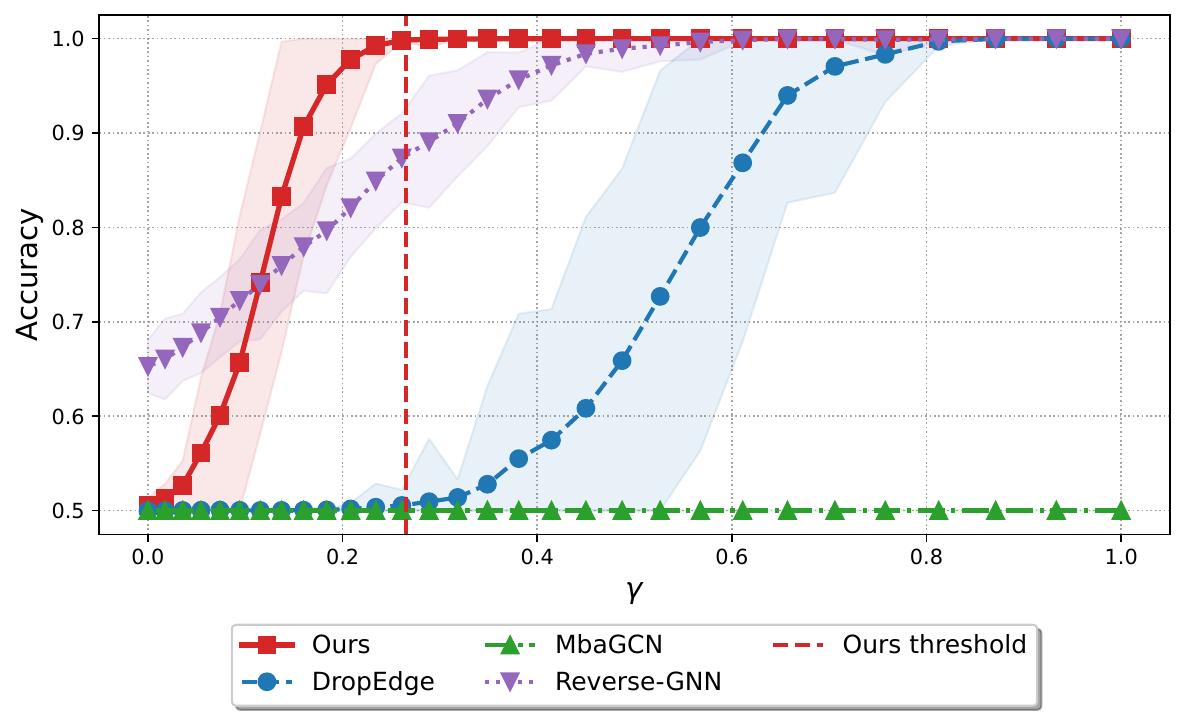}
        \caption{\(k=5\).}
    \end{subfigure}
    \begin{subfigure}[b]{0.32\textwidth}
        \includegraphics[width=\linewidth]{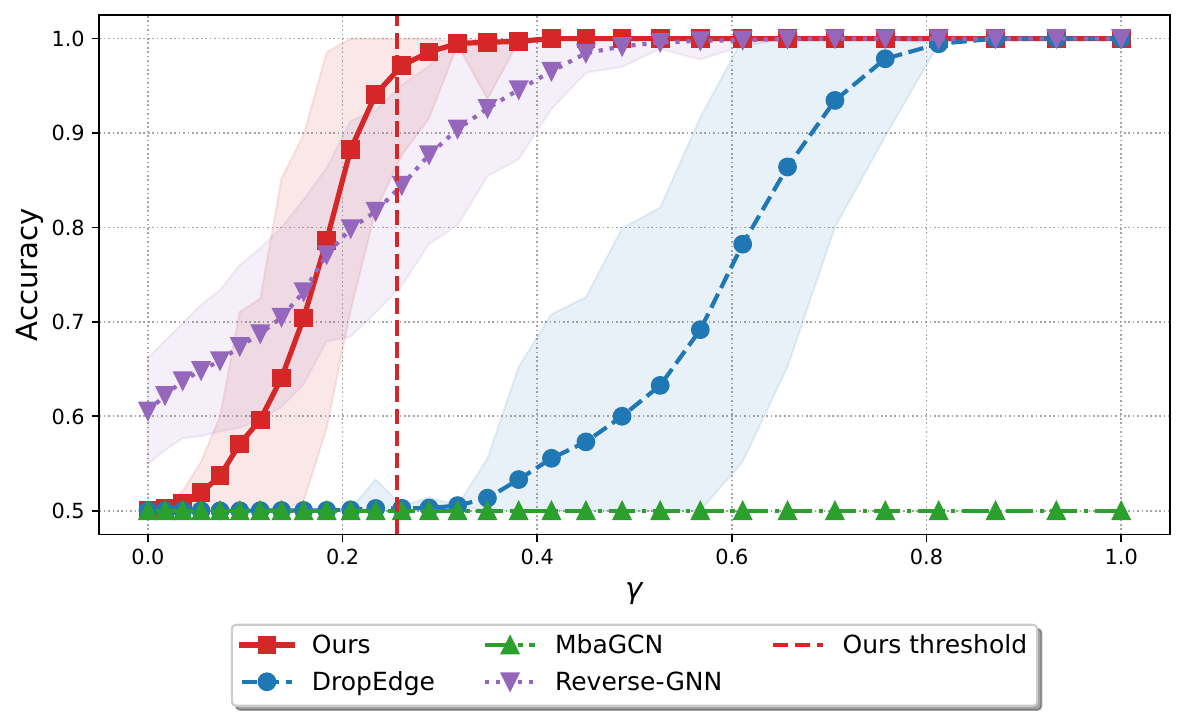}
        \caption{\(k=6\).}
    \end{subfigure}
    \caption{Synthetic CSBM comparison with recent oversmoothing baselines under varying graph signal.}
    \label{fig:sota-synthetic-gamma}
\end{figure}

\subsection{Real Data}
{
We evaluate on \emph{Cora}, \emph{CiteSeer}, \emph{PubMed}, \emph{Reddit}, \emph{ogbn-arxiv}, and
\emph{ogbn-products}.  Table~\ref{tab:dataset_stats} reports dataset statistics.
These experiments assess whether the correction is associated with reduced
depth-dependent degradation on heterogeneous multi-class graphs.
}

\begin{table}[h]
    \centering
    \caption{Statistics of real-world datasets used in our experiments.}
    \label{tab:dataset_stats}
    \resizebox{\textwidth}{!}{
    \begin{tabular}{lcccccc}
    \toprule
    \textbf{Dataset} & \textbf{Nodes} & \textbf{Edges} & \textbf{Features} & \textbf{Classes} & \textbf{Class Imbalance} \\
    \midrule
    \emph{Cora} & 2,708 & 10,556 & 1,433 & 7 & Mild \\
    \emph{CiteSeer} & 3,327 & 9,104 & 3,703 & 6 & Mild \\
    \emph{PubMed} & 19,717 & 88,648 & 500 & 3 & Moderate \\
    \emph{Reddit} & 232,965 & 114,615,892 & 602 & 41 & High \\
    \emph{ogbn-arxiv} & 169,343 & 1,166,243 & 128 & 40 & High \\
    \emph{ogbn-products} & 2,449,029 & 61,859,140 & 100 & 47 & High \\
    \bottomrule
    \end{tabular}
    }
\end{table}

{
\textbf{Balanced two-cluster subsets.}
To align with the balanced CSBM theory, we select the two largest classes in
each dataset and downsample the majority class to obtain a 1:1 class ratio.
In these balanced subsets, corrected propagation maintains more stable accuracy
as depth increases (Figure~\ref{fig:real-sample-layers}), consistent with the
predicted effect of removing the stationary component.
}

\begin{figure}[!ht]
    \centering
    \begin{subfigure}[b]{0.32\textwidth}
        \includegraphics[width=\linewidth]{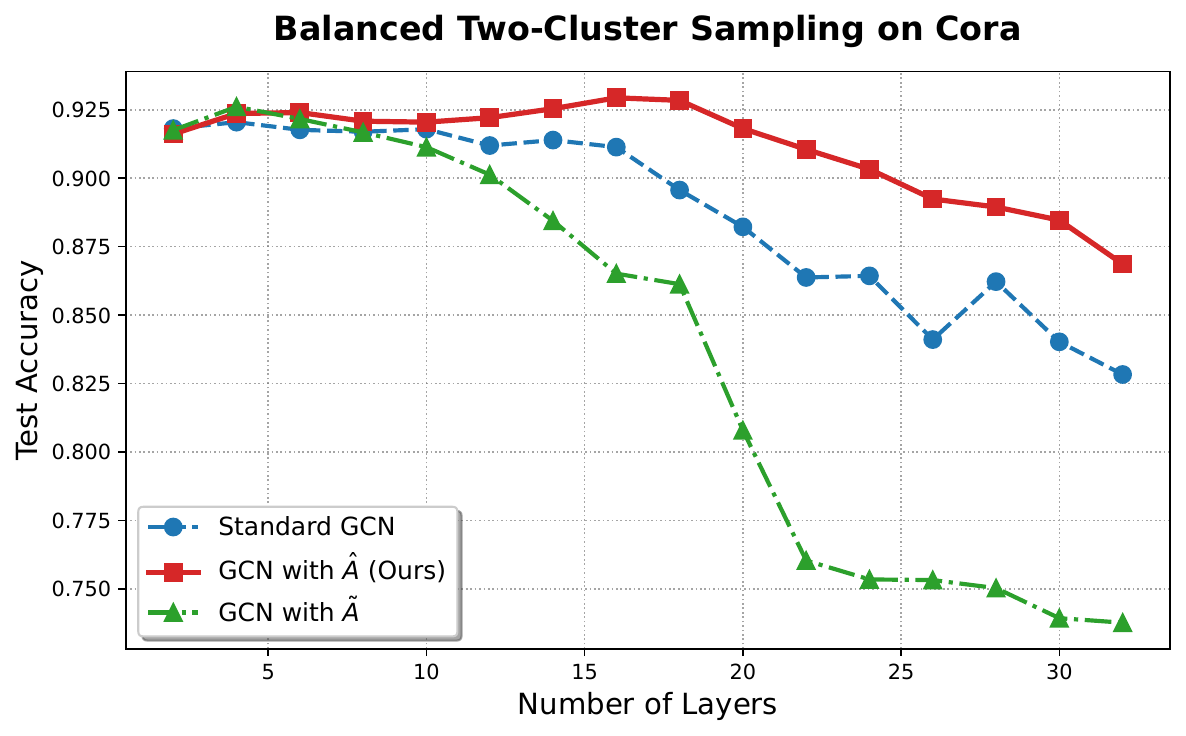}
        \caption{Cora.}
    \end{subfigure}
    \begin{subfigure}[b]{0.32\textwidth}
        \includegraphics[width=\linewidth]{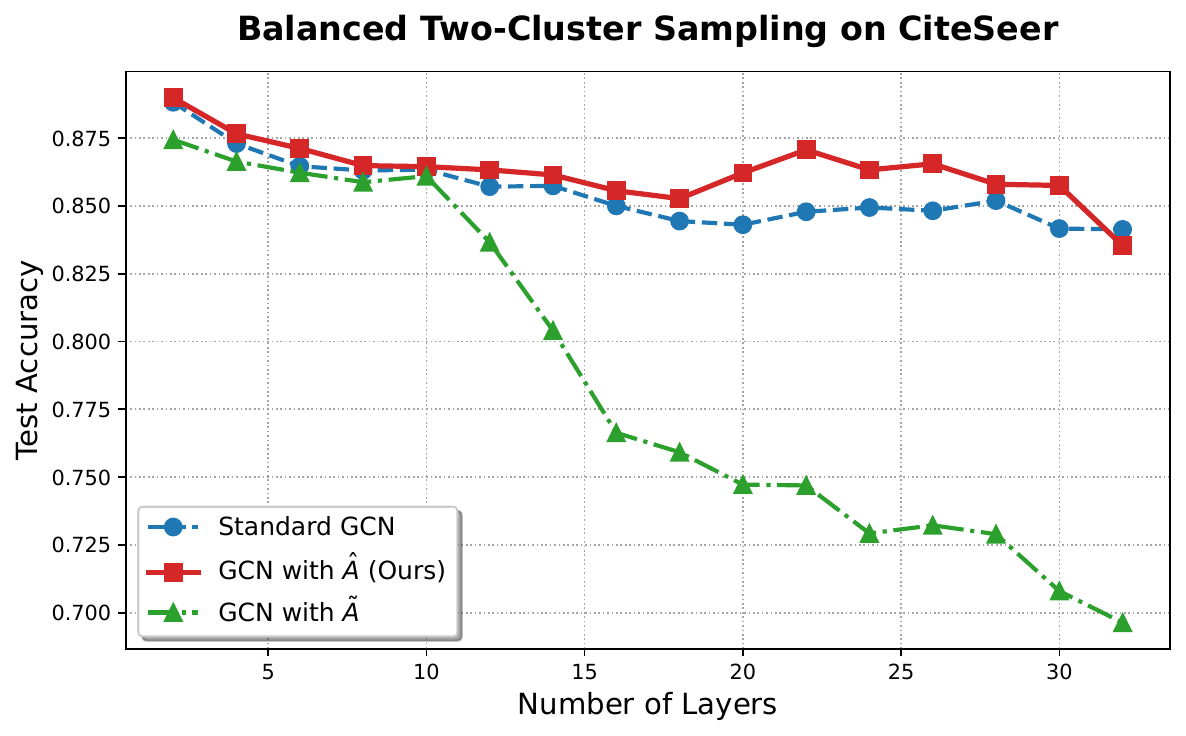}
        \caption{CiteSeer.}
    \end{subfigure}
    \begin{subfigure}[b]{0.32\textwidth}
        \includegraphics[width=\linewidth]{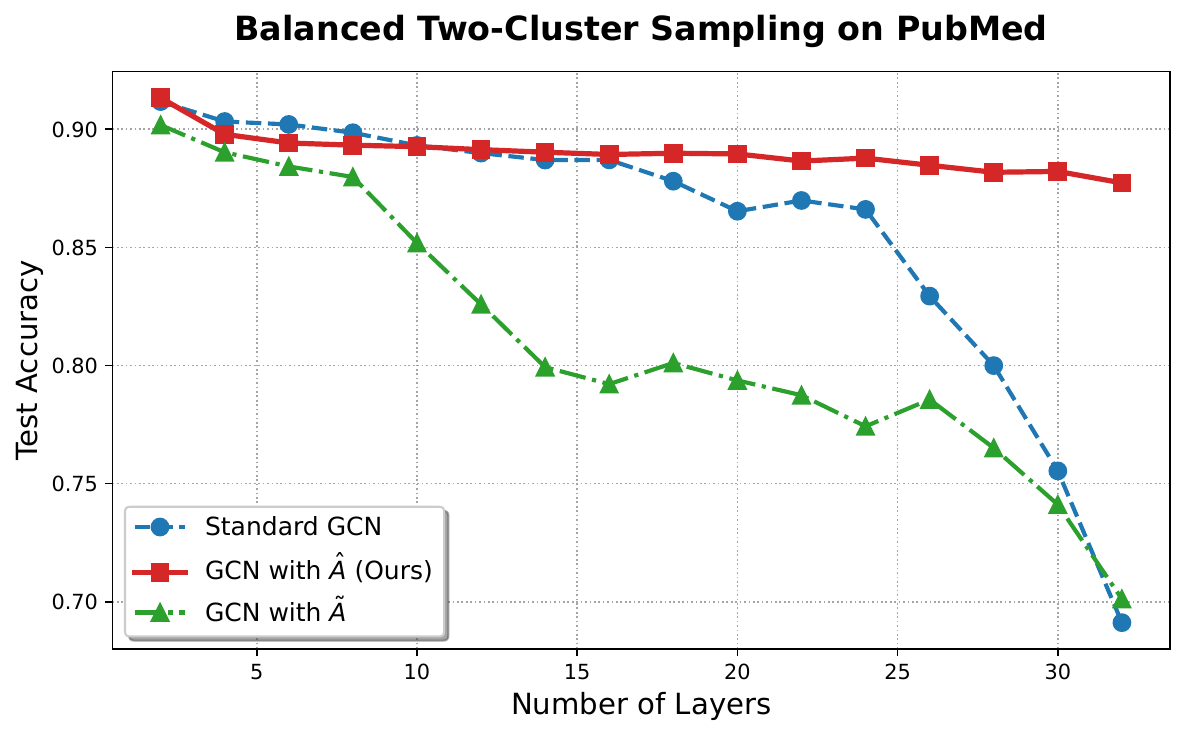}
        \caption{PubMed.}
    \end{subfigure}
    \begin{subfigure}[b]{0.32\textwidth}
        \includegraphics[width=\linewidth]{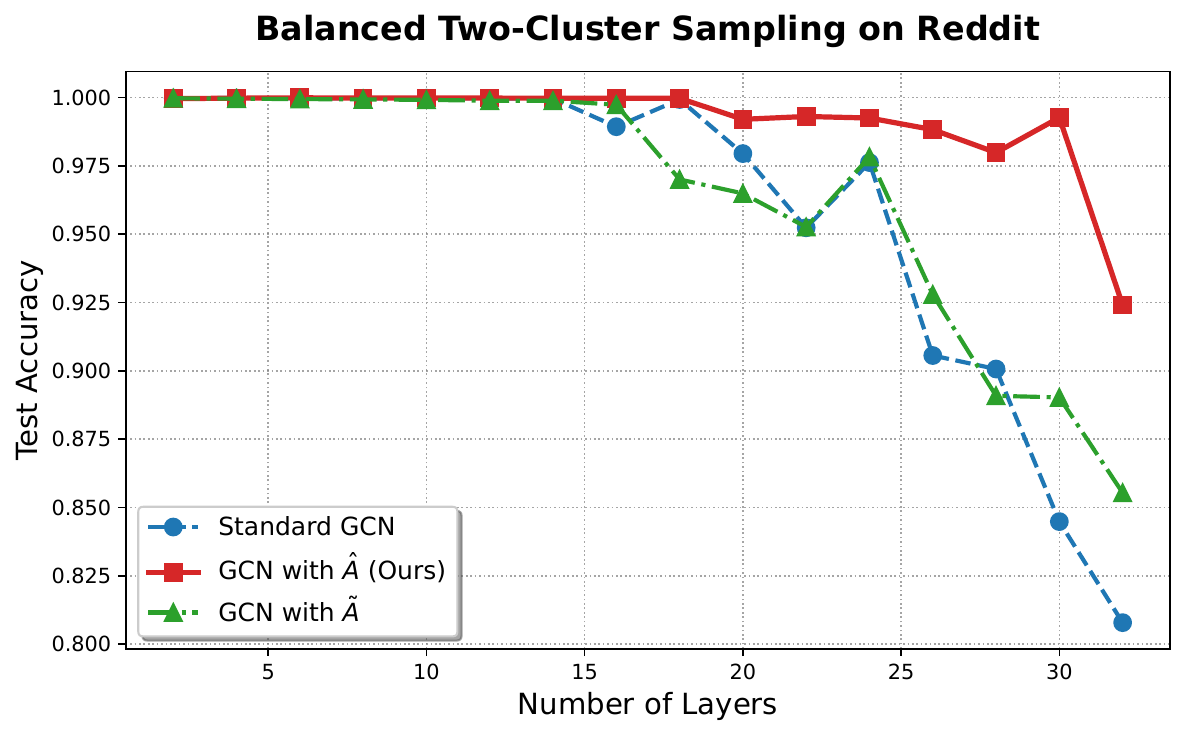}
        \caption{Reddit.}
    \end{subfigure}
    \begin{subfigure}[b]{0.32\textwidth}
        \includegraphics[width=\linewidth]{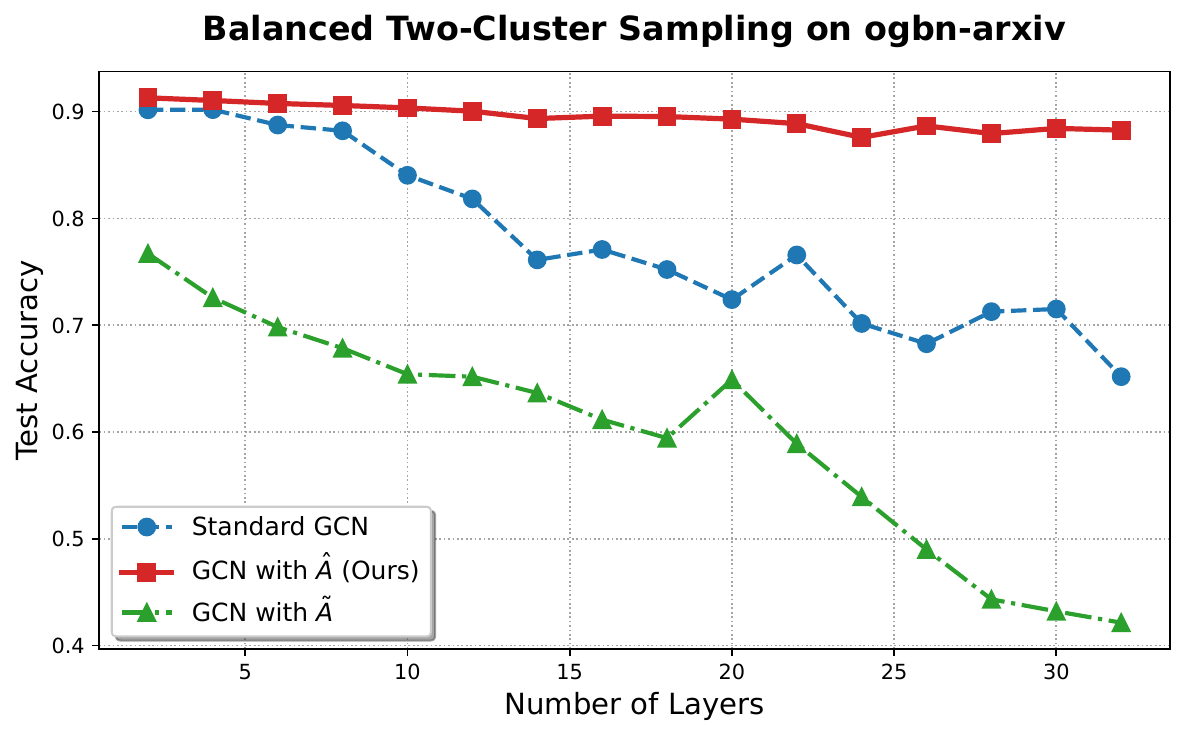}
        \caption{ogbn-arxiv.}
    \end{subfigure}
    \begin{subfigure}[b]{0.32\textwidth}
        \includegraphics[width=\linewidth]{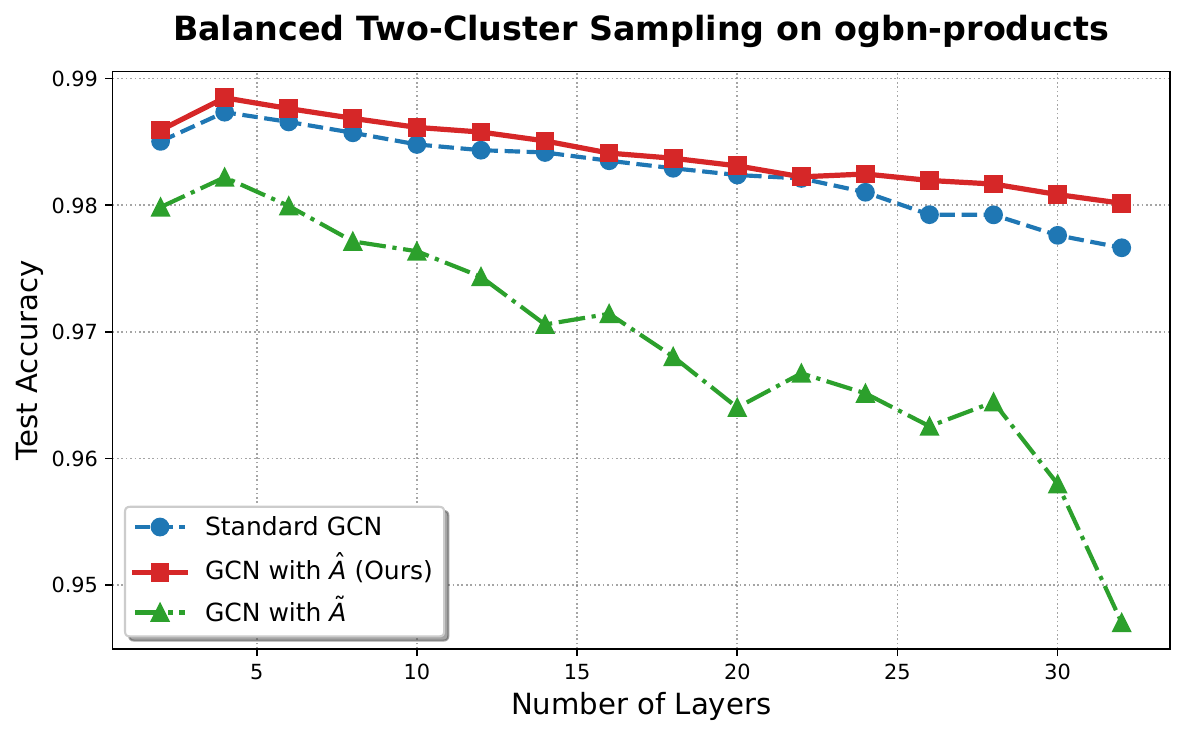}
        \caption{ogbn-products.}
    \end{subfigure}
    \caption{Accuracy versus propagation depth on balanced two-cluster sampled subsets.}
    \label{fig:real-sample-layers}
\end{figure}

{
\textbf{Full Dataset Evaluation.}
On the full datasets, we observe a smaller depth-dependent accuracy drop for
corrected propagation than for standard GCN propagation
(Figure~\ref{fig:real}).  The effect is smaller than on balanced subsets but
remains consistent with the spectral-collapse mechanism.
}

\begin{figure}[!ht]
    \centering
    \begin{subfigure}[b]{0.32\textwidth}
        \includegraphics[width=\linewidth]{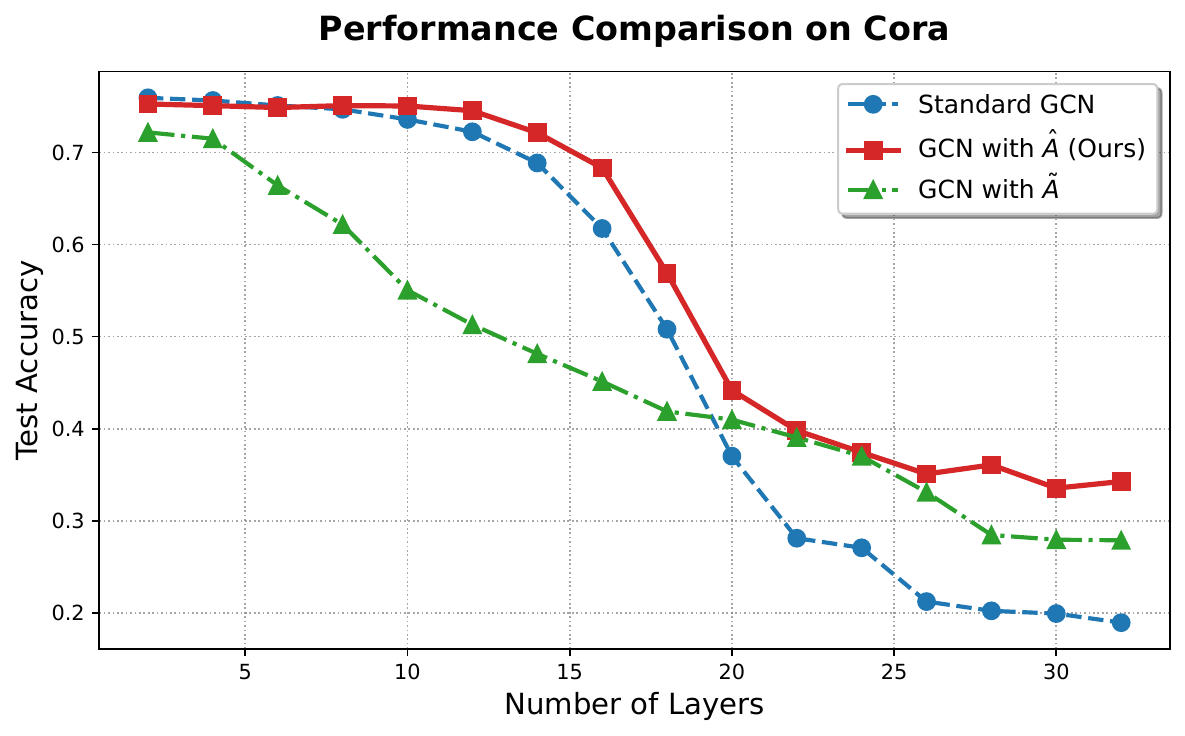}
        \caption{Cora.}
    \end{subfigure}
    \begin{subfigure}[b]{0.32\textwidth}
        \includegraphics[width=\linewidth]{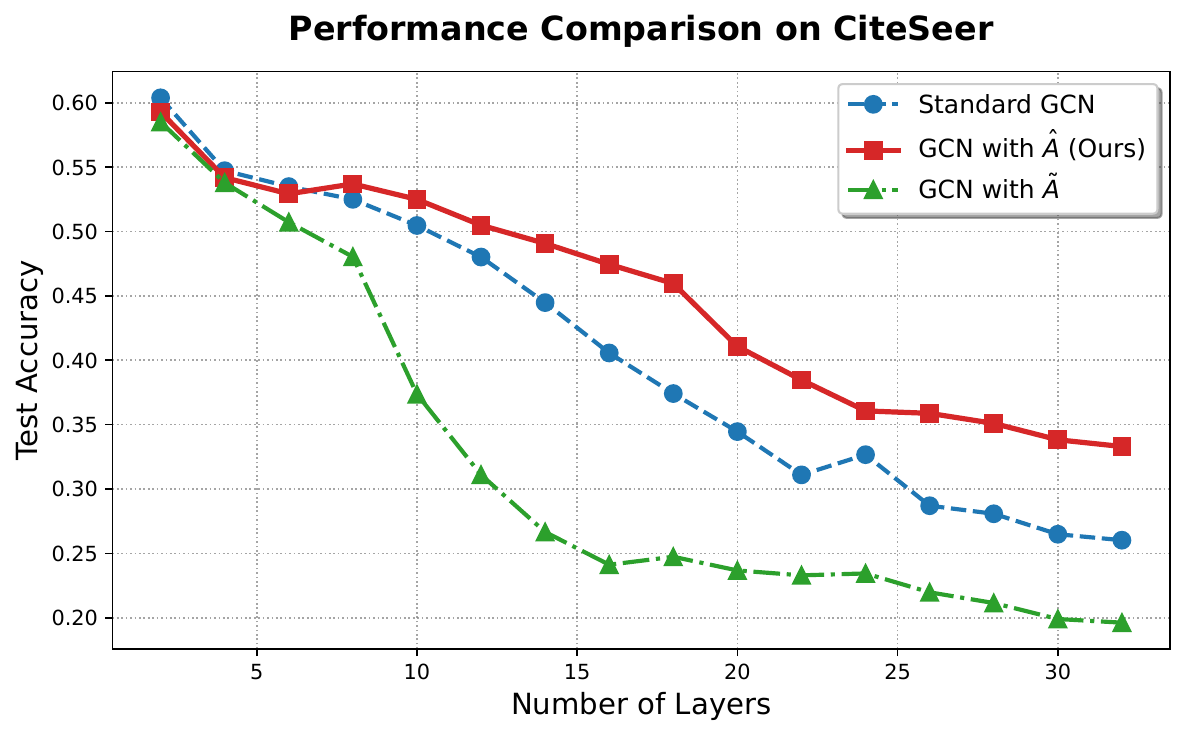}
        \caption{CiteSeer.}
    \end{subfigure}
    \begin{subfigure}[b]{0.32\textwidth}
        \includegraphics[width=\linewidth]{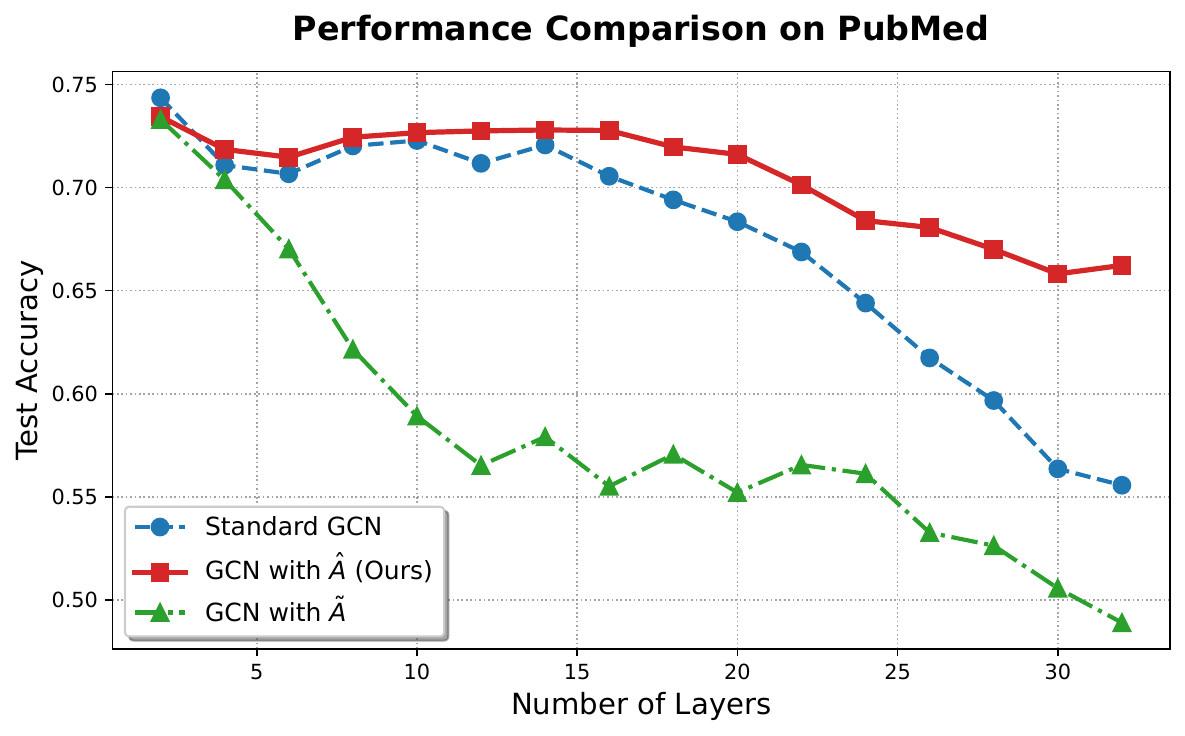}
        \caption{PubMed.}
    \end{subfigure}
    \begin{subfigure}[b]{0.32\textwidth}
        \includegraphics[width=\linewidth]{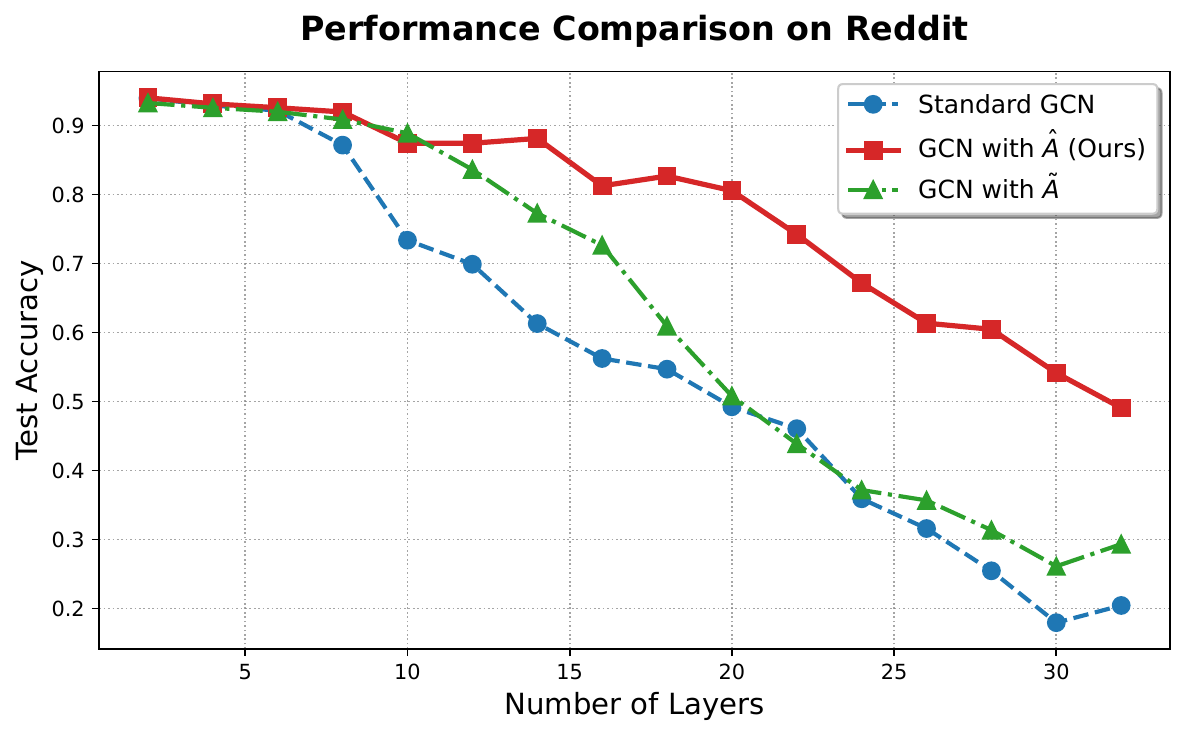}
        \caption{Reddit.}
    \end{subfigure}
    \begin{subfigure}[b]{0.32\textwidth}
        \includegraphics[width=\linewidth]{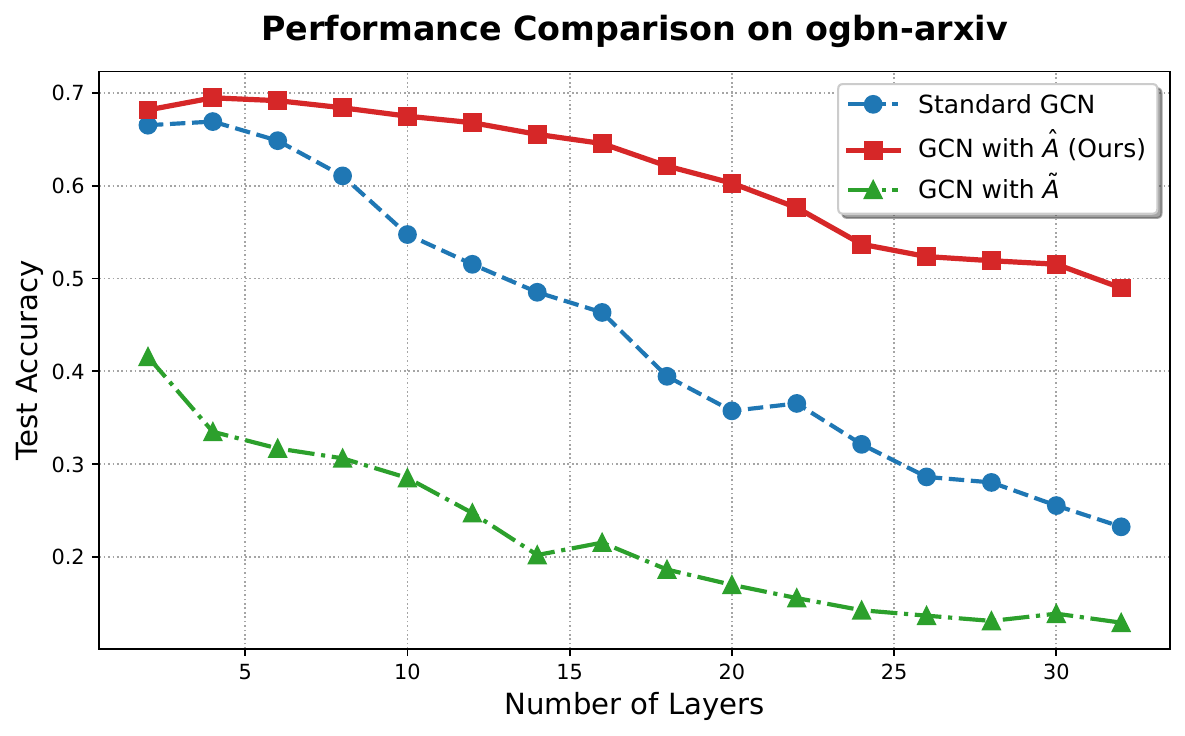}
        \caption{ogbn-arxiv.}
    \end{subfigure}
    \begin{subfigure}[b]{0.32\textwidth}
        \includegraphics[width=\linewidth]{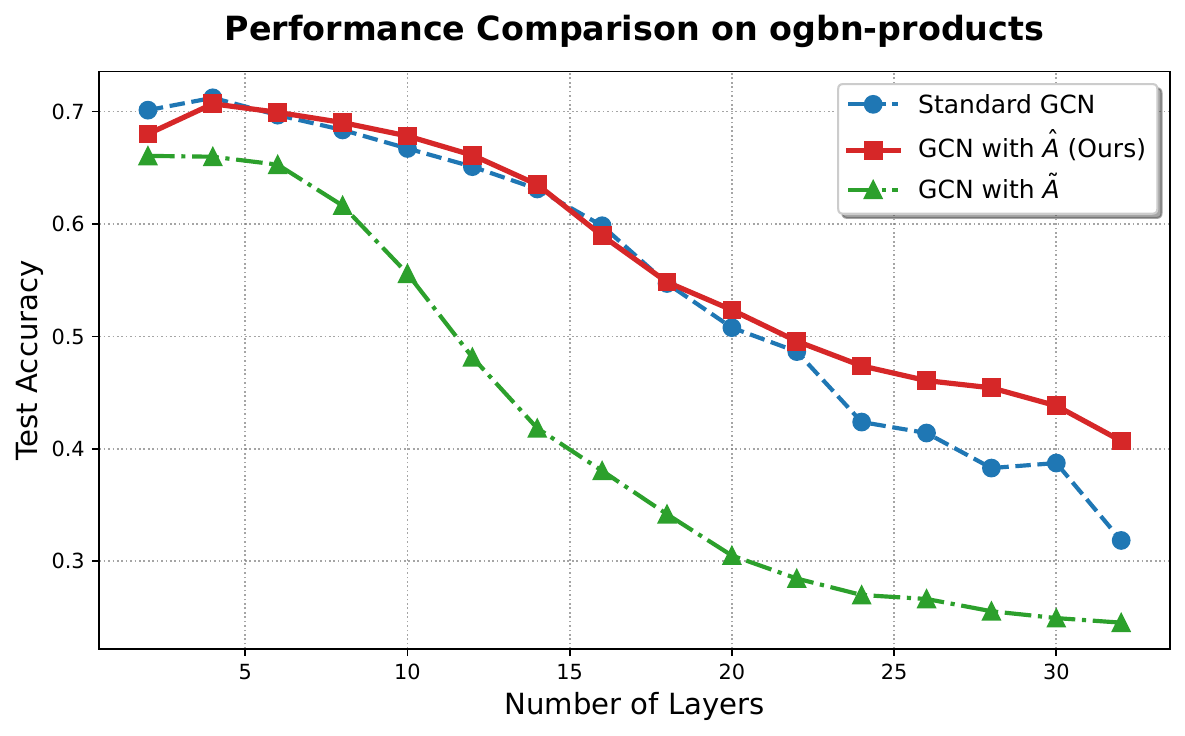}
        \caption{ogbn-products.}
    \end{subfigure}
    \caption{Accuracy versus propagation depth on full real-world datasets.}
    \label{fig:real}
\end{figure}

{
\textbf{Baseline comparison.}  Corrected propagation has comparable accuracy to
recent baselines on both balanced subsets and full datasets
(Figures~\ref{fig:sota-real-sample} and~\ref{fig:sota-real-full}), while using a
training-free operator-level correction.
}

\begin{figure}[!ht]
    \centering
    \begin{subfigure}[b]{0.32\textwidth}
        \includegraphics[width=\linewidth]{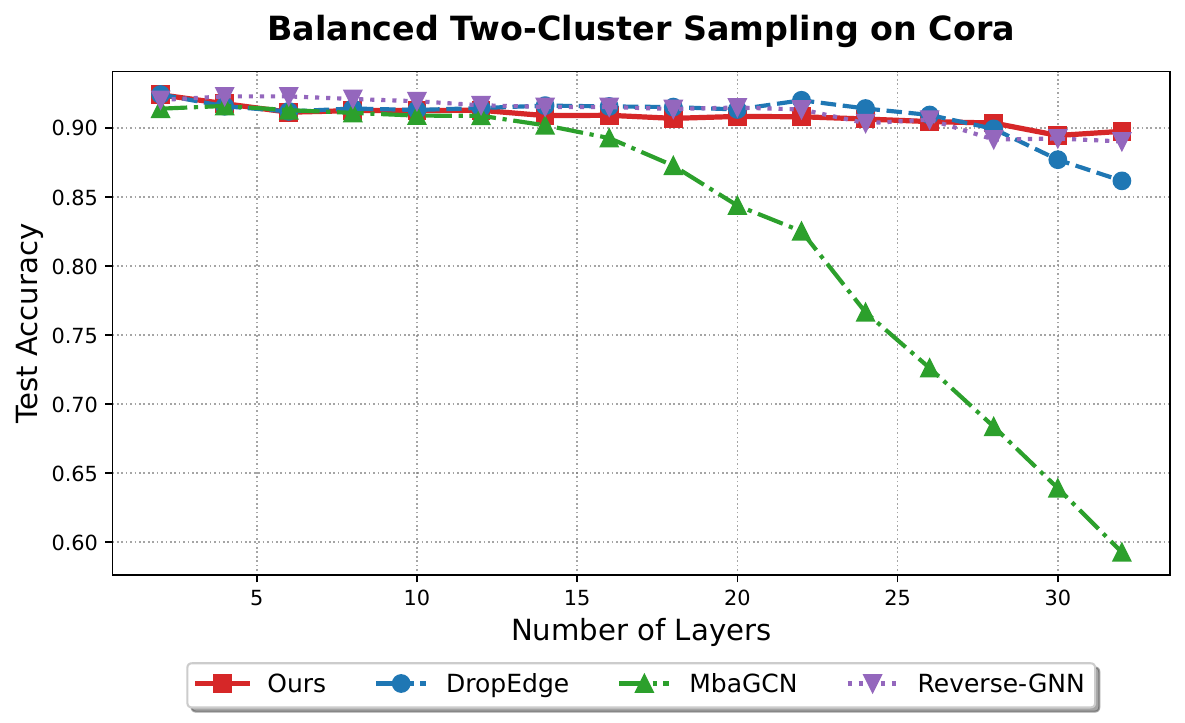}
        \caption{Cora.}
    \end{subfigure}
    \begin{subfigure}[b]{0.32\textwidth}
        \includegraphics[width=\linewidth]{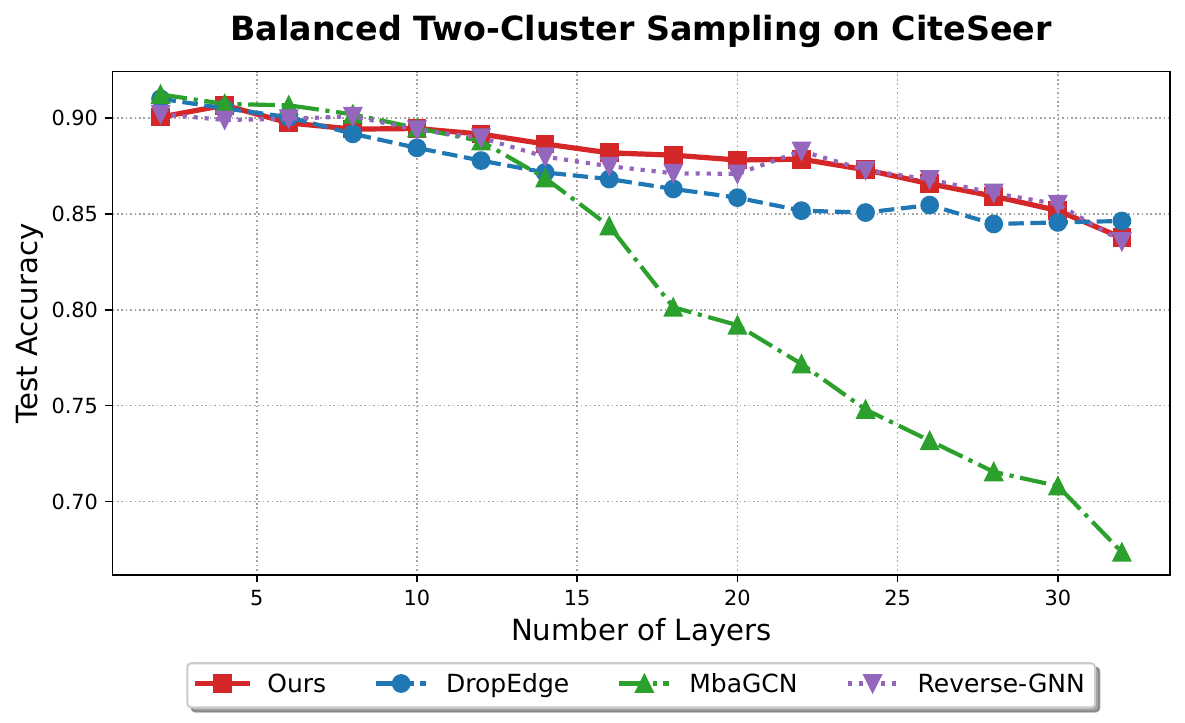}
        \caption{CiteSeer.}
    \end{subfigure}
    \begin{subfigure}[b]{0.32\textwidth}
        \includegraphics[width=\linewidth]{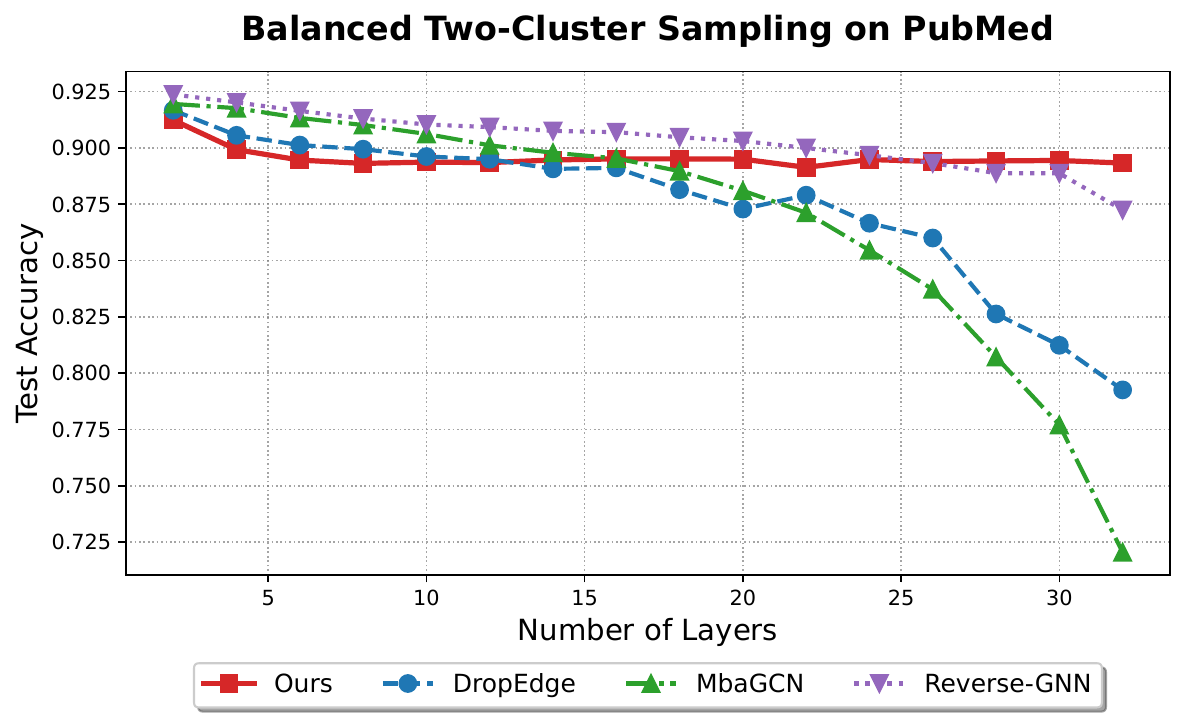}
        \caption{PubMed.}
    \end{subfigure}
    \begin{subfigure}[b]{0.32\textwidth}
        \includegraphics[width=\linewidth]{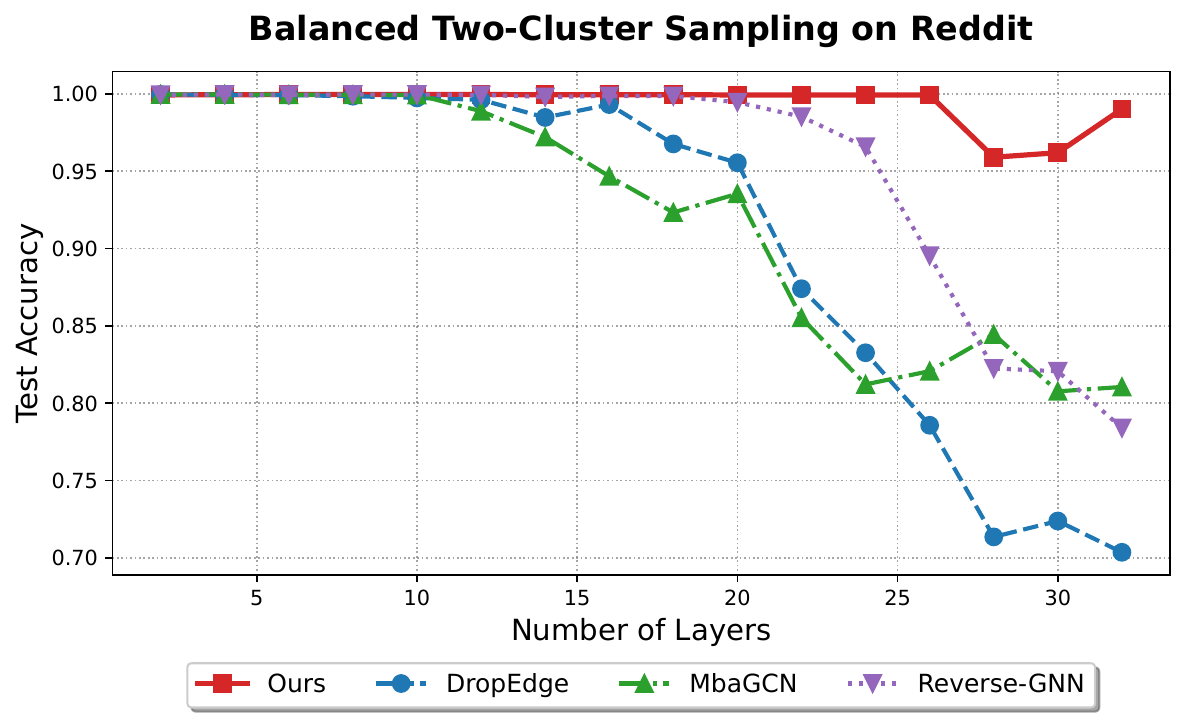}
        \caption{Reddit.}
    \end{subfigure}
    \begin{subfigure}[b]{0.32\textwidth}
        \includegraphics[width=\linewidth]{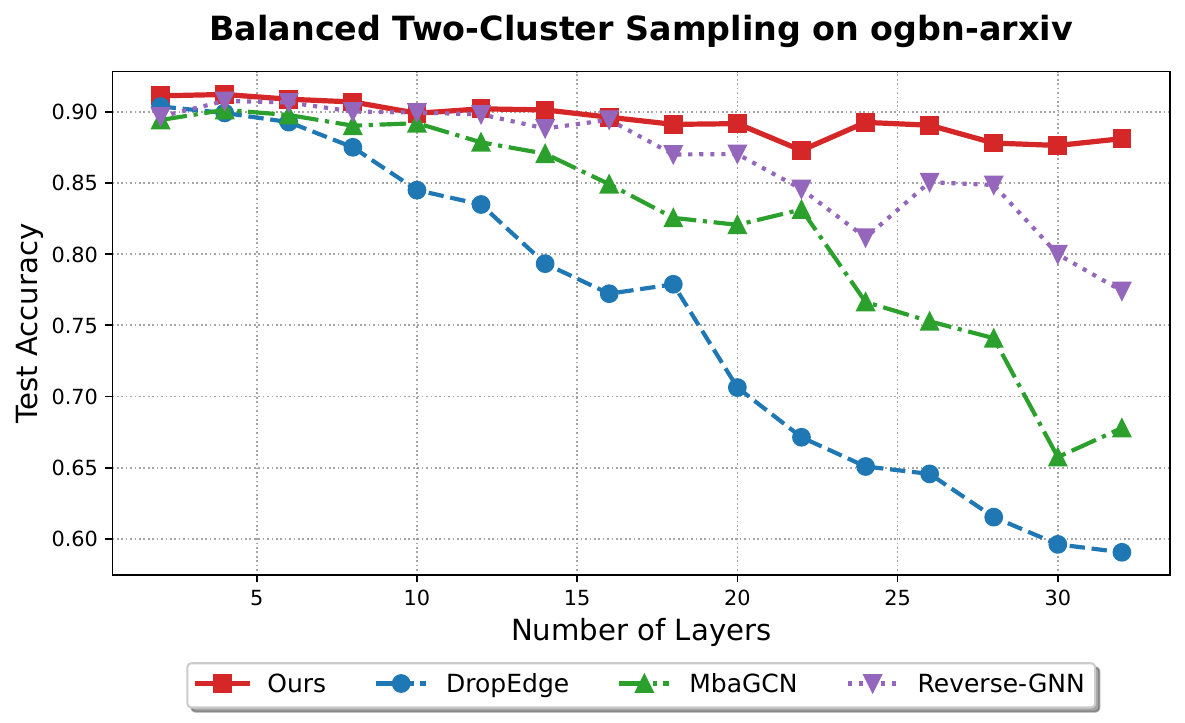}
        \caption{ogbn-arxiv.}
    \end{subfigure}
    \begin{subfigure}[b]{0.32\textwidth}
        \includegraphics[width=\linewidth]{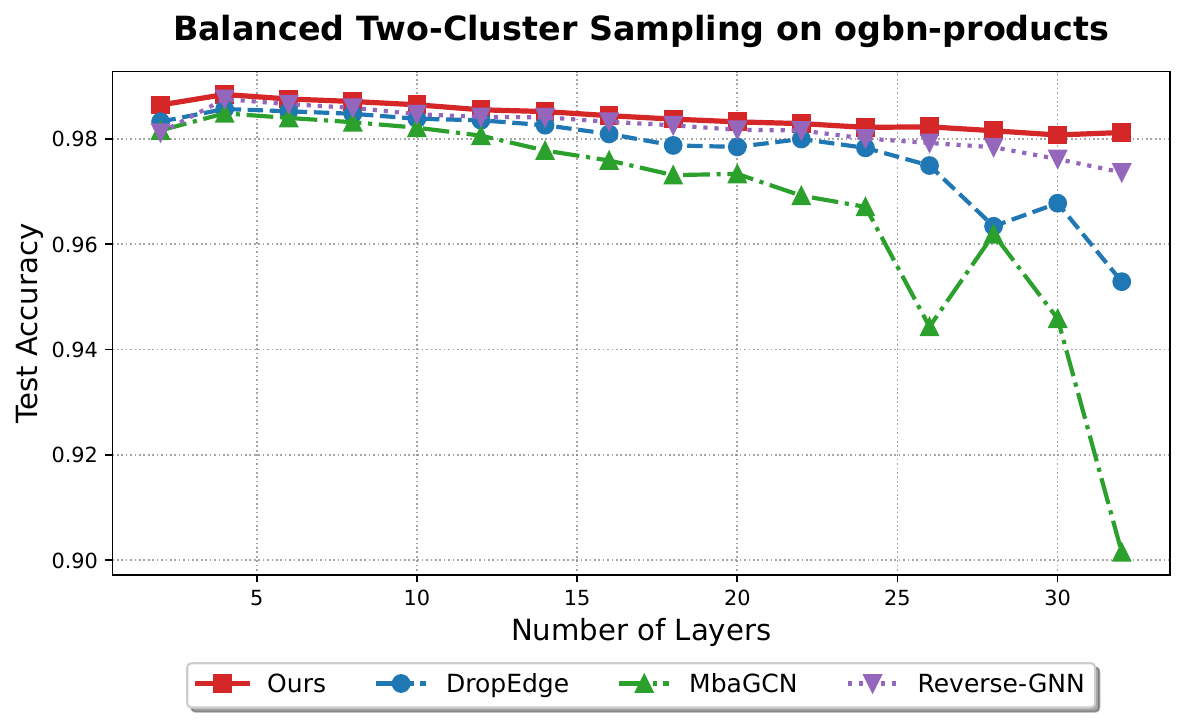}
        \caption{ogbn-products.}
    \end{subfigure}
    \caption{Balanced-subset comparison with recent oversmoothing baselines.}
    \label{fig:sota-real-sample}
\end{figure}

\begin{figure}[!ht]
    \centering
    \begin{subfigure}[b]{0.32\textwidth}
        \includegraphics[width=\linewidth]{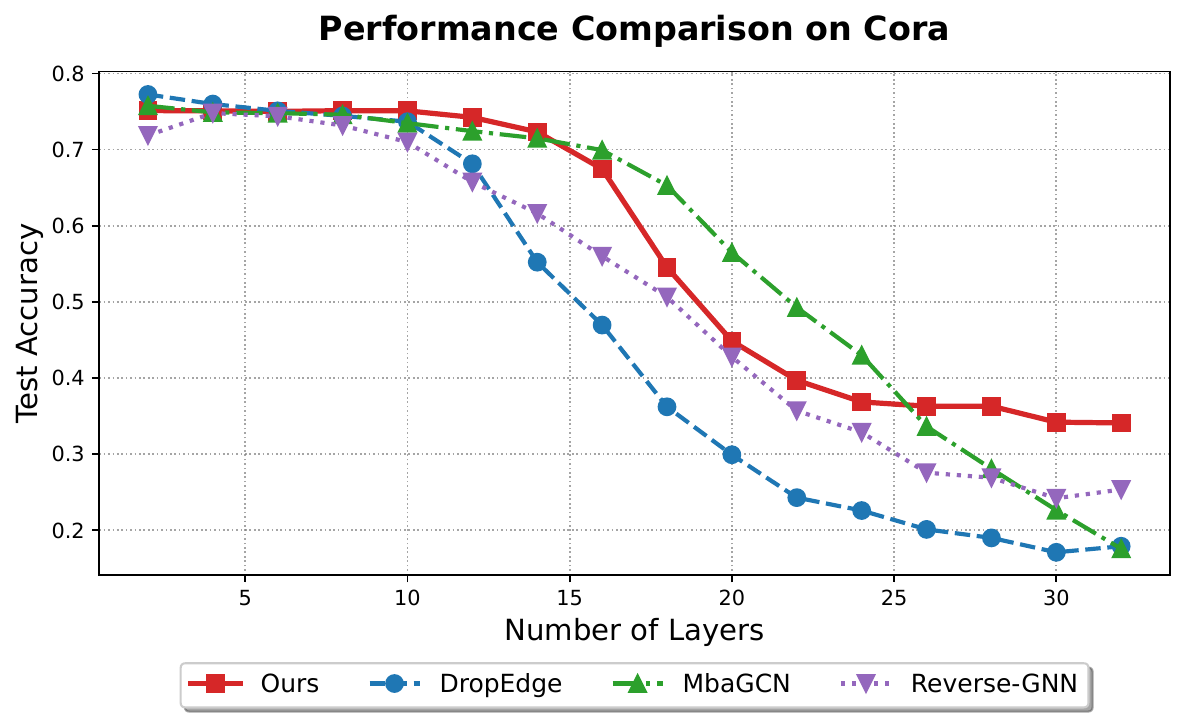}
        \caption{Cora.}
    \end{subfigure}
    \begin{subfigure}[b]{0.32\textwidth}
        \includegraphics[width=\linewidth]{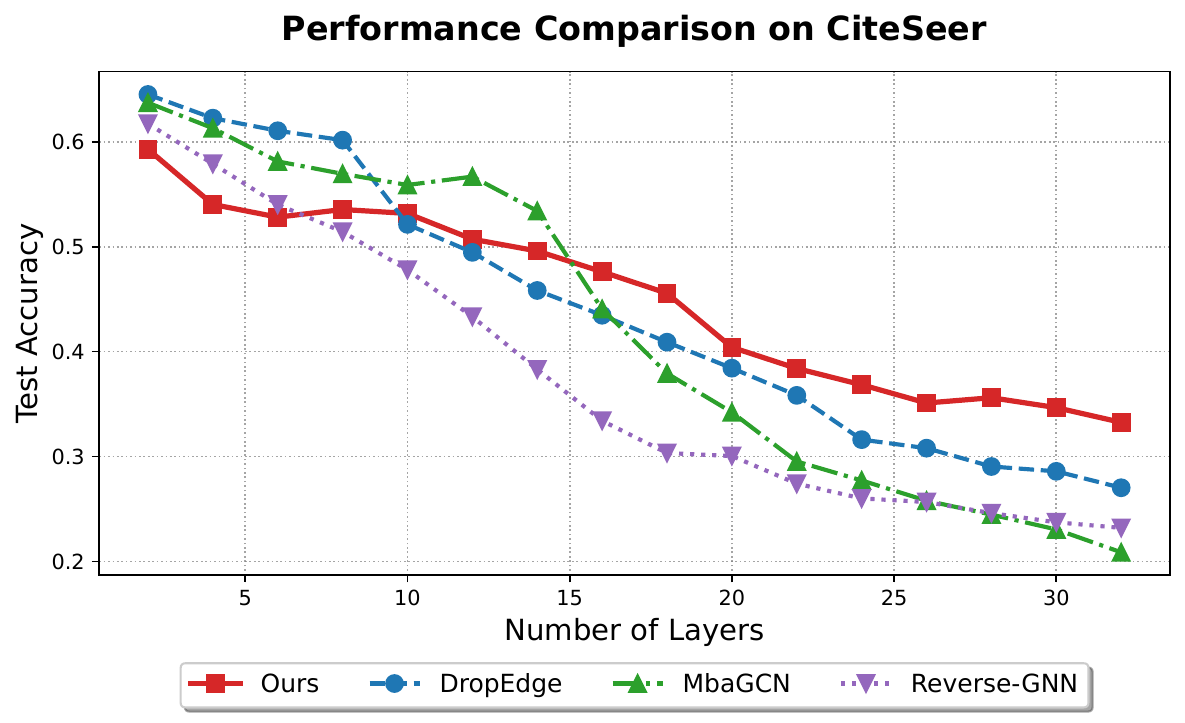}
        \caption{CiteSeer.}
    \end{subfigure}
    \begin{subfigure}[b]{0.32\textwidth}
        \includegraphics[width=\linewidth]{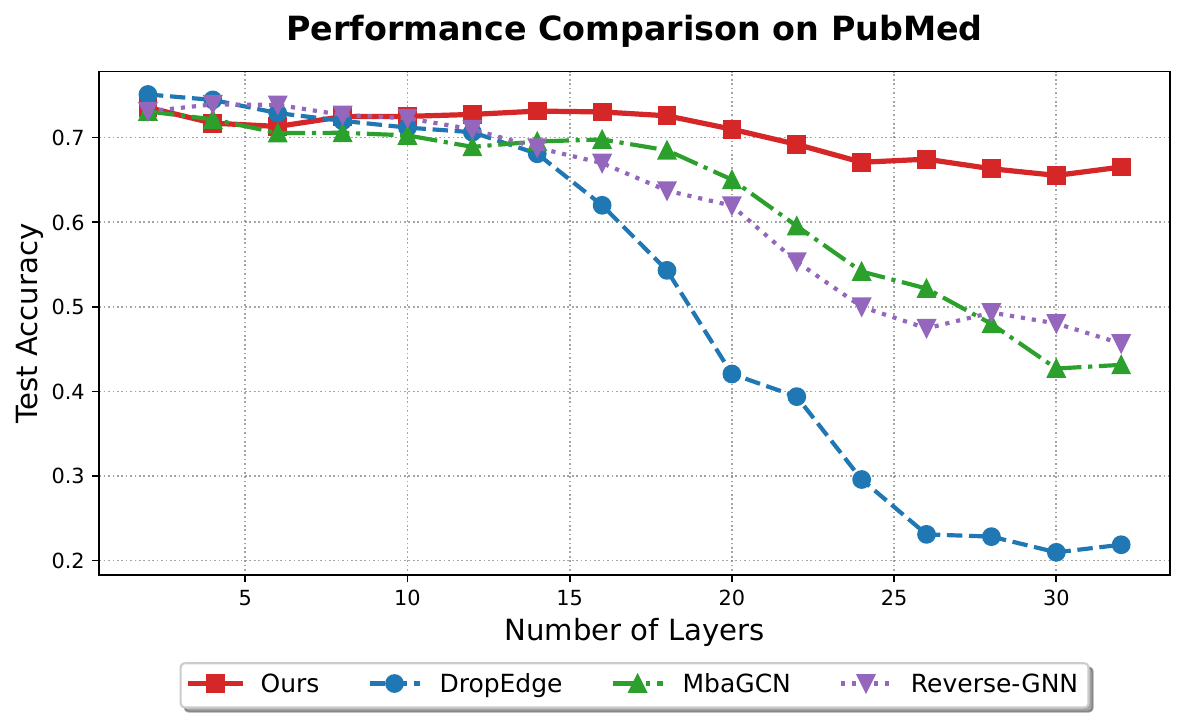}
        \caption{PubMed.}
    \end{subfigure}
    \begin{subfigure}[b]{0.32\textwidth}
        \includegraphics[width=\linewidth]{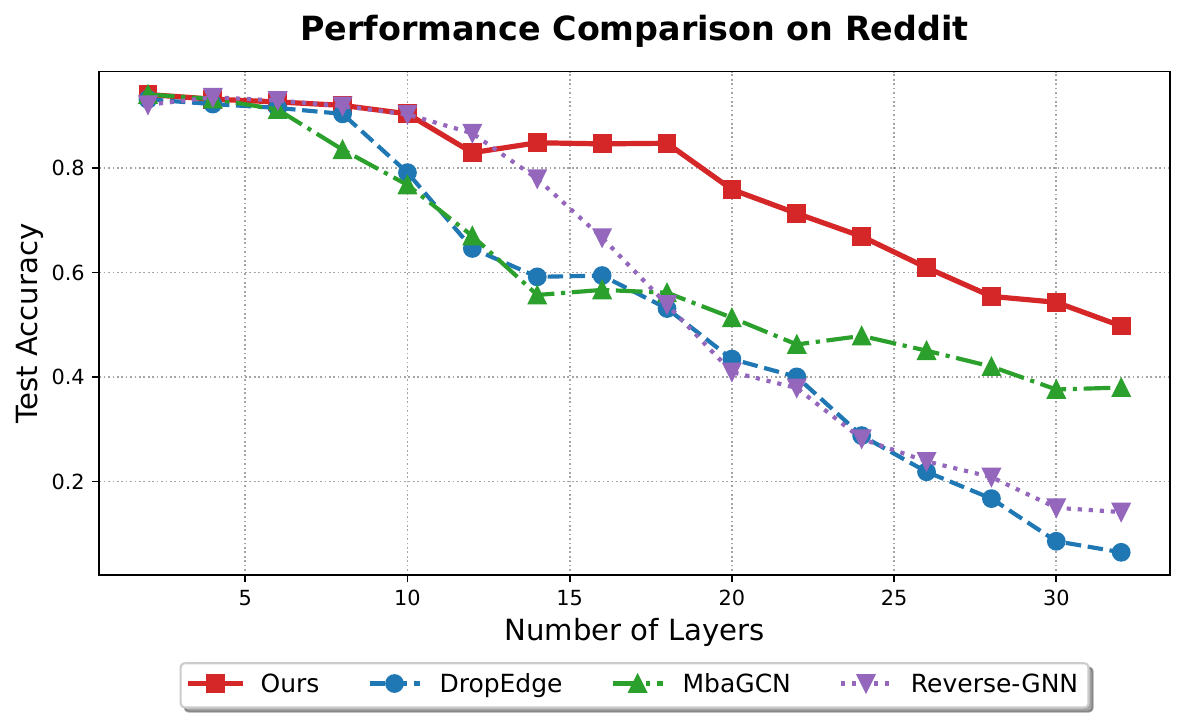}
        \caption{Reddit.}
    \end{subfigure}
    \begin{subfigure}[b]{0.32\textwidth}
        \includegraphics[width=\linewidth]{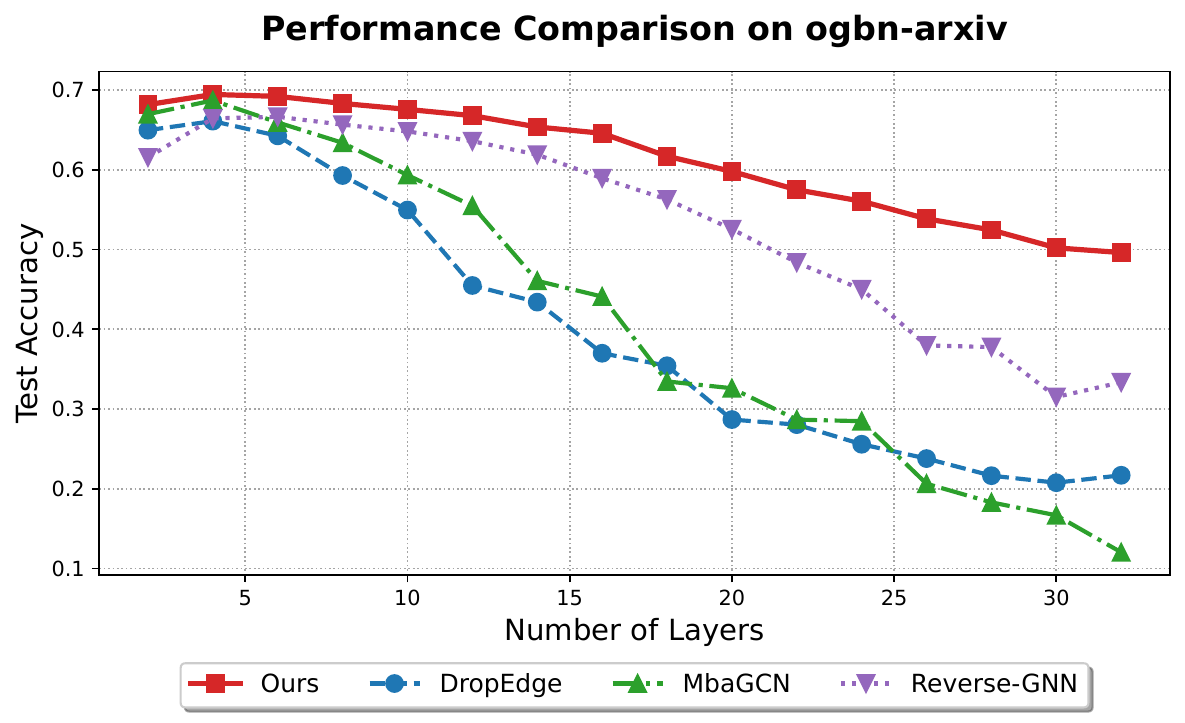}
        \caption{ogbn-arxiv.}
    \end{subfigure}
    \begin{subfigure}[b]{0.32\textwidth}
        \includegraphics[width=\linewidth]{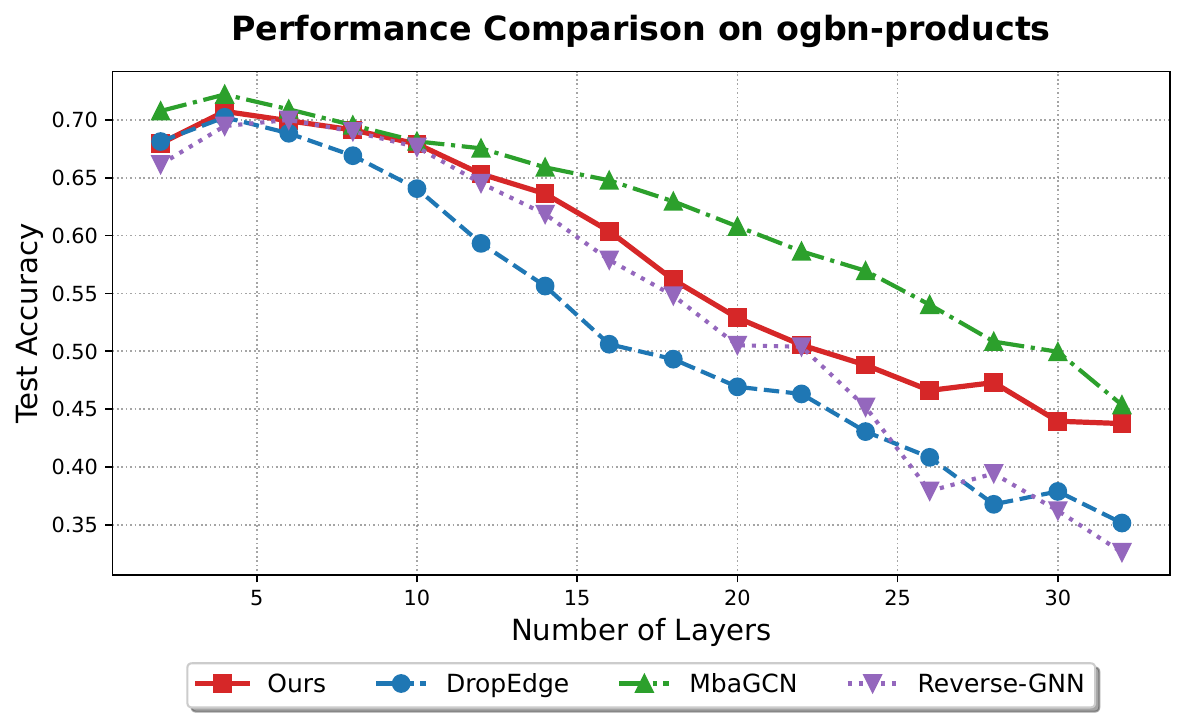}
        \caption{ogbn-products.}
    \end{subfigure}
    \caption{Full-dataset comparison with recent oversmoothing baselines.}
    \label{fig:sota-real-full}
\end{figure}

\section{Discussion and Limitations}
{
The theoretical contribution is threefold.  First, normalized corrected
propagation is identified as a spectral object in its own right: subtracting the
degree-stationary direction of \(D^{-1/2}AD^{-1/2}\) yields an operator whose
signal direction can remain separated for logarithmically many layers under
explicit CSBM conditions.  Second, the exact-recovery and partial-recovery
theorems connect this spectral correction to node classification guarantees.
Third, the proof develops normalized residual-power bounds that are entrywise
rather than only spectral-norm bounds, which is the level needed for exact
recovery.

The main technical message is that normalized correction requires tools beyond
the unnormalized analysis.  Random degrees enter through nonlinear factors and
also appear inside the correction term.  Our proof handles these effects by
expanding each normalized residual entry into positive-level atoms and then
counting the decorated walk patterns that survive in high moments.  This is the
main normalized-specific component of the theory.

Several limitations remain.  First, the CSBM does not capture all degree
heterogeneity, heterophily, or label semantics in real graphs.  Second, the
guarantees analyze the linear propagation backbone rather than nonlinear
activations, attention, optimization, or learned feature transformations.  Third,
the binary theorem is proved in the dense polylogarithmic regime
\(np\ge C\log^B n\) for fixed \(B>4\), leaving the sharper \(np\ge C\log^3 n\)
normalized threshold open.  Finally, very large graphs require scalable
approximations of the stationary-component correction.
}

\section{Conclusion and Future Work}
{
We developed a spectral theory of normalized corrected GNN propagation.  The
main result proves exact recovery in the binary CSBM after \(k=O(\log n)\)
layers under dense polylogarithmic graph density and explicit signal-to-noise
assumptions.  A multi-class theorem gives partial recovery for most nodes under
analogous signal-dominance conditions.  The central technical ingredient is
entrywise control of normalized residual powers, obtained by atom expansion and
decorated-walk counting.  Future work includes sharpening the density threshold,
extending the theory to heterophilous and degree-heterogeneous models,
incorporating nonlinear trained GNN layers, and developing scalable
approximations to the stationary correction.
}

\section{Acknowledgments}
{
This work was supported by the National Key RD Program of China 
(Grant No. 2023YFA1010202), the Central Guidance on Local Science 
and Technology Development Fund of Fujian Province (Grant No. 2023L3003), 
the National Natural Science Foundation of China (Grant No.11901094), 
the Interdisciplinary Frontier Research Project of PCL (Grant No. 2025QYB015).
}
\paragraph{AI Assistant Declaration:} 
During the preparation of this work, the authors used Gemini and ChatGPT to 
improve the readability and language of the manuscript, and to assist with 
LaTeX formatting. After using these tools, the authors reviewed and edited 
the content as needed and take full responsibility for the content of the publication.

\bibliographystyle{tmlr}
\bibliography{references}

\newpage
\appendix

\section{Mathematical Tools}
\label{appendix:mathematical_tools}

To rigorously analyze the spectral properties of the corrected graph convolution and 
its effect on node features, we require a set of fundamental probabilistic and algebraic tools. 
These tools allow us to bound the spectral norms of random matrices, control eigenvalue perturbations, 
and manage the concentration of scalar sums that arise from the CSBM's stochastic generation process.

We begin with a foundational result from matrix perturbation theory, which 
guarantees that small changes to a symmetric matrix result in bounded shifts in its eigenvalues.

\begin{theorem}\label{theorem:matrix-perturbation}
    Let $A$ and $B$ be $n \times n$ symmetric matrices. Let $\lambda_1 \geq \lambda_2 \geq \dots \geq \lambda_n$ and $\mu_1 \geq \mu_2 \geq \dots \geq \mu_n$ be the eigenvalues of $A$ and $B$, respectively. If $\norm{A-B} \leq \delta$, then $\max_i |\lambda_i - \mu_i| \leq \delta$.
\end{theorem}

Furthermore, we often need to bound the spectral norm of specific sparse symmetric matrices 
encountered in our graph constructions. The following theorem provides a general upper bound 
based on row sparsity and entry magnitude.

\begin{theorem}
\label{theorem:bound-norm-matrix}
Let $M$ be an $n\times n$ symmetric matrix such that each row of $M$ has at most $m$ non-zero 
entries and each entry of $M$ has an absolute value at most $\eps$. Then $\norm{M}\leq \eps m$.
\end{theorem}

While spectral bounds control the global operator behavior, 
local graph properties, such as node degrees and edge counts, 
are sums of independent Bernoulli random variables. To establish high-probability 
bounds on these quantities, we employ the standard Chernoff concentration inequality.

\begin{theorem}
\label{theorem:bound-sum-bernoulli}
Let $X_1,...X_n$ be independent Bernoulli random variables with mean at most $p$, and let $S_n = \sum_{i=1}^nX_i$. 
Then for any $t>0$,
\[
    \Pr{|S_n - \E[S_n]| > t} \leq \exp\left(-\Omega\left(\frac{t^2}{\E[S_n]}\right)\right).
\]
\end{theorem}

In the analysis of multi-layer graph convolutions, terms involving products 
and powers of $(1+x)$ frequently emerge. To simplify these expressions while maintaining tight bounds, 
we utilize the following standard exponential inequalities.

\begin{theorem}
\label{theorem:bound-exp}
For all $x \in \mathbb{R}$, $1+x\leq e^x$. Furthermore, if $x\leq 1$, then $e^{x} \leq 1+2x$.
\end{theorem}

Finally, since the initial node features in our CSBM framework are drawn from Gaussian mixture models, 
we rely on the concentration of Gaussian random vectors to bound the noise propagation.

\begin{theorem}
\label{theorem:Gaussian-norm-bound}
    Let $g \sim \mathcal{N}(0, \sigma^2 I_n)$ be a random Gaussian vector. Then for any $t > 0$,
    \[ \Pr{\norm{g} > t} \leq 2\exp\left(-\frac{t^2}{2n\sigma^2}\right). \]
\end{theorem}

\section{Additional Preliminaries and Lemmas}
\label{appendix:additional_preliminaries}

We recall the spectral theorem for real symmetric matrices. If $M \in \mathbb{R}^{n \times n}$ 
is symmetric, it admits an eigendecomposition $M = \sum_{i=1}^n \lambda_i w_i w_i^\top$, 
where $\lambda_1, \dots, \lambda_n \in \mathbb{R}$ are the eigenvalues and 
$w_1, \dots, w_n \in \mathbb{R}^n$ form an orthonormal basis of eigenvectors. 
In this case, the operator norm satisfies $\norm{M} = \max_{i} |\lambda_i|$.

Additionally, we denote the multivariate Gaussian distribution with mean vector 
$\mu$ and covariance matrix $\Sigma$ by $\mathcal{N}(\mu, \Sigma)$. For a scalar 
random variable $X \sim \mathcal{N}(\mu, \sigma^2)$, we utilize the standard tail bound:
\[
    \Pr{|X-\mu| > t\sigma} \leq \exp\left(-\frac{t^2}{2}\right).
\]

A key insight from \cite{pmlr-v139-baranwal21a} establishes that the analysis 
of linear classifiers on general $m$-dimensional CSBMs can be effectively mapped to 
a simpler one-dimensional setting. To formalize this, we introduce the following definition:

\begin{definition}[One-Dimensional Centered CSBM]
    A CSBM is said to be \textbf{one-dimensional and centered} with parameters $n, p, q, \sigma$, 
    if the feature dimensionality is $m=1$. In this model, the single feature vector 
    $x \in \mathbb{R}^n$ is decomposed as $x = s + g$, where:
    \begin{itemize}
        \item $g \sim \mathcal{N}(0, \sigma^2 I_n)$ represents the Gaussian noise vector.
        \item $s \in \mathbb{R}^n$ is the \textit{signal} vector defined by class membership:
        \[
            s_i = \begin{cases} \frac{1}{\sqrt{n}} & \text{if } i \in S, 
                \\ \frac{-1}{\sqrt{n}} & \text{if } i \in T. \end{cases}
        \]
    \end{itemize}
\end{definition}

This construction ensures that the signal vector $s$ has a unit norm ($\norm{s}=1$), 
simplifying subsequent notation.

The lemma presented below enables the reduction of linear classifier analysis from the general 
CSBM to the centered 1-dimensional model. Accordingly, our proof strategy for the main theorems 
begins with the 1-dimensional case, subsequently applying Lemma~\ref{lemma:csbm-reduction} to extend the results.

\begin{lemma}\label{lemma:csbm-reduction}
Consider an \(m\)-dimensional balanced two-block \(\operatorname{CSBM}(n,m,p,q,\mu,\nu,\sigma)\) with \(p>q>0\) with centered feature
means \(\mu+\nu=0\) and \(\mu\ne\nu\).  Let \(s_i=n^{-1/2}\) on \(S\) and
\(s_i=-n^{-1/2}\) on \(T\).  Then the projection
\[
w=\frac{2(\mu-\nu)}{\sqrt n\,\|\mu-\nu\|^2}
\]
satisfies
\[
Xw=s+g,\qquad g\sim N(0,\sigma'^2 I_n),\qquad
\sigma'=\frac{2\sigma}{\sqrt n\,\|\mu-\nu\|}.
\]
Consequently, for every \(k\ge1\),
\[
\widehat A^kXw=\widehat A^k(s+g).
\]
Thus any sign-separation result for the one-dimensional centered model transfers
to the original CSBM through the linear classifier with normal vector \(w\).
\end{lemma}

\begin{proof}
It follows from $\mu+\nu=0$ that 
\[
\nu=-\mu,
\qquad
\mu-\nu=2\mu.
\]
For each vertex \(i\), the feature vector can be written as
\[
x_i=
\begin{cases}
\mu+\xi_i, & i\in S,\\
\nu+\xi_i, & i\in T,
\end{cases}
\]
where \(\xi_i\sim N(0,\sigma^2 I_m)\) are independent.
Recall that
\[
w=\frac{2(\mu-\nu)}{\sqrt n\,\|\mu-\nu\|^2}.
\]
Then, if \(i\in S\),
\[
\mu^\top w
=
\mu^\top
\frac{2(\mu-\nu)}{\sqrt n\,\|\mu-\nu\|^2}
=
\frac{2\mu^\top(2\mu)}{\sqrt n\,\|2\mu\|^2}
=
\frac{1}{\sqrt n}.
\]
Similarly, if \(i\in T\), since \(\nu=-\mu\), we also have 
\[
\nu^\top w=-\frac{1}{\sqrt n}.
\]
Therefore
\[
x_i^\top w=s_i+\xi_i^\top w.
\]
Define
\[
g_i:=\xi_i^\top w.
\]
Since the vectors \(\xi_i\) are independent Gaussian vectors with distribution
\(N(0,\sigma^2 I_m)\), the variables \(g_i\) are independent centered Gaussian
random variables. Moreover,
\[
\operatorname{Var}(g_i)
=
\sigma^2\|w\|^2.
\]
By the definition of \(w\),
\[
\|w\|^2
=
\frac{4\|\mu-\nu\|^2}{n\|\mu-\nu\|^4}
=
\frac{4}{n\|\mu-\nu\|^2}, 
\]
which implies that 
\[
g_i\sim N\left(0,\frac{4\sigma^2}{n\|\mu-\nu\|^2}\right).
\]
Equivalently,
\[
g_i\sim N(0,\sigma'^2),
\qquad
\sigma'=\frac{2\sigma}{\sqrt n\,\|\mu-\nu\|}.
\]
Thus, in vector form,
\[
Xw=s+g,
\]
where \(g=(g_1,\ldots,g_n)^\top\) has independent coordinates distributed as
\(N(0,\sigma'^2)\).

Since \(\widehat A\) acts on the vertex coordinate and \(w\) is fixed in the feature
coordinate, we have, for every integer \(k\ge1\),
\[
\widehat A^kXw=\widehat A^k(Xw)=\widehat A^k(s+g).
\]
If the coordinates of \(\widehat A^k(s+g)\) have signs agreeing with the true
community labels, then
\[
\operatorname{sign}\bigl((\widehat A^kX)_i w\bigr)
=
\operatorname{sign}\bigl((\widehat A^k(s+g))_i\bigr)
\]
correctly classifies every vertex \(i\). Hence the rows of \(\widehat A^kX\) are
linearly separable by the hyperplane with normal vector \(w\) and threshold \(0\).
\end{proof}


In the 1-dimensional model, the presence of the signal $s$ in the feature space is immediate. 
Furthermore, the convolution matrix effectively captures this signal, as it can be modeled as 
a perturbation of the signal matrix $ss^\top$. Following the framework of
\citep{Wang2024Analysis}, the un-normalized convolution matrix $\Tilde{A}$ admits the decomposition
\begin{align} 
        \Tilde{A} = \eta ss^\top + \frac{1}{d}R + d' \one\one^\top
\label{eqn: convolution-expression}
\end{align}
where $\eta := \frac{(p-q)n}{2d} $ represents the \textit{signal strength}, 
$d':= \frac{(p+q)n-2d}{2nd}$ denotes the \textit{average degree deviation}, 
and $R$ is the ``edge-deviation" matrix defined by $R_{i,j} = A_{i,j} - \E[A_{i,j}]$.
The expected signal strength ratio is defined as $\gamma = \frac{|p-q|}{p+q}$.
As established in \citep{Wang2024Analysis}, since $R$ consists of i.i.d. zero-mean 
entries with bounded variance, and $d'$ is controlled by degree concentration, 
the matrix $\Tilde{A}$ concentrates around $\gamma ss^\top$. This behavior extends
to the normalized convolution matrix \(\widehat A\) as well. The relevant
concentration bounds from \citep{Wang2024Analysis} are summarized in the following proposition.


Our analysis commences with a spectral norm bound for the corrected graph convolution matrix. 
Lemma \ref{lemma:1} bounds the spectral norm of the deviation between  \(\widehat A\) 
and $\Tilde{A}$. Control over this spectral norm is fundamental to 
limiting error propagation across network layers.


\begin{lemma}\label{lemma:1}
Suppose that \(p=\Omega(\log n/n)\) and \(p>q>0\). Then, there exists a constant $C>0$ such that with probability at least
\(1-n^{-\Omega(1)}\),
\[
\|\widehat A-\widetilde A\|
\le
C\sqrt{\frac{\log n}{np}} .
\]
\end{lemma}

\begin{proof}
Let
\[
\bar d:=\frac{(p+q)n}{2},
\qquad
\varepsilon:=K\sqrt{\frac{\log n}{np}},
\]
where \(K>0\) is a sufficiently large absolute constant. Since \(q\le p\), we have $\bar d=\Theta(np).$ For each vertex \(v\), its degree \(d_v\) is a sum of independent Bernoulli random variables and
\[
\mathbb E d_v=\bar d+O(p).
\]
The \(O(p)\) term comes from excluding the vertex \(v\) itself from its own community and is negligible compared with \(\bar d=\Theta(np)\). By Theorem \ref{theorem:bound-sum-bernoulli},
\[
\mathbb P\bigl(|d_v-\mathbb E d_v|>\varepsilon \bar d/2\bigr)
\le
2\exp\left(-c\varepsilon^2\bar d\right)
\le
2\exp(-c'K^2\log n).
\]
Taking \(K\) large enough and applying a union bound over all \(v\in[n]\), we obtain, with probability at least \(1-n^{-\Omega(1)}\),
\begin{align}\label{LEMMA:concentration-dv}
d_v\in [\bar d(1-\varepsilon),\bar d(1+\varepsilon)]
\quad\text{for every }v\in[n].
\end{align}
On this event,
\begin{align}\label{LEMMA:concentration-average-degree}
d=\frac1n\sum_{v=1}^n d_v
\in [\bar d(1-\varepsilon),\bar d(1+\varepsilon)].
\end{align}

Now we bound $\|\widehat A-\widetilde A\|$. Recall that 
\[
\widehat A-\widetilde A
=
\left(D^{-1/2}AD^{-1/2}-\frac1dA\right)
-
\left(
\frac{1}{1^\top D1}D^{1/2}11^\top D^{1/2}
-\frac1n11^\top
\right).
\]
It is enough to bound the two terms separately. Let 
\[
M_1:=D^{-1/2}AD^{-1/2}-\frac1dA .
\]
Then  \(M_1\) is symmetric, and for every \(u,v\),
\[
(M_1)_{uv}
=
A_{uv}
\left(
\frac1{\sqrt{d_ud_v}}-\frac1d
\right).
\]
By (\plaineqref{LEMMA:concentration-dv}) and (\plaineqref{LEMMA:concentration-average-degree}), we have 
\[
\left|
\frac1{\sqrt{d_ud_v}}-\frac1d
\right|
\le
\frac{C\varepsilon}{\bar d}.
\]
Moreover,
\[
\Delta(G):=\max_v d_v\le 2\bar d
\]
for all sufficiently large \(n\). Therefore every row of \(M_1\) has absolute row sum at most
\[
\sum_v |(M_1)_{uv}|
\le
d_u\cdot \frac{C\varepsilon}{\bar d}
\le
C'\varepsilon .
\]
Applying Theorem \ref{theorem:bound-norm-matrix} gives that
\begin{align}\label{LEMMA:bound-M1-c}
\|M_1\|\le C'\varepsilon .
\end{align}

Next let
\[
M_2:=
\frac{1}{1^\top D1}D^{1/2}11^\top D^{1/2}
-\frac1n11^\top .
\]
Then  \(M_2\) is symmetric and its entries are
\[
(M_2)_{uv}
=
\frac{\sqrt{d_ud_v}}{\sum_{w=1}^n d_w}
-\frac1n .
\]
Again (\plaineqref{LEMMA:concentration-dv}) and the fact 
\[
\sum_{w=1}^n d_w
\in[n\bar d(1-\varepsilon),n\bar d(1+\varepsilon)],
\]
we obtain
\[
\left|
\frac{\sqrt{d_ud_v}}{\sum_w d_w}
-\frac1n
\right|
\le
\frac{C\varepsilon}{n}.
\]
Thus every row of \(M_2\) has absolute row sum at most
\[
\sum_v |(M_2)_{uv}|
\le
n\cdot \frac{C\varepsilon}{n}
=
C\varepsilon .
\]

Applying Theorem \ref{theorem:bound-norm-matrix} again gives that
\begin{align}\label{LEMMA:bound-M2-c}
\|M_2\|\le C\varepsilon.
\end{align}

Combining (\plaineqref{LEMMA:bound-M1-c}) and (\plaineqref{LEMMA:bound-M2-c}) gives 
\[
\|\widehat A-\widetilde A\|
\le
\|M_1\|+\|M_2\|
\le
C\varepsilon
=
C K\sqrt{\frac{\log n}{np}} .
\]
Absorbing \(K\) into the constant \(C\), we conclude
\[
\|\widehat A-\widetilde A\|
\le
C\sqrt{\frac{\log n}{np}} .
\]
This completes the proof.
\end{proof}

\begin{lemma}\label{prop:degree-concentration}
Assume that \(p \ge C\frac{\log n}{n}\) and \(p > q > 0\). Let
\[
\gamma := \frac{p-q}{p+q}, \qquad
\bar d := \frac{(p+q)n}{2}, \qquad
\eta := \frac{(p-q)n}{2d}, \qquad
d' := \frac{\bar d-d}{nd}.
\]
Then, with probability at least \(1-n^{-\Omega(1)}\), the empirical average degree \(d\) satisfies
\[
d = \bar d\left(1 + O\left(\sqrt{\frac{\log n}{n^2p}}\right)\right).
\]
Consequently, the deviations of \(\eta\) and \(d'\) are bounded by
\[
|\eta-\gamma| \le C\frac{\sqrt{\log n}}{n\sqrt p}, \qquad
|d'| \le C\frac{\sqrt{\log n}}{n^2\sqrt p}.
\]
Moreover, let \(R := A - \mathbb E A\) denote the centered adjacency fluctuation. Then the spectral norms are bounded as follows:
\[
\|R\| \le C\sqrt{np}, \qquad
\left\| \widetilde A - \gamma ss^\top \right\| \le \frac{C}{\sqrt{np}}.
\]
Finally, combined with Lemma~\ref{lemma:1}, we have
\[
\left\| \widehat A - \gamma ss^\top \right\| \le C\sqrt{\frac{\log n}{np}}.
\]
Consequently,
\[
\left\| \widehat A - \eta ss^\top \right\| \le C\sqrt{\frac{\log n}{np}}.
\]
\end{lemma}

\begin{proof}
Note that $d=2|E|/n$ is the average degree  and  \(|E|\) is a sum of independent Bernoulli random variables with
\[
\mathbb E |E|
=
\frac{n^2}{4}(p+q)+O(np). 
\]
We have
\[
\mathbb E d
=
\bar d+O(p).
\]
Note that the variance of \(|E|\) is at most \(Cn^2p\). By Theorem \ref{theorem:bound-sum-bernoulli}, with
probability at least \(1-n^{-\Omega(1)}\),
\[
\left|
|E|-\mathbb E|E|
\right|
\le
C n\sqrt{p\log n}.
\]
Therefore
\[
|d-\bar d|
\le
C\sqrt{p\log n}.
\]
Since \(\bar d=\Theta(np)\), this implies
\[
d=\bar d\left(1+O\left(\frac{\sqrt{p\log n}}{np}\right)\right)
=
\bar d\left(1+O\left(\sqrt{\frac{\log n}{n^2p}}\right)\right).
\]

We now estimate \(\eta-\gamma\). By definition,
\[
\eta=\frac{(p-q)n}{2d},
\qquad
\gamma=\frac{p-q}{p+q}
=
\frac{(p-q)n}{2\bar d}.
\]
Hence
\[
|\eta-\gamma|
=
\frac{(p-q)n}{2}
\left|
\frac1d-\frac1{\bar d}
\right|
=
\frac{(p-q)n}{2}
\frac{|d-\bar d|}{d\bar d}.
\]
Using \(d,\bar d=\Theta(np)\), \(p-q\le p+q=O(p)\), and
\[
|d-\bar d|\le C\sqrt{p\log n},
\]
we obtain
\begin{align}\label{LEMMA:eta-gamms-samll}
|\eta-\gamma|
\le
C
\frac{np\sqrt{p\log n}}{n^2p^2}
=
C\frac{\sqrt{\log n}}{n\sqrt p}.
\end{align}
Similarly,
\begin{align}\label{LEMMA:bound-d'}
|d'|
=
\frac{|\bar d-d|}{nd}
\le
\frac{C\sqrt{p\log n}}{n\cdot np}
=
C\frac{\sqrt{\log n}}{n^2\sqrt p}.
\end{align}

Next we control the centered adjacency matrix
\[
R=A-\mathbb E A.
\]
The entries of \(R\) are independent above the diagonal, centered, and
bounded by \(1\), with variances at most \(p\). By a standard spectral
norm bound for sparse centered adjacency matrices, since
\(p\ge C\log n/n\), with probability at least \(1-n^{-\Omega(1)}\),
\begin{align}\label{LEMMA:bound-R}
\|R\|\le C\sqrt{np}.
\end{align}

We now decompose \(\widetilde A\). Up to the harmless diagonal correction,
which has operator norm \(O(1/d)=O(1/(np))\), the expectation of \(A\) can be written as
\[
\mathbb E A
=
\frac{p+q}{2}11^\top
+
\frac{p-q}{2}yy^\top,
\]
where \(y_i=1\) on \(S\) and \(y_i=-1\) on \(T\). It follows from $s=y/\sqrt n$   that 
\[
yy^\top=nss^\top.
\]
Thus
\[
\frac1d\mathbb E A-\frac1n11^\top
=
\eta ss^\top
+
d'11^\top
+
E_{\mathrm{diag}},
\]
where
\begin{align}\label{LEMMA:E-diag}
\|E_{\mathrm{diag}}\|\le \frac{C}{d}\le \frac{C}{np}.
\end{align}
Therefore
\[
\widetilde A
=
\frac1dA-\frac1n11^\top
=
\eta ss^\top
+
\frac1dR
+
d'11^\top
+
E_{\mathrm{diag}}.
\]
Consequently,
\[
\widetilde A-\gamma ss^\top
=
(\eta-\gamma)ss^\top
+
\frac1dR
+
d'11^\top
+
E_{\mathrm{diag}}.
\]
Taking operator norms gives
\[
\left\|\widetilde A-\gamma ss^\top\right\|
\le
|\eta-\gamma|
+
\frac1d\|R\|
+
|d'|\|11^\top\|
+
\|E_{\mathrm{diag}}\|.
\]
It follows from the fact  \(p\ge C\log n/n\) that
\[
\frac{\sqrt{\log n}}{n\sqrt p}
\le
\frac{C}{\sqrt{np}}.
\]
This together with (\plaineqref{LEMMA:eta-gamms-samll}), (\plaineqref{LEMMA:bound-d'}), (\plaineqref{LEMMA:bound-R}) and (\plaineqref{LEMMA:E-diag}) yields that 
\[
\left\|
\widetilde A-\gamma ss^\top
\right\|
\le
\frac{C}{\sqrt{np}}.
\]

By Lemma~\ref{lemma:1}, we have that 
\[
\left\|
\widehat A-\widetilde A
\right\|
\le
C\sqrt{\frac{\log n}{np}}.
\]
Therefore, 
\[
\left\|
\widehat A-\gamma ss^\top
\right\|
\le
\left\|
\widehat A-\widetilde A
\right\|
+
\left\|
\widetilde A-\gamma ss^\top
\right\|
\le
C\sqrt{\frac{\log n}{np}},
\]
and then together with (\plaineqref{LEMMA:eta-gamms-samll}), this implies 
\[
\left\|\widehat A-\eta ss^\top\right\|
\le
\left\|\widehat A-\gamma ss^\top\right\|+|\gamma-\eta|
\le
C\sqrt{\frac{\log n}{np}}. 
\]
This completes the proof. 
\end{proof}

For the next lemma, let \(y_i=\sqrt n\,s_i\in\{\pm1\}\).  For an unordered
edge \(e=\{r,t\}\), write \(\xi_e=A_{rt}-\mathbb E A_{rt}\).  An
\emph{atom} is one monomial term produced by the finite Taylor expansion of a
single normalized-residual entry and has the form
\[
\mathfrak a(i,j)=
c_{\mathfrak a}\,\chi_{\mathfrak a}(i,j)\,
n^{-b_{\mathfrak a}}\bar d^{-q_{\mathfrak a}}
\sum_{\phi\in\Phi_{\mathfrak a}(i,j)}
\prod_{h=1}^{m_{\mathfrak a}}\xi_{e_h(\phi)} ,
\]
where \(\bar d=(p+q)n/2\), \(\chi_{\mathfrak a}(i,j)\in
\{1,y_i,y_j,y_iy_j\}\), and \(\Phi_{\mathfrak a}(i,j)\) is the finite set of
auxiliary vertex choices generated by local and global degree sums.  For
\(\phi\in\Phi_{\mathfrak a}(i,j)\), \(e_h(\phi)\) is the \(h\)-th centered
edge variable selected by that auxiliary choice.  The integer
\(m_{\mathfrak a}\) counts centered edge variables, \(q_{\mathfrak a}\) is the
power of \(\bar d^{-1}\), and \(b_{\mathfrak a}\) records powers of \(n^{-1}\)
from global degree averages.  The atom level is
\[
\lambda_{\mathfrak a}:=2q_{\mathfrak a}-m_{\mathfrak a}.
\]

\begin{lemma}
\label{lem:entrywise-atom-expansion}
Let \(R'=\widehat A-\eta ss^\top\), where
\(\eta=(p-q)n/(2d)\), and fix \(B>4\).  Assume
\(p\ge C_d\log^B n/n\).  Fix \(A_0>0\), let \(M\le A_0\log n\), and set
\[
L_*:=K_0M,\qquad
\tau:=\left(C\frac{\log n}{\sqrt{np}}\right)^{L_*+1},
\]
where \(K_0\) is sufficiently large depending only on \(A_0\).  There is an
event \(\mathcal G\) with
\(\mathbb P(\mathcal G^c)\le \exp(-c\log^2 n)\), a clipped residual
\(R_{\rm cl}\), and a decomposition
\[
R_{\rm cl}=R_{\rm pol}+R_{\rm rem},
\qquad R_{\rm cl}=R'\quad\text{on }\mathcal G,
\]
such that, with the atom notation above,
\[
(R_{\rm pol})_{ij}=\sum_{\mathfrak a\in\mathcal A}\mathfrak a(i,j)
\qquad (i,j\in[n]),
\]
for a finite deterministic atom family \(\mathcal A\) satisfying
\[
\lambda_{\mathfrak a}\ge1,\qquad
m_{\mathfrak a}\le C\lambda_{\mathfrak a},\qquad
\#\{\mathfrak a\in\mathcal A:\lambda_{\mathfrak a}=\lambda\}\le C^\lambda .
\]

Deterministic diagonal/no-self atoms have \(q=1,m=0\) and are contracted before
the centered-edge skeleton is counted.  On \(\mathcal G\),
\[
\|R_{\rm rem}\|_{\infty\to\infty}\le C\tau,\qquad
\|R_{\rm cl}\|_{\infty\to\infty}+\|R_{\rm pol}\|_{\infty\to\infty}\le C .
\]
\end{lemma}

\begin{proof}
\textbf{Off-diagonal decomposition.}
We first isolate the degree-normalization factors in one off-diagonal entry.
Write
\[
\beta_i:=\frac{d_i-\bar d}{\bar d},\qquad
\beta:=\frac{d-\bar d}{\bar d}.
\]
Then \(d_i=\bar d(1+\beta_i)\) and \(d=\bar d(1+\beta)\).  For
\(i\ne j\), define
\[
F_{ij}:=(1+\beta_i)^{-1/2}(1+\beta_j)^{-1/2}
=\frac{\bar d}{\sqrt{d_i d_j}},
\]
and
\[
G_{ij}:=(1+\beta_i)^{1/2}(1+\beta_j)^{1/2}(1+\beta)^{-1}
=\frac{\sqrt{d_i d_j}}{\bar d(1+\beta)}.
\]
It is easy to see that 
\[
(D^{-1/2}AD^{-1/2})_{ij}=\frac{A_{ij}}{\bar d}F_{ij}, \quad \text{ and }
\quad
\frac{\sqrt{d_i d_j}}{\one^\top D\one}=\frac{1}{n}G_{ij}.
\]
For \(i\ne j\),
\[
\mathbb E A_{ij}=\frac{\bar d}{n}(1+\gamma y_i y_j),
\qquad
\eta=\gamma(1+\beta)^{-1},
\qquad
s_is_j=\frac{y_iy_j}{n}.
\]
Substituting \(A_{ij}=\mathbb E A_{ij}+\xi_{ij}\) into
\(R'=\widehat A-\eta ss^\top\) gives the off-diagonal decomposition
\begin{align}\label{LEMMA:R'-ij-main}
R'_{ij}
=
\frac{\xi_{ij}}{\bar d}F_{ij}
+\frac{1}{n}\bigl(F_{ij}-G_{ij}\bigr)
+\frac{\gamma y_i y_j}{n}\bigl(F_{ij}-(1+\beta)^{-1}\bigr).
\end{align}
The right-hand side vanishes at the idealized deterministic point
\(\xi=0\), \(\beta_i=\beta_j=\beta=0\).  The actual no-self offset
\(\mathbb E d_i-\bar d=O(p)\) is deterministic of size \(O(1/n)\) after
normalization and is charged as a deterministic atom with \(q=1,m=0\).

We next expand the degree-normalization factors.  This is where clipping is
used: it makes the Taylor expansion globally defined while leaving the matrix
unchanged on the high-probability event below.
Define the event \(\mathcal G\) by
\[
\max_i |d_i-\mathbb E d_i|\le C\sqrt{np}\log n,\qquad
\left|\sum_v d_v-n\bar d\right|\le Cn\sqrt p\,\log n .
\]
Chernoff bounds and a union bound give
\(\mathbb P(\mathcal G^c)\le\exp(-c\log^2 n)\), and on \(\mathcal G\) all
degree fluctuations in (\plaineqref{LEMMA:R'-ij-main}) have absolute value \(O(\log n/\sqrt{np})<c_*\)
after increasing \(C_d\).  In particular, \(1+\beta_i\), \(1+\beta_j\), and
\(1+\beta\) are bounded away from zero on \(\mathcal G\), so the factors
\(F_{ij}\) and \(G_{ij}\) are well-defined and uniformly bounded there.
Replace \(x^{-1/2}\), \(x^{1/2}\), and \(x^{-1}\) by smooth clipped functions
agreeing with them on \([1-c_*,1+c_*]\).  Taylor expand each clipped function
to depth \(L_*\) around leave-one-out degrees \(d_i^{(-j)}\),
\(d_j^{(-i)}\), and the global average \(d\).

\textbf{Atom expansion.}
The polynomial Taylor terms give \(R_{\rm pol}\).  Each exposed edge contributes
\(\xi_{ij}/\bar d\), each local degree decoration contributes
\(\bar d^{-1}\sum_{\ell\ne i,j}\xi_{i\ell}\) or
\(\bar d^{-1}\sum_{\ell\ne i,j}\xi_{j\ell}\), and each global decoration
contributes \(n^{-1}\bar d^{-1}\sum_{r<s}\xi_{rs}\).  In the atom notation,
the auxiliary map \(\phi\) records the choices of vertices in these sums.  Thus
an atom with \(m_{\mathfrak a}\) centered edge variables and
\(q_{\mathfrak a}\) powers of \(\bar d^{-1}\) has level
\(2q_{\mathfrak a}-m_{\mathfrak a}\).  The cancellation in (A) removes all
level-zero deterministic terms.  Exposed-edge and local/global degree atoms
have level at least one, deterministic no-self atoms have level two, and the
number of Taylor multi-indices of level \(\lambda\) is at most \(C^\lambda\).
This proves the atom bounds.

\textbf{Remainder bound.}
The Taylor remainders form \(R_{\rm rem}\).  On \(\mathcal G\), every normalized
degree fluctuation \(W\) appearing in a remainder, such as a local leave-one-out
degree fluctuation, a global average-degree fluctuation, or the exposed-edge
term \(A_{ij}/\bar d\), satisfies
\(|W|\le C\log n/\sqrt{np}\), so
\[
|\rho_{\alpha,L_*}(W)|\le
\left(C\frac{\log n}{\sqrt{np}}\right)^{L_*+1}=\tau .
\]
Remainder terms in the adjacency part have row sums at most \(C\tau\), and
remainder terms in the stationary correction have entries at most \(C\tau/n\).
The displayed \(\ell_\infty\to\ell_\infty\) bounds follow from the degree
concentration on \(\mathcal G\).
\end{proof}

\begin{lemma}
\label{lem:decorated-pattern-counting}
Consider a product of \(2r\) path monomials obtained from the atoms in
Lemma~\ref{lem:entrywise-atom-expansion}.  Suppose the product has
\(2M\) residual positions and total level \(\Lambda=2M+h\).  After summing over
all canonical decorated patterns of this level and over all vertex assignments,
the total contribution to the \(2r\)-th moment is at most
\[
n^{-r}(CM)^M(np)^{-M}
\left(\frac{CM}{\sqrt{np}}\right)^h .
\]
\end{lemma}

\begin{proof}
For a fixed product, let \(m\) be the total number of centered edge variables
before identification, let \(\ell\) be the number of distinct centered edges
after identification, and let \(Q\) be the total \(\bar d^{-1}\)-charge.  By the
definition of level,
\[
Q-\frac m2=\frac{\Lambda}{2}.
\]
The expectation is zero unless every centered edge appears at least twice; hence
\(\ell\le m/2\).  For a nonzero pattern,
\[
\left|\mathbb E\prod_e \xi_e^{\nu_e}\right|\le C^m p^\ell,
\qquad \nu_e\ge2.
\]

Global decorations have factors \(n^{-1}\bar d^{-1}\sum_{r<s}\xi_{rs}\).  If a
global-only edge is not identified with the local path skeleton, its two free
endpoints give at most \(n^2\) choices and the two corresponding global
normalizations give \(n^{-2}\).  Hence global-only decorations do not create a
positive power of \(n\).  After these cancellations, reveal the remaining
canonical skeleton from the fixed root \(u\).  If \(F\) is the number of new
free vertices, each first appears through a centered edge that is paired later,
so \(F\le\ell\).  The \(2r\) terminal factors \(s_{v_a}\) contribute \(n^{-r}\),
and the total vertex factor is at most \(n^{\ell-r}\).

For one fixed canonical pattern, the contribution is therefore bounded by
\[
n^{\ell-r}p^\ell(np)^{-Q}
=n^{-r}(np)^{\ell-Q}
\le n^{-r}(np)^{-\Lambda/2}.
\]
It remains to count canonical patterns.  The \(2M\) residual positions, their
paired first occurrences, and the choices of atom sources give at most
\((CM)^M\) possibilities at base level \(2M\).  The extra level \(h\), coming
from additional Taylor decorations or deterministic charges, contributes at
most \((CM)^h\) choices.  Thus the contribution of all patterns of level
\(\Lambda=2M+h\) is at most
\[
n^{-r}(CM)^M(np)^{-M}
\left(\frac{CM}{\sqrt{np}}\right)^h ,
\]
as claimed.
\end{proof}

\begin{lemma}\label{lem:decorated-expansion}
Let \(R'=\widehat A-\eta ss^\top\), where
\(\eta=(p-q)n/(2d)\).  Fix \(B>4\), and assume
\(p\ge C_d\log^B n/n\).
Fix \(A_0>0\).  For every \(M\le A_0\log n\), every pair
\((a,r)\) with \(ar=M\), and every \(u\in[n]\), there is a decomposition
\[
e_u^\top R'^a s=Z_{u,a}^{(M)}+\mathcal E_{u,a}^{(M)}
\]
such that
\[
\mathbb E|\mathcal E_{u,a}^{(M)}|^{2r}
\le
n^{-r}(CM)^M(np)^{-M},
\qquad
\mathbb E|Z_{u,a}^{(M)}|^{2r}
\le
n^{-r}(CM)^M(np)^{-M}.
\]
The constant \(C\) depends only on \(A_0\) and \(B\).
\end{lemma}

\begin{proof}
Apply Lemma~\ref{lem:entrywise-atom-expansion} with this value
of \(M\).  On the event \(\mathcal G\), we have
\[
R'=R_{\rm cl}=R_{\rm pol}+R_{\rm rem},
\qquad
\|R_{\rm rem}\|_{\infty\to\infty}\le C\tau,\qquad
\|R_{\rm pol}\|_{\infty\to\infty}+\|R_{\rm cl}\|_{\infty\to\infty}\le C,
\]
where
\[
\tau=\left(C\frac{\log n}{\sqrt{np}}\right)^{K_0M+1}.
\]
Expand
\[
e_u^\top R_{\rm pol}^{a}s
=
\sum_{v_1,\ldots,v_a}
\prod_{t=1}^{a}(R_{\rm pol})_{v_{t-1}v_t}s_{v_a},
\qquad v_0=u.
\]
The level of a path monomial is the sum of the levels of its \(a\) atoms.  Since
all atoms have level at least one, every path monomial has level at least \(a\).
Let \(Z_{u,a}^{(M)}\) be the sum of all level-\(a\) path monomials and define
\[
\mathcal E_{u,a}^{(M)}:=e_u^\top R'^a s-Z_{u,a}^{(M)}.
\]

We first bound \(Z_{u,a}^{(M)}\).  Every monomial appearing in
\(|Z_{u,a}^{(M)}|^{2r}\) has \(2M=2ar\) residual positions and total level
\(2M\). Applying Lemma~\ref{lem:decorated-pattern-counting} with \(h=0\) gives
\[
\mathbb E|Z_{u,a}^{(M)}|^{2r}
\le n^{-r}(CM)^M(np)^{-M}.
\]

On \(\mathcal G\), decompose
\[
\mathcal E_{u,a}^{(M)}
=\mathcal E_{\rm pol}+\mathcal E_{\rm Tay},
\]
where \(\mathcal E_{\rm pol}\) is the sum of polynomial path monomials of level
at least \(a+1\), and \(\mathcal E_{\rm Tay}\) contains all paths with at least
one \(R_{\rm rem}\) factor.  In \(|\mathcal E_{\rm pol}|^{2r}\), terms of total
level \(2M+h\) with \(h\ge1\) contribute at most
\[
n^{-r}(CM)^M(np)^{-M}\left(\frac{CM}{\sqrt{np}}\right)^h
\]
by Lemma~\ref{lem:decorated-pattern-counting}.  Since \(M\le A_0\log n\) and
\(np\ge C_d\log^B n\) with \(B>4\), increasing \(C_d\) gives
\[
\frac{CM}{\sqrt{np}}\le \frac12.
\]
Summing the geometric series over \(h\ge1\) yields
\[
\mathbb E|\mathcal E_{\rm pol}|^{2r}\mathbf 1_{\mathcal G}
\le n^{-r}(CM)^M(np)^{-M}.
\]

For the Taylor-remainder part, the row-sum bounds from
Lemma~\ref{lem:entrywise-atom-expansion} imply
\[
\mathcal E_{\rm Tay}
=
e_u^\top(R_{\rm cl}^{a}-R_{\rm pol}^{a})s
=
\sum_{t=0}^{a-1}
e_u^\top R_{\rm cl}^{t}R_{\rm rem}R_{\rm pol}^{a-1-t}s,
\]
and therefore, using \(\|s\|_\infty=n^{-1/2}\),
\[
|\mathcal E_{\rm Tay}|\mathbf 1_{\mathcal G}
\le \frac{aC^a\tau}{\sqrt n}.
\]
Hence
\[
\mathbb E\!\left[|\mathcal E_{\rm Tay}|^{2r}\mathbf 1_{\mathcal G}\right]
\le n^{-r}(Ca)^{2r}\tau^{2r}.
\]
Because \(M=ar\), \((Ca)^{2r}\le (CM)^M\).  Choosing \(K_0\) large and then
\(C_d\) large gives
\[
\tau^2\le (np)^{-M},
\]
so
\[
\mathbb E\!\left[|\mathcal E_{\rm Tay}|^{2r}\mathbf 1_{\mathcal G}\right]
\le n^{-r}(CM)^M(np)^{-M}.
\]

It remains to handle \(\mathcal G^c\).  Under the conventions before
(\plaineqref{eq:graph_convolution}), the true and clipped residual matrices have
absolute row sums at most \(n^C\).  Thus, after taking the \(2r\)-th power,
\[
\left|e_u^\top R'^a s\right|^{2r}
+\left|e_u^\top R_{\rm cl}^a s\right|^{2r}
\le \exp(Car\log n)=\exp(CM\log n).
\]
Since \(M\le A_0\log n\), \(r\le M\), and \(np\le n\), the target scale is at
least \(\exp(-C_{A_0}\log^2 n)\).  Taking the constant in the definition of
\(\mathcal G\) sufficiently large gives
\[
\exp(CM\log n)\mathbb P(\mathcal G^c)
\le n^{-r}(CM)^M(np)^{-M}.
\]
Combining the \(Z\), polynomial-error, Taylor-error, and bad-event bounds, and
absorbing fixed factors from \(|x+y+z|^{2r}\) into \((CM)^M\), proves
\[
\mathbb E|\mathcal E_{u,a}^{(M)}|^{2r}
\le n^{-r}(CM)^M(np)^{-M}.
\]
Together with the bound for \(Z_{u,a}^{(M)}\), this proves the lemma.
\end{proof}

\begin{lemma}
\label{lem:B7_detailed_further}
Let
\[
R' := \widehat A-\eta ss^\top,
\qquad
s_i =
\begin{cases}
n^{-1/2}, & i\in S,\\
-n^{-1/2}, & i\in T,
\end{cases}
\qquad
\eta:=\frac{(p-q)n}{2d}.
\]
Assume
\[
 p\ge C_d\frac{\log^B n}{n}
\]
for a fixed \(B>4\) and sufficiently large \(C_d=C_d(B)\).  Let \(A_0>0\) be fixed.  Then there exists a constant \(C=C(A_0,B)>0\) such that, for every \(u\in[n]\) and every pair of integers \(a,r\ge 1\) satisfying \(ar\le A_0\log n\),
\[
\mathbb E\bigl|e_u^\top R'^a s\bigr|^{2r}
\le
n^{-r}\left(\frac{C ar}{np}\right)^{ar}.
\]
\end{lemma}

\begin{proof}
Set \(M:=ar\).  By Lemma~\ref{lem:decorated-expansion},
\[
e_u^\top R'^a s=Z_{u,a}^{(M)}+\mathcal E_{u,a}^{(M)},
\]
with
\[
\mathbb E|\mathcal E_{u,a}^{(M)}|^{2r}
\le n^{-r}(CM)^M(np)^{-M},
\qquad
\mathbb E|Z_{u,a}^{(M)}|^{2r}
\le
n^{-r}(CM)^M(np)^{-M}.
\]
After increasing \(C\),
\[
n^{-r}(CM)^M(np)^{-M}
\le
n^{-r}\left(\frac{C ar}{np}\right)^{ar}.
\]
The elementary inequality \(|x+y|^{2r}\le 2^{2r-1}(|x|^{2r}+|y|^{2r})\) and
the fact that \(r\le M\le A_0\log n\) allow the factor \(2^{2r}\) to be absorbed
into the constant \(C\).  This gives the claimed bound.
\end{proof}

\begin{lemma}
\label{lem:krylov-infty}
Let
\[
R' := \widehat A - \eta ss^\top,
\qquad
\eta := \frac{(p-q)n}{2d}.
\]
Assume
\[
 p\ge C_d\frac{\log^B n}{n},\qquad 1\le k\le c\log n,
\]
where \(B>4\) is fixed, \(C_d=C_d(B)\) is sufficiently large, and \(c>0\) is sufficiently small.  Then there exists a constant \(K>0\) such that, with probability at least \(1-n^{-\Omega(1)}\),
\[
\forall u\in[n],\quad \forall 1\le a\le k:
\qquad
\bigl|e_u^\top R'^a s\bigr|
\le
\frac{(K\sqrt{\log n/(np)})^a}{\sqrt n}.
\]
\end{lemma}

\begin{proof}
Fix \(u\in[n]\) and \(1\le a\le k\).  Choose
\[
r_a:=\left\lceil \frac{A_1\log n}{a}\right\rceil ,
\]
where \(A_1\) is a sufficiently large absolute constant.  Since
\(a\le k\le c\log n\), by taking \(c\le A_1\) we have
\[
A_1\log n\le a r_a\le (A_1+c)\log n\le 2A_1\log n .
\]
Applying Lemma~\ref{lem:B7_detailed_further} with its fixed parameter
\(A_0=2A_1\) and with \(r=r_a\) gives
\[
\mathbb E|e_u^\top R'^a s|^{2r_a}
\le
n^{-r_a}\left(\frac{C a r_a}{np}\right)^{a r_a}.
\]
Let
\[
T_a:=\frac{(K\sqrt{\log n/(np)})^a}{\sqrt n}.
\]
By Markov's inequality,
\[
\mathbb P(|e_u^\top R'^a s|>T_a)
\le
\left(\frac{C a r_a}{K^2\log n}\right)^{a r_a}.
\]
Since \(a r_a\asymp A_1\log n\), choosing \(A_1\) fixed and then \(K\)
sufficiently large in terms of \(A_1\) yields
\[
\mathbb P(|e_u^\top R'^a s|>T_a)\le n^{-30}.
\]
A union bound over all \(u\in[n]\) and all \(1\le a\le k\le c\log n\) gives failure probability at most \(n\log n\cdot n^{-30}\le n^{-28}\).  This proves the claimed uniform bound.
\end{proof}

\begin{lemma}
\label{lem:error-evolution-normalized}
Let \(\eta>0\) and \(\theta := C\sqrt{\frac{\log n}{np}}\). Assume the decomposition
\[
\widehat A=\eta ss^\top+R',
\]
satisfies the spectral norm bound \(\|R'\|\le \theta\) and the entrywise Krylov bound
\[
|e_u^\top R'^a s|\le \frac{\theta^a}{\sqrt n}
\]
for all \(u\in[n]\) and \(1\le a\le k\).  Let \(g\sim N(0,\sigma'^2I_n)\)
be independent of the graph.  If \(1\le k\le C_0\log n\) and
\(k\theta/\eta\le c\) for a sufficiently small constant \(c>0\), then, with
probability at least \(1-n^{-\Omega(1)}\) conditional on the graph event above,
\[
\left\|\eta^{-k}\widehat A^k(s+g)-s\right\|_\infty
\le
C\left(
\frac{k\theta/\eta}{\sqrt n}
+\frac{\sigma'\sqrt{\log n}}{\sqrt n}
+\left(\frac{\theta}{\eta}\right)^k\sigma'\sqrt{\log n}
\right).
\]
\end{lemma}

\begin{proof}
Write \(P:=ss^\top\), so that \(P^2=P\), \(Ps=s\), and
\(\widehat A=\eta P+R'\).  We first control the signal part.  Expand
\((\eta P+R')^k s\) into words in the two letters \(\eta P\) and \(R'\).
The all-projector word equals \(\eta^k s\).  We claim that every other word
with exactly \(j\ge1\) copies of \(R'\) contributes, after division by
\(\eta^k\), at most
\[
\frac{1}{\sqrt n}\left(\frac{\theta}{\eta}\right)^j
\]
in \(\ell_\infty\)-norm.  If the word has no \(P\), then it is \(R'^k s\), and
the claim is exactly the assumed Krylov bound.

Now suppose the word contains at least one \(P\).  After deleting consecutive
projectors and using \(Ps=s\), the action of the word on \(s\) is a product of
blocks of the form \(R'^{a}\) separated by projectors \(P\).  Each interior
projector turns the vector to a multiple of \(s\):
\[
P R'^a s=s(s^\top R'^a s).
\]
The scalar is bounded by
\[
|s^\top R'^a s|\le \|R'^a s\|_2\le \theta^a,
\]
and the last block, if it is not followed by a projector on the left, is
controlled coordinatewise by
\[
\|R'^a s\|_\infty
\le
\frac{\theta^a}{\sqrt n}
\]
from the Krylov assumption.  Thus a word with \(j\) total \(R'\)-factors has
coordinate size at most \(\eta^{k-j}\theta^j/\sqrt n\).  Dividing by
\(\eta^k\) proves the claim.  Summing over the at most \(\binom{k}{j}\) choices
of the \(R'\)-positions gives
\[
\left\|\eta^{-k}\widehat A^k s-s\right\|_\infty
\le
\frac{C}{\sqrt n}\sum_{j=1}^k
\binom{k}{j}\left(\frac{\theta}{\eta}\right)^j
\le
\frac{Ck\theta/\eta}{\sqrt n}.
\]

Next we control the Gaussian part.  Conditional on the graph, the pure
projector word gives
\[
\eta^{-k}(\eta P)^k g=Pg=s(s^\top g).
\]
Since \(s^\top g\sim N(0,\sigma'^2)\), a standard Gaussian tail bound gives
\[
\|Pg\|_\infty\le C\frac{\sigma'\sqrt{\log n}}{\sqrt n}
\]
with probability at least \(1-n^{-\Omega(1)}\).

For the pure \(R'\)-word, for every \(u\),
\[
e_u^\top R'^k g
\sim
N\!\left(0,\sigma'^2\|e_u^\top R'^k\|^2\right),
\qquad
\|e_u^\top R'^k\|\le \|R'\|^k\le \theta^k .
\]
A union bound over \(u\in[n]\) yields
\[
\max_u |e_u^\top R'^k g|
\le C\sigma'\theta^k\sqrt{\log n}
\]
with probability at least \(1-n^{-\Omega(1)}\).  Hence the pure \(R'\)-word contributes
at most
\[
C\left(\frac{\theta}{\eta}\right)^k\sigma'\sqrt{\log n}.
\]

It remains to check the mixed words, namely those containing at least one
projector and at least one \(R'\).  Fix such a word \(W\) with \(j\) copies of
\(R'\).  Let \(\eta P\) be the rightmost projector letter in \(W\).  Then
\[
Wg=U\,(\eta P)R'^a g
=\eta\, U s\, (s^\top R'^a g),
\]
where \(0\le a\le j\) is the number of consecutive \(R'\)-letters to the right
of this projector, and \(U\) is the prefix to its left.  The prefix \(U\)
contains exactly \(j-a\) copies of \(R'\) and \(k-j-1\) copies of \(\eta P\).
For the scalar,
\[
s^\top R'^a g\sim
N\!\left(0,\sigma'^2\|s^\top R'^a\|^2\right),
\qquad
\|s^\top R'^a\|\le \theta^a
\]
with the convention \(\theta^0=1\).  Hence, by a union bound over
\(0\le a\le k\),
\[
|s^\top R'^a g|
\le
C\sigma'\theta^a\sqrt{\log n}
\]
with probability at least \(1-n^{-\Omega(1)}\).  Applying the signal-word bound
proved above, with \(k\) replaced by the length of the prefix, gives
\[
\|Us\|_\infty
\le
\frac{\eta^{k-j-1}\theta^{j-a}}{\sqrt n},
\]
because \(U\) has \(j-a\) residual letters and \(k-j-1\) projector letters.
Therefore, after division by \(\eta^k\), the word contributes at most
\[
C\frac{\sigma'\sqrt{\log n}}{\sqrt n}
\left(\frac{\theta}{\eta}\right)^j .
\]
There are at most \(\binom{k}{j}\) such words with \(j\) copies of \(R'\).
Consequently the total mixed contribution is bounded by
\[
C\frac{\sigma'\sqrt{\log n}}{\sqrt n}
\sum_{j=1}^{k-1}\binom{k}{j}\left(\frac{\theta}{\eta}\right)^j
\le
C\frac{\sigma'\sqrt{\log n}}{\sqrt n},
\]
where the last inequality follows from
\((1+\theta/\eta)^k-1\le 2k\theta/\eta\le C\) under \(k\theta/\eta\le c\).
Combining the signal estimate, the pure
projector noise estimate, the mixed Gaussian estimate, and the pure
\(R'^k g\) estimate proves the lemma.
\end{proof}

To prove our multi-class results (\cref{theorem:multi-class-partial}), we first characterize the unnormalized convolution matrix $\tilde{A}$ in terms of its expected behavior and a bounded random perturbation. These fundamental characterizations for the unnormalized setting were established by \citep{Wang2024Analysis}.

\begin{lemma}[Multi-class Characterization for $\tilde{A}$~\citep{Wang2024Analysis}]\label{lemma:multi-class-characterization}
    In the multi-class setting, the unnormalized convolution matrix, $\tilde{A}$, can be decomposed as $\tilde{A} = M + R'$ where:
    \begin{itemize}
        \item $M$ is symmetric and has rank $L-1$: its nonzero eigenvalues are
        all equal to \(\lambda=\frac{(p-q)n}{\bar d L}\), and the remaining
        eigenvalues are zero.  Also, \(MU=\lambda U\).
        \item $R'$ is a random matrix such that with probability at least $ 1-n^{-\Omega(1)}$, $\norm{R'} \leq \delta$ 
        where $\delta = C(\frac{1}{\bar{d}}(\sqrt{\frac{np(1-p)}{L}} + \sqrt{nq(1-q)}))$.
    \end{itemize}
\end{lemma}

We also need to bound the operator norm distance between the $k^{th}$ convolution and $M^k$, which directly follows from the properties of $\tilde{A}$.

\begin{lemma}[Operator Norm Bound for $\tilde{A}^k$~\citep{Wang2024Analysis}]\label{lemma:A^k}
    Suppose $|\lambda| > 4k\delta$. Then with high probability, we have
    $$\norm{\frac{1}{\lambda^k}(\tilde{A}^k - M^k)} \leq \frac{2k\delta}{|\lambda|}.$$
\end{lemma}

\begin{lemma}[Normalization perturbation for multi-class CSBM]
\label{lem:multiclass-normalization}
In the multi-class CSBM, let
\[
\bar d=\frac{pn}{L}+\frac{(L-1)qn}{L}
\]
and assume \(\bar d\ge C\log n\).  Then, with probability at least
\(1-n^{-\Omega(1)}\),
\[
\|\widehat A-\widetilde A\|
\le
C\sqrt{\frac{\log n}{\bar d}}.
\]
\end{lemma}

\begin{proof}
Let
\[
\epsilon=C_0\sqrt{\frac{\log n}{\bar d}} .
\]
For every vertex \(v\),
\[
\mathbb E d_v=(n/L-1)p+(n-n/L)q=\bar d-p,
\]
so \(\mathbb E d_v=\bar d+O(p)\).  Since \(\bar d\ge C\log n\), Chernoff's
inequality and a union bound give, with probability \(1-n^{-\Omega(1)}\),
\[
d_v\in[\bar d(1-\epsilon),\bar d(1+\epsilon)]
\quad\text{for all }v,
\qquad
d=\frac1n\sum_v d_v\in[\bar d(1-\epsilon),\bar d(1+\epsilon)]
\]
after increasing \(C_0\).

On this event,
\[
\widehat A-\widetilde A
=
\left(D^{-1/2}AD^{-1/2}-\frac1dA\right)
-
\left(
\frac{D^{1/2}\one\one^\top D^{1/2}}{\one^\top D\one}
-\frac1n\one\one^\top
\right).
\]
For the first term,
\[
\left|
\frac1{\sqrt{d_ud_v}}-\frac1d
\right|
\le
\frac{C\epsilon}{\bar d}.
\]
Since \(d_u\le 2\bar d\), every row has absolute row sum at most \(C\epsilon\).
The matrix is symmetric, hence its operator norm is at most \(C\epsilon\).
For the second term,
\[
\left|
\frac{\sqrt{d_ud_v}}{\sum_wd_w}-\frac1n
\right|
\le
\frac{C\epsilon}{n},
\]
so every row has absolute row sum at most \(C\epsilon\), and again the operator
norm is at most \(C\epsilon\).  Combining the two estimates proves the lemma.
\end{proof}

\section{Proofs of Main Theorems}
\label{appendix:proofs}

\subsection{Proof of Theorem \ref{theorem:exactly-recovery-k}}
\label{sec:proof-4.3-detailed-final}

\begin{proof}
Let
\[
s_i=\begin{cases}n^{-1/2},&i\in S,\\-n^{-1/2},&i\in T,
\end{cases}
\qquad
\eta=\frac{(p-q)n}{2d},
\qquad
R'=\widehat A-\eta ss^\top .
\]
Then \(\widehat A=\eta ss^\top+R'\).  By Lemma~\ref{prop:degree-concentration}, and by choosing the constant hidden in
\(p\ge\Omega_B(\log^B n/n)\) sufficiently large, with probability
\(1-n^{-\Omega(1)}\),
\[
|\eta-\gamma|\le C\frac{\sqrt{\log n}}{n\sqrt p},
\qquad
\|R'\|\le C\sqrt{\frac{\log n}{np}}=:\theta.
\]
By the graph-signal assumption, with the constant hidden in \(\Omega_B(\cdot)\)
chosen sufficiently large, \(k\theta/\gamma\) is smaller than the absolute
constant required below.  Increasing constants if necessary gives
\(\eta\asymp\gamma\) and \(\theta/\eta\le 2\theta/\gamma\).

By Lemma~\ref{lemma:csbm-reduction}, there exists \(w\in\mathbb R^m\) such that
\[
Xw=s+g,
\qquad
 g\sim N(0,\sigma'^2 I_n),
\qquad
\sigma'=\frac{2\sigma}{\sqrt n\,\|\mu-\nu\|}.
\]
Hence \(\widehat A^kXw=\widehat A^k(s+g)\).  Lemma~\ref{lem:krylov-infty} gives, on the same high-probability graph event,
\[
|e_u^\top R'^a s|\le \frac{\theta^a}{\sqrt n},
\qquad \forall u\in[n],\ 1\le a\le k,
\]
after increasing \(\theta\) by an absolute factor.  The use of
Lemma~\ref{lem:krylov-infty} is legitimate because the constant hidden in
\(k=O_B(\log n)\) is chosen no larger than the constant \(c\) in that lemma.
Applying Lemma~\ref{lem:error-evolution-normalized} therefore yields
\[
\left\|\eta^{-k}\widehat A^k(s+g)-s\right\|_\infty
\le
C\left(
\frac{k\theta/\eta}{\sqrt n}
+\frac{\sigma'\sqrt{\log n}}{\sqrt n}
+\left(\frac{\theta}{\eta}\right)^k\sigma'\sqrt{\log n}
\right).
\]
The first term is at most \(1/(10\sqrt n)\) by the graph-signal assumption.
Since \(\eta\asymp\gamma\), the constant \(C_B\) in the SNR condition is chosen
larger than the absolute constants hidden in \(\theta/\eta\).  The two remaining
terms are therefore at most \(1/(10\sqrt n)\) each by the assumed lower bound on
\(\|\mu-\nu\|/\sigma\), together with
\[
\sigma'=\frac{2\sigma}{\sqrt n\,\|\mu-\nu\|}.
\]
Thus
\[
\left\|\eta^{-k}\widehat A^k(s+g)-s\right\|_\infty<\frac{1}{2\sqrt n}.
\]
Every coordinate of \(s\) equals \(\pm n^{-1/2}\), so the signs of \(\eta^{-k}\widehat A^k(s+g)\), and hence of \(\widehat A^k(s+g)\), agree with the true labels.  Since \(\widehat A^kXw=\widehat A^k(s+g)\), the hyperplane with normal vector \(w\) and threshold zero separates the rows of \(\widehat A^kX\).  A union bound over the graph and Gaussian events gives failure probability \(n^{-\Omega(1)}\).
\end{proof}

\subsection{Proof of Theorem \ref{theorem:multi-class-partial}}
\label{proof:theorem:multi-class-partial}
\begin{proof}
We work on the high-probability graph event on which
\[
\widetilde A=M+R,\qquad \|R\|\le \delta,
\]
from Lemma~\ref{lemma:multi-class-characterization}.  Because
\(L\) is fixed, the probability losses and constants in the multi-class
spectral estimates are absorbed into \(n^{-\Omega(1)}\) and the displayed
absolute constants.  By
Lemma~\ref{lem:multiclass-normalization}, on the same event after a union bound,
the degree-normalization perturbation satisfies
\[
E_{\rm norm}:=\widehat A-\widetilde A,
\qquad
\|E_{\rm norm}\|\le \epsilon .
\]
Hence
\[
\widehat A=M+H,\qquad H:=R+E_{\rm norm},\qquad \|H\|\le \delta+\epsilon .
\]
Since \(M\) has rank \(L-1\) and satisfies \(MU=\lambda U\), we also have
\(M^kU=\lambda^kU\).

Let \(X=U+G\), where \(G\) has independent \(N(0,\sigma^2)\) entries.  Then
\[
X^{(k)}-U
=
\lambda^{-k}\widehat A^kU-U+\lambda^{-k}\widehat A^kG .
\]
Using \(\|A+B\|_F^2\le 2\|A\|_F^2+2\|B\|_F^2\), it is enough to control the
signal term and the Gaussian term separately.

\medskip
\noindent
\textbf{Signal term.}
For the signal part,
\[
\lambda^{-k}\widehat A^kU-U
=
\lambda^{-k}(\widehat A^k-M^k)U .
\]
We decompose
\[
\widehat A^k-M^k
=
(\widehat A^k-\widetilde A^k)+(\widetilde A^k-M^k).
\]
Lemma~\ref{lemma:A^k} gives
\[
\left\|\lambda^{-k}(\widetilde A^k-M^k)\right\|
\le
\frac{2k\delta}{|\lambda|}.
\]
For the normalization part, use the telescoping identity
\[
\widehat A^k-\widetilde A^k
=
\sum_{t=0}^{k-1}\widehat A^t(\widehat A-\widetilde A)\widetilde A^{k-1-t}.
\]
Moreover,
\[
\|\widehat A\|\le |\lambda|+\delta+\epsilon,
\qquad
\|\widetilde A\|\le |\lambda|+\delta .
\]
Therefore
\[
\left\|\lambda^{-k}(\widehat A^k-\widetilde A^k)\right\|
\le
\frac{k\epsilon}{|\lambda|}
\left(1+\frac{\delta+\epsilon}{|\lambda|}\right)^{k-1}
\le
C\frac{k\epsilon}{|\lambda|},
\]
where the last inequality uses \(k(\delta+\epsilon)/|\lambda|\le C_0^{-1}\).
Consequently,
\[
\left\|\lambda^{-k}(\widehat A^k-M^k)\right\|
\le
C\frac{k(\delta+\epsilon)}{|\lambda|}.
\]
This bound is dominated by
\[
C\left[
\left(\frac{\epsilon}{|\lambda|}\right)^k
+2k\frac{\delta+\epsilon}{|\lambda|}
\right],
\]
so
\[
\left\|\lambda^{-k}\widehat A^kU-U\right\|_F^2
\le
C\left[
\left(\frac{\epsilon}{|\lambda|}\right)^{2k}
+4k^2\left(\frac{\delta+\epsilon}{|\lambda|}\right)^2
\right]\|U\|_F^2 .
\]

\medskip
\noindent
\textbf{Gaussian term.}
Conditional on the graph, set \(B:=\lambda^{-k}\widehat A^k\).  Since \(G\)
has independent Gaussian columns \(g_1,\dots,g_m\),
\[
\|BG\|_F^2=\sum_{r=1}^m g_r^\top B^\top B g_r .
\]
A standard Gaussian quadratic-form concentration bound gives, with probability
at least \(1-n^{-\Omega(1)}\),
\[
\|BG\|_F^2
\le
C\sigma^2 m\log n\, \operatorname{Tr}(B^\top B).
\]
Here the assumption \(m\le n^{C_m}\) allows the Gaussian tail probability to be
union bounded over the \(m\) feature coordinates while keeping a \(\log n\)
factor.
Because \(\widehat A\) is symmetric,
\[
\operatorname{Tr}(B^\top B)
=
\frac{1}{|\lambda|^{2k}}\operatorname{Tr}(\widehat A^{2k}).
\]
Since \(M\) is symmetric with rank \(L-1\) and \(\|M\|=|\lambda|\), Weyl's
inequality for singular values gives
\[
s_j(\widehat A)\le s_j(M)+\|H\| .
\]
More explicitly, for \(1\le j\le L-1\),
\[
s_j(\widehat A)\le |\lambda|+\|H\|\le |\lambda|+\delta+\epsilon,
\]
while the rank inequality
\[
s_{j+L-1}(M+H)\le s_j(H)
\]
implies \(s_j(\widehat A)\le \delta+\epsilon\) for every \(j\ge L\).
Thus at most \(L-1\) singular values are bounded by
\(|\lambda|+\delta+\epsilon\), and all remaining singular values are bounded by
\(\delta+\epsilon\).  Because \(\widehat A\) is symmetric,
\(\operatorname{Tr}(\widehat A^{2k})=\sum_{j=1}^n s_j(\widehat A)^{2k}\).
Hence
\[
\frac{1}{|\lambda|^{2k}}\operatorname{Tr}(\widehat A^{2k})
\le
\left(1+\frac{\delta+\epsilon}{|\lambda|}\right)^{2k}(L-1)
+n\left(\frac{\delta+\epsilon}{|\lambda|}\right)^{2k}.
\]
Thus
\[
\left\|\lambda^{-k}\widehat A^kG\right\|_F^2
\le
C\left[
\left(1+\frac{\delta+\epsilon}{|\lambda|}\right)^{2k}(L-1)
+n\left(\frac{\delta+\epsilon}{|\lambda|}\right)^{2k}
\right]\sigma^2m\log n .
\]

\medskip
\noindent
\textbf{Counting bad vertices.}
Let \(n_e\) be the number of vertices satisfying
\(\|x_i^{(k)}-\mu_{\ell(i)}\|\ge \Delta/2\).  Then
\[
n_e\frac{\Delta^2}{4}
\le
\|X^{(k)}-U\|_F^2 .
\]
Combining the two bounds above and absorbing absolute constants gives
\begin{align*}
n_e
\le
O\Bigg(
&\left[
\left(\frac{\epsilon}{|\lambda|}\right)^{2k}
+4k^2\left(\frac{\delta+\epsilon}{|\lambda|}\right)^2
\right]\frac{\|U\|_F^2}{\Delta^2}
\\
&+
\left[
\left(1+\frac{\delta+\epsilon}{|\lambda|}\right)^{2k}(L-1)
+n\left(\frac{\delta+\epsilon}{|\lambda|}\right)^{2k}
\right]\frac{\sigma^2m\log n}{\Delta^2}
\Bigg).
\end{align*}
Every vertex outside this exceptional set is within distance \(\Delta/2\) of
its true class center, and hence is closer to its own center than to any other
class center.  This proves the claimed partial recovery bound.
\end{proof}

\end{document}

%% file: math_commands.tex
\usepackage{amsmath,amsfonts,bm}

\def\eqref#1{equation~\ref{#1}}

\def\plaineqref#1{\ref{#1}}

\def\1{\bm{1}}

\def\eps{{\epsilon}}

\DeclareMathAlphabet{\mathsfit}{\encodingdefault}{\sfdefault}{m}{sl}
\SetMathAlphabet{\mathsfit}{bold}{\encodingdefault}{\sfdefault}{bx}{n}

\newcommand{\E}{\mathbb{E}}

